\documentclass{article}
\ifdefined\pdfminorversion\pdfminorversion=7\fi

\PassOptionsToPackage{numbers,sort&compress}{natbib}
\PassOptionsToPackage{table}{xcolor}
\usepackage{preprint}

\usepackage[utf8]{inputenc}
\usepackage[T1]{fontenc}

\usepackage{textcomp}
\usepackage{xcolor}
\usepackage{microtype}
\usepackage{graphicx}
\usepackage{subcaption}
\usepackage{wrapfig}
\usepackage{booktabs}
\usepackage{tabularx}
\usepackage{multirow}

\usepackage{amsmath}
\usepackage{amssymb}
\usepackage{amsthm}
\usepackage{bm}

\usepackage{enumitem}
\usepackage{float}
\usepackage{algorithm}
\usepackage{algorithmic}

\usepackage[skins,breakable,listings]{tcolorbox}
\usepackage[framemethod=TikZ]{mdframed}

\definecolor{faintblue}{rgb}{0,0.08,0.65}
\usepackage{hyperref}
\hypersetup{
  pdftitle={Optimization-like Reasoning in LLMs over Expanding Search Spaces},
  pdfauthor={Nicolas Astorga, Nabeel Seedat, Mihaela van der Schaar},
  colorlinks=true,
  citecolor=faintblue,
  urlcolor=faintblue,
  linkcolor=faintblue
}

\setcitestyle{numbers,square,comma}

\theoremstyle{plain}
\newtheorem{theorem}{Theorem}[section]
\newtheorem{proposition}[theorem]{Proposition}
\newtheorem{lemma}[theorem]{Lemma}
\newtheorem{corollary}[theorem]{Corollary}

\theoremstyle{definition}
\newtheorem{definition}[theorem]{Definition}
\newtheorem{assumption}[theorem]{Assumption}

\theoremstyle{remark}
\newtheorem{remark}[theorem]{Remark}

\definecolor{teal}{HTML}{008080}
\newcommand{\OPTSTAR}{\texorpdfstring{OPT\ensuremath{\star}}{OPT-STAR}}
\newcommand{\OPTSTARfull}{OPTimization-based Scalable Tasks for Auto-verifiable Reasoning}

\newcommand{\tabstat}[2]{\ensuremath{#1\,{\color{black!55}\pm\,#2}}}
\newcommand{\tabbest}[2]{\ensuremath{\mathbf{#1}\,{\color{black!55}\pm\,#2}}}

\newsavebox{\tablebox}

\newcounter{researchQ}
\renewcommand{\theresearchQ}{\textbf{\textcircled{\raisebox{-0.2pt}{\alph{researchQ}}}}}

\graphicspath{{images/polyomino/}}
\newcolumntype{L}{>{\raggedright\arraybackslash}X}

\mdfsetup{%
  middlelinewidth = 1pt,
  roundcorner     = 2pt,
}

\newmdenv[
  middlelinecolor = none,
  backgroundcolor = gray!7,
  linecolor = gray!56,
]{graybox}

\definecolor{takeawaycolor}{RGB}{192, 192, 192}
\colorlet{takeawaycolor}{takeawaycolor!10}

\newenvironment{takeaway}[1][]{%
  \begin{tcolorbox}[
    enhanced,
    breakable,
    title=#1,
    colback=takeawaycolor,
    colframe=gray!70,
    colbacktitle=takeawaycolor,
    fonttitle=\bfseries,
    coltitle=black,
    attach boxed title to top left={yshift=-3mm, xshift=2mm},
    boxed title style={size=small, colback=takeawaycolor, frame hidden},
    sharp corners,
    rounded corners,
    arc=1mm,
    top=1mm,
    bottom=1mm,
    left=1mm,
    right=1mm,
    boxrule=0.6pt,
    width=\linewidth-.5mm,
  ]%
}{%
  \end{tcolorbox}
}

\newtcolorbox{promptbox}[2][]{%
  enhanced,
  breakable,
  colback=gray!3,
  colframe=black!60,
  boxrule=0.6pt,
  arc=2mm,
  left=3mm,right=3mm,top=2mm,bottom=2mm,
  fonttitle=\bfseries,
  title={#2},
  #1
}

\tcbset{jsonlisting/.style={
  listing only,
  listing engine=listings,
  colback=white,
  colframe=black!40,
  boxrule=0.4pt,
  arc=1.5mm,
  left=2mm,right=2mm,top=1mm,bottom=1mm,
  breakable,
  listing options={
    basicstyle=\ttfamily\small,
    columns=fullflexible,
    keepspaces=true,
    showstringspaces=false,
    breaklines=true
  }
}}

\DeclareCaptionStyle{mainpaper}{font=footnotesize,labelfont=footnotesize,textfont=footnotesize}

\title{Optimization-like Reasoning in LLMs\\over Expanding Search Spaces}

\author{%
  Nicolas Astorga\thanks{Correspondence: \texttt{nja46@cam.ac.uk}} \\
  University of Cambridge
  \And
  Nabeel Seedat \\
  University of Cambridge
  \And
  Mihaela van der Schaar \\
  University of Cambridge
}

\begin{document}

\maketitle

\begin{abstract}
Verifiable reward training has improved mathematical and coding reasoning, but these domains capture only part of step-by-step decision making. Many real-world tasks require finding a high-value feasible plan among many valid alternatives. We introduce \OPTSTAR{}, a scalable family of optimization-style tasks for training and evaluating LLM step-by-step optimization-like reasoning along a complexity axis: each task provides a feasibility checker and evaluator, while a complexity parameter expands the search space without requiring new human labels. This motivates studying these tasks in two regimes: (i) solver-guided online policy optimization, which uses a solver as a value oracle for partial states and uses rank-based advantages to reinforce better next steps, and (ii) search-based offline RL when such solvers are unavailable. Theoretically, we relate success in large search spaces to the information a reasoner extracts per unit of search budget. Empirically, we ablate the ingredients that make search efficient on \OPTSTAR{} and show that training on \OPTSTAR{} improves step-by-step optimization-like reasoning. Code:~\href{https://github.com/jumpynitro/optstar}{GitHub Repository}.
\end{abstract}

\section{Introduction}

Large language models (LLMs) have shown remarkable success on diverse reasoning tasks \cite{wei2022cot, zhou2022least}, but they remain brittle on constrained step-by-step decision-making tasks. In such problems, a locally plausible step may make the remaining problem infeasible or force a low-quality completion \cite{lake2017buildingthinklikepeople, yao2023react}. This issue is especially visible in optimization-style tasks: there are often many feasible answers, but the goal is to construct a better answer under constraints, not just any feasible completion.

A common way to improve reasoning is \emph{verifiable-reward training}: given a question and an automatically checkable answer, the model samples candidate traces and reinforces those that lead to correct answers \citep{zelikman2022star,yuan2023scaling,singh2023beyond,chen2023teachingselfdebug,zelikman2024quietstar,hosseini2024v-star,li2023symbolicdistillcot}. This recipe is especially effective for math and code, where correctness can often be verified through exact answers and many established benchmarks exist. However, scaling the idea to develop other forms of reasoning faces two obstacles. First, it is difficult to generate many tasks that reliably require nontrivial reasoning. Second, even when such tasks are available, obtaining supervision that identifies which intermediate steps are useful is often expensive.

\paragraph{Optimization as a source of step-by-step supervision.}
We study constrained optimization tasks as a scalable source for developing optimization-like reasoning. Optimization problems naturally contain the ingredients that verifiable training needs: many instances can be generated automatically; partial actions can often be checked for feasibility; complete solutions can be scored by an objective; and difficulty can be increased by enlarging the decision space. They also differ from standard answer-correctness settings. In many optimization problems, there are many feasible solutions, but only some are high value. Thus, the model must learn not only to produce a valid solution, but also to choose intermediate steps that lead toward a better final outcome.
\begin{figure}[t]
\captionsetup{style=mainpaper}
  \centering
  \includegraphics[width=1.\textwidth]{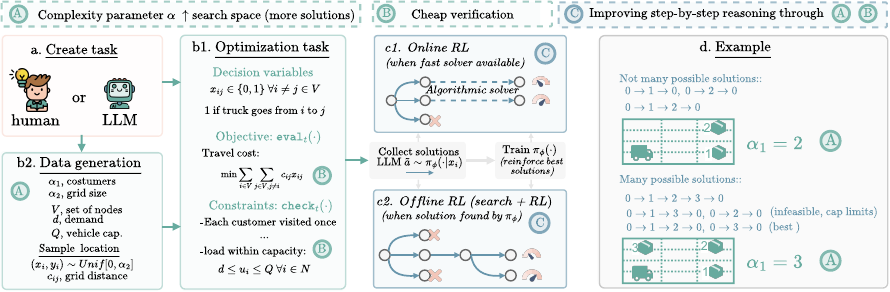}
  \vspace{-1.5em}
  \caption{\textbf{\OPTSTAR{} overview using a Traveling Salesman-style example.} \textbf{(a)} Tasks can be instantiated by humans, programs, or LLM-assisted generators. \textbf{(b1)} Classical optimization problems naturally expose objectives and constraints, making them auto-verifiable. \textbf{(b2)} Difficulty can be scaled by changing a complexity parameter $\alpha$, such as the number of customers or the structure of the instance. Increasing $\alpha$ enlarges the search space while feasibility and objective evaluation remain cheap. \textbf{(c)} We study two training recipes: solver-guided next-step supervision when a solver is available, and search-based discovery when only feasibility checks and terminal objective values are available. \textbf{(d)} A model action can correspond to fixing one or more structured decision variables.}
  \label{fig:main-fig}
  \vspace{-0.4cm}
\end{figure}

This perspective is useful both scientifically and practically. Scientifically, optimization tasks provide controlled environments for studying step-by-step decision making under constraints. Practically, many real workflows require optimization-like reasoning: assigning people to roles, scheduling jobs, routing deliveries, packing objects, selecting subsets under a budget, or coordinating multi-step plans. These tasks require models to balance hard constraints against soft preferences, local gains against future flexibility, and feasibility against objective value.

\paragraph{\OPTSTAR{} tasks.}
We formalize this setting as \OPTSTAR{}: \emph{\OPTSTARfull}. An \OPTSTAR{} task is generated by a procedure $\mathrm{Build}_\alpha(\cdot)$, where $\alpha$ controls task complexity. Each instance exposes two inexpensive verification routines: a feasibility checker $\texttt{chk}_t(\cdot)$ for partial actions and an outcome evaluator $\texttt{eval}_t(\cdot)$ for complete solutions. Increasing $\alpha$ enlarges the number of possible solutions or feasible trajectories, but does not require new human labels. This decouples reasoning difficulty from annotation effort.

\begin{graybox}[leftmargin=0pt, rightmargin=0pt, innerleftmargin=5.5pt, innerrightmargin=5.5pt, skipbelow=0pt]
\textbf{Key properties of \OPTSTAR{} tasks.} (See  Traveling Salesman Problem example in Figure \ref{fig:main-fig})

\textbf{1. Scalable generation.}
Task instances are sampled from a generator controlled by a complexity parameter $\alpha$. Increasing $\alpha$ increases the size or structure of the search space (e.g., \# of customers).

\textbf{2. Cheap verification.}
A feasibility checker $\texttt{chk}_t(\cdot)$ (e.g., enforcing load limits) validates partial actions, and an outcome evaluator $\texttt{eval}_t(\cdot)$ (e.g., Euclidean distance ) scores complete solutions. These routines are much cheaper than searching for an optimal solution.

\textbf{3. Step-level training signal.}
When a solver is available, it can score candidate next steps by their best possible completion. When no fast solver is available, search can still discover high-value trajectories using $\texttt{chk}_t(\cdot)$ and $\texttt{eval}_t(\cdot)$.
\end{graybox}
\vspace{-.5cm}

\paragraph{Research objective:} We study optimization-like reasoning over expanding search spaces as $\alpha$ grows, focusing on how performance is affected by search-space size and how to improve across different regimes. $\blacktriangleright$ \textbf{Empirically:} we use the \OPTSTAR{} structure in two complementary settings. In the large-search-space regime, a fast solver is unavailable or too expensive to call at every state. We therefore use structure-aware search to identify high-value complete trajectories, then distill them into the model. Feasibility checks prune invalid actions, while duplicate-action merging avoids wasting search budget on distinct text outputs that correspond to the same structured move. In the small-search-space regime, a solver is available. We use it as a value oracle for partial states: candidate next steps are scored by the best completion reachable after taking that step, and the policy is updated with rank-based advantages. $\blacktriangleright$ \textbf{Theoretically:} we analyze optimization-like reasoning capability and how its components are affected as the search space expands.

\textbf{Why learn reasoning when solvers exist?}
Even when solvers can verify or solve some generated optimization problems, learning a policy is valuable: it distills reusable optimization heuristics into general-purpose models and can be applied when exact solvers are unavailable, too slow to call repeatedly, or not integrated into the deployment environment.
\vspace{-0.25cm}

{\footnotesize
\begin{takeaway}[Contributions]
\textbf{\textcircled{\raisebox{-0.9pt}{1}} \; (Conceptual)}
In \S\ref{sec:vast}, we formalize \OPTSTAR{} as a family of \emph{auto-verifiable, scalable} optimization tasks, specifying components for task generation, constraint checking, and outcome evaluation. In \S\ref{sec:examples-vast}, we discuss what fits in \OPTSTAR{}, highlighting task diversity and the reasoning challenge when $\alpha$ increases.

\vspace{2pt}

\textbf{\textcircled{\raisebox{-0.9pt}{2}} \; (Algorithmic)}
In \S\ref{subsec:better-reasoning}, we identify the challenges of step-by-step reasoning on \OPTSTAR{}. Based on this, we develop complementary offline RL (\S\ref{sec:offline-rl}) and online RL (\S\ref{sec:online-rl}) procedures tailored to \OPTSTAR{}.

\vspace{4pt}

\textbf{\textcircled{\raisebox{-0.9pt}{3}} \; (Empirical)}
In Exp.~\ref{task:opte}--\ref{task:optf}, we show that \OPTSTAR{} improves step-by-step reasoning in optimization tasks. In offline RL, we examine how feasibility pruning via $\texttt{chk}_t(\cdot)$, grouping equivalent actions, and reward shaping based on $\texttt{eval}_t(\cdot)$ guide search toward promising reasoning traces. In online RL, Exp.~\ref{task:credit} compares terminal and partial-state credit under autonomous evaluation; Exp.~\ref{task:opta}--\ref{task:optc} examine transfer and curriculum effects as $\alpha$ increases.
\end{takeaway}
}

\vspace{-.15cm}

\section{\OPTSTARfull~(\OPTSTAR{})}
\label{sec:vast}

\paragraph{Preliminaries.}
A task instance is a finite- or countable-horizon constrained decision process
\begin{equation}
\label{eq:vast}
t=(S,A,P,s_0,\tau, \texttt{chk}_t,\texttt{eval}_t),
\end{equation}
with state space $S$, action space $A$, transition $P(\cdot\mid s,a)$, initial state $s_0$, and
terminal predicate $\tau:S\!\to\!\{0,1\}$. We write the terminal states as $\mathrm{Term}_t:=\{s\in S:\tau(s)=1\}$.

\begin{graybox}[leftmargin=0pt, rightmargin=0pt, innerleftmargin=5.5pt, innerrightmargin=5.5pt, skipbelow=0pt]
\begin{definition}[Task feasibility and admissibility]\label{def:feas}
An action $a$ is \emph{admissible} at state $s$ for a task $t$ if $\texttt{chk}_t(s,a)=1$. We define the set of admissible actions as $\mathcal{A}_t(s) := \{a \in A \mid \texttt{chk}_t(s,a)=1\}$. A trajectory is \emph{feasible} if it only uses admissible actions (from $\mathcal{A}_t(s)$).
\end{definition}

\begin{definition}[Outcome evaluator]\label{def:evaluator}
An \emph{outcome evaluator} $\texttt{eval}_t$ is a function that assigns a real-valued score to any terminal state:
{\setlength{\abovedisplayskip}{3pt}
 \setlength{\belowdisplayskip}{3pt}
\[
\texttt{eval}_t:\mathrm{Term}_t\longrightarrow\mathbb{R}
\]}
Without loss of generality, we assume this score is to be maximized.
\end{definition}

\begin{definition}[Auto-verifiable task]\label{def:auto-verifiable}
A task $t$ is \emph{auto-verifiable} if there exist algorithms
$\texttt{chk}_t$ and $\texttt{eval}_t$ such that, for all state-action pairs $(s,a)$ and terminal states $s_T \in \mathrm{Term}_t$,
\(\texttt{chk}_t(s,a)\) and \(\texttt{eval}_t(s_T)\) can be computed in time polynomial in the input size. Thus, checking feasibility and evaluating a proposed solution are fast, even when finding a high-scoring solution may be hard.
\end{definition}
\end{graybox}

\begin{definition}[\OPTSTARfull{}]\label{def:scalable}
For each complexity level $\alpha$, let $\Theta_\alpha$ be an instance-parameter space and let $\mathcal{D}_\alpha$ be a polynomial-time sampleable distribution over $\Theta_\alpha$. Let $\mathrm{Build}_\alpha:\Theta_\alpha\to\mathcal{T}$ construct an auto-verifiable task:
\[
t=\mathrm{Build}_\alpha(\theta)=(S,A,P,s_0,\tau,\texttt{chk}_t,\texttt{eval}_t),
\qquad \theta\sim\mathcal{D}_\alpha.
\]
We write $t\sim\mathcal{T}_\alpha$ for the induced task distribution. The family $\{\mathcal{T}_\alpha\}$ is \emph{scalable} if there exist an unbounded nondecreasing function $g$ and a small $\epsilon$ such that, with probability at least $1-\epsilon$ over $t\sim\mathcal{T}_\alpha$, (i) a chosen search-complexity measure $M(t)$ satisfies $M(t)\geq g(\alpha)$, and (ii) at least one feasible terminal state is reachable from $s_0$. Typical choices of $M(t)$ include the logarithm of the number of feasible trajectories or the logarithm of the number of reachable terminal states.
\end{definition}

This definition captures the main design goal: as $\alpha$ increases, the model faces a larger reasoning problem, but the task remains automatically checkable and comparable through $\texttt{chk}_t$ and $\texttt{eval}_t$.

\subsection{Why \OPTSTAR{}?}
\label{sec:examples-vast}

\textbf{Motivation.} As LLMs are used in more complex workflows, they will need skills beyond standard mathematical derivations and code generation.
Optimization under constraints includes resource allocation, scheduling, routing, packing, subset selection, and multi-step coordination.
In these settings, there is rarely a single correct answer.
Instead, many completions are feasible, and the value of reasoning comes from consistently choosing better trade-offs while satisfying hard constraints.

\textbf{Synthetic tasks.} Any task family satisfying the auto-verifiability and scalability conditions in \S\ref{sec:vast} can instantiate \OPTSTAR{}. Such tasks can be designed to target specific reasoning skills---for example, grid-based navigation with obstacles for spatial planning, discrete resource-allocation puzzles for combinatorial trade-offs, or rule-satisfaction games where agents must construct sequences that obey logical constraints (see App.~\ref{sec:task_examples}). By exposing $\texttt{chk}_t$ and $\texttt{eval}_t$ and scaling a complexity parameter $\alpha$ (e.g., grid size or horizon), these simulators yield large, auto-verifiable curricula. In this paper, we focus on synthetic tasks motivated mainly from traditional optimization problems.

\textbf{Traditional optimization tasks.} Mathematical optimization models are natural \OPTSTAR{} tasks: decision variables $\bm{x}$ are selected to maximize an objective $f(\bm{x})$ subject to equality constraints $h_j(\bm{x}) = 0$ and inequality constraints $g_i(\bm{x}) \geq 0$. Their explicit formulations enable direct feasibility checks and objective evaluations, satisfying \OPTSTAR{} conditions. The task-solving complexity $M(t)$ is decoupled from the complexity parameter $\alpha$: in the Traveling Salesman Problem (TSP, Fig.~\ref{fig:main-fig}), for example, increasing the number of cities/nodes ($\alpha$) dramatically enlarges the set of possible routes ($M(t)$), while route validity and length remain cheap to verify and compute. Here, a single action may fix one or more decision variables, i.e., choose a component or subset of $\bm{x}$.
\begin{figure}[t]
\captionsetup{style=mainpaper}
  \centering

  \scriptsize
  \setlength{\tabcolsep}{4pt}
  \renewcommand{\arraystretch}{0.95}
  \resizebox{\textwidth}{!}{%
    \begin{tabular}{p{1.8cm} p{3.0cm} p{2.0cm} p{5.3cm} p{1.9cm}}  \hline
  \textbf{Task family} & \textbf{Scales in $\alpha$} & \textbf{Search space $M(t)$} & \textbf{Objective \& sample hard rule} & \textbf{Reasoning} \\
  \hline
  \textbf{Role assignment}
  & roles $n$; extra cands; conflict
  & $n!$ perms
  & Max total fit; e.g. forbidden pairs or $\le 1$ per group.
  & comb. matching \\[1pt]

  \textbf{Task scheduling}
  & jobs $n$; precedence density
  & $n!$ job orders
  & Min total weighted tardiness; e.g. release dates
  & temporal planning \\[1pt]

  \textbf{Max satisfiability}
  & vars $n$; clauses $m$; clause mix
  & $2^n$ assignments
  & Max satisfied clause weight; e.g. budget $\sum_i x_i \le B$.
  & logical reasoning \\[1pt]

  \textbf{TSP}
  & cities $n$
  & $\approx (n{-}1)!/2$ tours
  & Shortest tour visiting all cities; e.g. time windows.
  & route planning \\[1pt]

  \textbf{2D packing}
  & $n$ pieces; bin size; rotations?
  & super-exponential
  & Feasible non-overlap packing (or max filled area).
  & spatial reasoning \\[1pt]

  \textbf{QAP}
  & facilities $n$; flow sparsity
  & $n!$ assignments
  & Min flow$\times$distance; e.g. cluster/capacity limits.
  & comb. optimization \\
  \hline
  \end{tabular}
  }

  \vspace{1mm}

  \includegraphics[width=\textwidth]{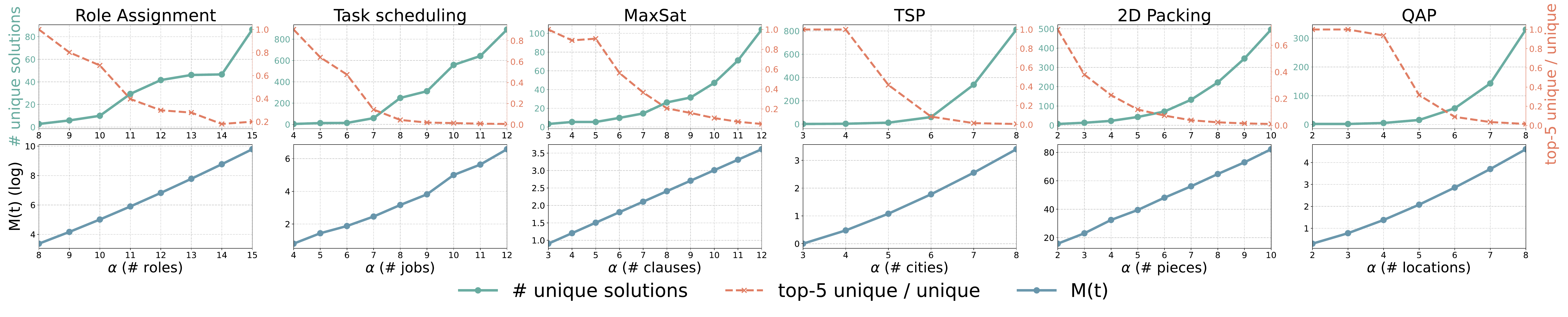}

  \vspace{-3mm}

  \caption{\textbf{Traditional optimization tasks.} \textbf{Top}: Task description detailing objectives, constraints, required reasoning, and scaling behavior with $\alpha$. \textbf{Bottom}: Solver experimental results averaged over eight runs. Plots illustrate search space growth, the number of unique solutions, and the fraction representing the top-5 solutions as a function of complexity parameter $\alpha$.}
  \label{fig:vast-at-a-glance-panel}
  \vspace{-0.25cm}
\end{figure}

\textbf{Complexity and diversity.}
Different optimization families elicit different reasoning skills. Routing problems require spatial planning and global consistency. Assignment and matching problems require trade-offs under capacity, preference, and conflict constraints. Scheduling problems require temporal reasoning and coordination. These tasks can also be enriched with additional structure, such as time windows, precedence relations, fairness requirements, or logical clauses, so that the same base family can elicit different forms of reasoning.

\textbf{Examples.} Figure~\ref{fig:vast-at-a-glance-panel} illustrates classical optimization problems and the reasoning capabilities required to solve them. As the complexity parameter $\alpha$ increases, both the search space $M(t)$ (blue line) and the number of unique solutions (teal line) grow, making the \textit{best solution} (orange line) increasingly difficult to find and thereby necessitating more advanced reasoning strategies or heuristics if these tasks were solved by a reasoning model.

\subsection{Reasoning on \OPTSTAR{} Tasks}
\label{sec:reasoning-vast}

\paragraph{From environment dynamics to language reasoning.}
\OPTSTAR{} tasks are defined as Eq. \ref{eq:vast}. However, reasoning is performed by an LLM in natural language. To enable this, we expose a language interface that converts states into prompts and parses model outputs back into actions.

\paragraph{Language interface.}
Let \(\Sigma\) be the token alphabet and let \(\Sigma^\ast\) denote the set of finite token strings. Let \(\mathcal{X}\subseteq\Sigma^\ast\) denote the set of well-formed prompt states containing the task description, constraints, and relevant history. Let \(\tilde{\mathcal{A}}\subseteq\Sigma^\ast\) denote model outputs containing a reasoning trace and a parsable action. The interface consists of a promptization map \(\psi_{\alpha}:\mathcal{T}_\alpha\times S\to\mathcal{X}\), implemented using a template family \(\Psi_{\alpha}\), and a deterministic parser \(\rho:\tilde{\mathcal{A}}\to A^{\leq K}\) that extracts up to \(K\) structured actions from a model output \(\tilde{a}_i \sim \pi_{\phi}(\cdot \mid x_i)\). In short:
\[
(t,s)\xrightarrow{\ \psi_{\alpha}\ } x\in\mathcal{X}
\ \xrightarrow{\ \mathrm{LLM}\ }\ \tilde{a}\in\tilde{\mathcal{A}}
\ \xrightarrow{\ \rho\ }\ a\in A^{\leq K}.
\]
The prompt \(x_i=\psi_{\alpha}(t,s_i)\) includes: (i) a natural-language task description, (ii) the constraints that must be respected (e.g., admissible actions), and (iii) any auxiliary instructions. When \(K=1\) the model acts step by step; for \(K>1\) it may emit a plan in one shot. Importantly, stating constraints in the prompt does not guarantee compliance, which motivates the method in the next section.

\textbf{State trajectories. } State updates follow the environment dynamics $s_{i+1}\sim P(\cdot | s_i, a_i)$, with current state $s_i$ and action $a_i =  \rho(\tilde{a}_i)$. We then reuse $\psi_{\alpha}$ to construct the next instruction prompt, i.e., $x_{i+1} = \psi_{\alpha}(t, s_{i+1})$. The language-level evolution mirrors the environment-level trajectory:
\begingroup
\setlength{\abovedisplayskip}{3pt}
\setlength{\belowdisplayskip}{3pt}
\begin{equation}
\label{eq:prompt-state-coupling}
 \underbrace{x_0 \xrightarrow{\tilde{a}_0} x_1 \xrightarrow{\tilde{a}_1} \cdots \xrightarrow{\tilde{a}_{N-1}} x_N}_{\text{Information processed in natural language}}
 \quad\implies\quad
 \underbrace{s_0 \xrightarrow{a_0} s_1 \xrightarrow{a_1} \cdots \xrightarrow{a_{N-1}} s_T}_{\text{Information processed by formal programs}}.
\end{equation}
\endgroup
where $s_T\in\mathrm{Term}_t$ is terminal. For simplicity, in what follows we assume that $P$ is deterministic. LLMs process information in natural language, while formal programs---such as Python environments or simulators---maintain and update environment states, check action feasibility, and evaluate completed solutions. When available, solvers operate on the same formal representations.

\vspace{-.125cm}
\begin{takeaway}[Theorem~\ref{thm:complexity_tradeoff}: scaling difficulty without scaling annotation cost]
\setlength{\abovedisplayskip}{1.7pt}
\setlength{\belowdisplayskip}{1.7pt}

Let \(Z^\star=Z^\star(T)\) be the canonical optimum of
\(T\sim\mathcal T_\alpha\), and suppose \(Z^\star\) is spread over
\(N_\alpha\) effective optima. For any reasoner \(R\), let \(H_j\) be
its history after budget \(j\), and let \(Y_j\) be its output. Then
\[
\Pr\!\left[Y_j=Z^\star\right]
\le
\left (I(Z^\star;H_j)+\log 2\right) / \log N_\alpha.
\]
Thus, as \(\alpha\) increases and the effective number of optima grows,
a constant success rate requires either more budget or more information
about \(Z^\star\) per unit of budget.

We describe empirical search concentration through the effective branching
factor. For an \(\epsilon\)-good terminal set
\[
\mathcal G_\epsilon(t)
=
\{s_T:\texttt{eval}_t(s_T)\ge V_t^\star-\epsilon\},
\]
let \(p_{\mathrm{hit},R}(t)\) be the probability that a budget-\(B\)
search procedure \(R\) discovers some terminal state in
\(\mathcal G_\epsilon(t)\). Define
\[
p_{\epsilon,R}(t)
=
1-\left(1-p_{\mathrm{hit},R}(t)\right)^{1/B},
\qquad
b_{\mathrm{eff},R}(t)
=
1/p_{\epsilon,R}(t).
\]
Lower \(b_{\mathrm{eff}}\) means that the same search budget places more
effective mass on high-value terminal states. We use this success-derived
quantity as a theorem-motivated diagnostic; it does not independently
estimate the \(\Omega(\log N_\alpha)\) information requirement.
\end{takeaway}

\vspace{-0.125cm}

\subsection{What constitutes ``better'' step-by-step reasoning?}
\label{subsec:better-reasoning}

Let $\mathrm{Term}_t(s) \subseteq \mathrm{Term}_t$ be the set of terminal states reachable from $s$ by feasible trajectories with $\mathrm{chk}_t(\cdot,a) = 1$. For deterministic $P$, the optimal terminal value from any partial state
$s$ is given by \(V_t^\star(s):=\max_{s_T \in \mathrm{Term}_t(s)}\mathrm{eval}_t(s_T)\). In general, we want to reinforce responses \(\tilde{a}\) sampled from the model conditioned on prompt \(x_{i}\) (representing \(s_i\))
that lead to a higher downstream value:
\vspace{-.125cm}
\begin{equation}
\label{eq:oracle-step}
\tilde{a}_{i}^{\star} =
\underset{\tilde{a}_{i}\in\tilde{\mathcal{A}}_t(s_i)}{\arg\max}
\; V_t^\star\Bigl(P(s_i,\rho(\tilde{a}_{i}))\Bigr)
\end{equation}
\vspace{-.125cm}
where $\tilde{\mathcal{A}}_t(s_i) := \{\tilde{a}_{i,g}\}_{g=1}^G$ with $\tilde{a}_{i,g} \sim \pi(\cdot \mid x_i )$. Eq.~\ref{eq:oracle-step} raises some practical challenges:
\vspace{-.155cm}
\begin{itemize}[leftmargin=*]
    \item \textbf{[P1]} We note that an action $\tilde{a}_i\sim \pi(\cdot\mid x_i)$ does not necessarily imply that $\texttt{chk}_t(s_i, a_i = \rho(\tilde{a}_i)) = 1$, i.e., responses from the LLM can be invalid, failing to respect the dynamics of the problem.
    \item \textbf{[P2]} Finding the best solution from a partial state $s_i$, i.e., computing $V_t^\star(s_i)$, requires computing $\max_{s_T \in \mathrm{Term}_t(s)}$ over potentially large space. If these policies are LLM-based, the responses can be redundant, i.e., $\tilde{a}_{i,g_1} \neq \tilde{a}_{i, g_2}$ but $\rho(\tilde{a}_{i,g_1}) = \rho(\tilde{a}_{i, g_2})$.
\end{itemize}

\vspace{-.15cm}
\section{Offline and Online Reinforcement Learning with \OPTSTAR{}}
\label{sec:methods}
\vspace{-.125cm}

In this section we instantiate the framework from \S\ref{sec:vast} with practical algorithms that optimize Eq.~\ref{eq:oracle-step}. The idea is to turn the inexpensive terminal-only signal from the outcome evaluator $\mathrm{eval}_t(\cdot)$ into informative intermediate feedback that rewards partial decisions and steers subsequent actions. $\blacktriangleright$ \textbf{Developed techniques:} We explore two \textbf{\emph{complementary}} modalities to improve step-by-step reasoning capabilities: offline and online RL. We use offline RL in the general case where a fast solver is unavailable (or too slow), so we rely on the model policy itself to generate many candidate complete solutions that are compared post hoc using $\texttt{eval}_t(\cdot)$, and we reinforce the best-achieved trajectories. We use online RL when we have access to a solver that can compute, or tightly approximate within a fixed budget, the best completion value from a partial state, enabling fast feedback about which next action is best from a given intermediate state. \textbf{Why offline and online RL?} Offline training is broadly applicable because it only requires ranking completed solutions, whereas online methods can yield stronger improvements via on-policy updates when fast solver-based verification is available~\cite{lanchantin2025bridginggapofflineonline}. Our offline procedure is closest to search-guided rejection sampling and SFT on evaluator-selected trajectories. We use the term “offline RL” in a broad sense: trajectories are selected using programmatic rewards from $\texttt{eval}_t$, and training is performed on an offline-sampled set of trajectories.

\subsection{Offline RL: Search and bootstrapping top solutions with \OPTSTAR{}}
\label{sec:offline-rl}
\vspace{-0.125cm}

\textbf{Training}. For offline RL, we simplify Eq.~\ref{eq:oracle-step} by estimating \(V_t^\star(s_0) := \max_{s_T \in \mathrm{Term}_t(s_0)} \mathrm{eval}_t(s_T)\) only at the root state \(s_0\), turning learning into a search problem over trajectories from the root. After search, we \emph{reinforce the best-achieved trajectory} for each problem by supervised fine-tuning on each \((x_i, \tilde{a}_i)\) along the trajectory attaining the highest \(\mathrm{eval}_t(\cdot)\), using a rejection-sampling style objective. Other training objectives (e.g., DPO) are also compatible. This procedure distills search-time discoveries into \(\pi_\phi\), improving search performance at inference~\cite{silver2016mastering, silver2017mastering} and implicitly maximizing Eq.~\ref{eq:oracle-step}, while avoiding expensive data collection from all nodes in the search tree.

\textbf{Search}. To improve step-by-step reasoning, we approximate \(V_t^\star(s) := \max_{s_T \in \mathrm{Term}_t(s)} \mathrm{eval}_t(s_T)\) using a finite set of LLM-generated solutions. In principle, any search method could be used; here we study MCTS and beam search, adapted to \OPTSTAR{} as MCTS\(\star\) and BeamSearch\(\star\). We modify their procedures to better balance exploration and exploitation, focusing on child creation to address \textbf{[P1--2]}. \textbf{[C1]} After expanding a node, we first prune any child with \(\texttt{chk}_t(s,a) = 0\) (addressing \textbf{[P1]}).  \textbf{[C2]} We then address \textbf{[P2]} by \emph{grouping} children with different textual responses \(\tilde{a}_{i,g_1} \neq \tilde{a}_{i,g_2}\) that map to the same structured action \(\rho(\tilde{a}_{i,g_1}) = \rho(\tilde{a}_{i,g_2})\), retaining a single response chosen uniformly at random. Note that \textbf{[C1--2]} are \textbf{\emph{plug-and-play}} methods that can be applied to any search method. We implement two MCTS variants and one beam-search variant, with details in App.~\ref{app:numcts-details}; relative to standard implementations, the main difference is that rewards are recomputed as new terminal scores are discovered. In App.~\ref{app:search-efficiency-theorems}, Theorem~\ref{thm:p1p2_smarter} explains why these components improve search.

\vspace{-0.125cm}
\subsection{Online RL: \OPTSTAR{} Solver-Guided Policy Optimization (\texorpdfstring{PO$\star$}{PO-STAR})}
\label{sec:online-rl}
\vspace{-0.125cm}

\begin{wrapfigure}{r}{0.46\textwidth}
\vspace{-1.2\intextsep}
\begin{minipage}{0.44\textwidth}
\captionsetup{type=algorithm,font=scriptsize,   aboveskip=1pt,
  belowskip=0pt}

{\scriptsize
\setlength{\parskip}{0pt}
\renewcommand{\baselinestretch}{0.88}\selectfont

\hrule
\vspace{1pt}
\caption{\textsc{\OPTSTAR{}-}\textsc{PO}$\star$: solver-guided PO}
\label{alg:nupo}
\vspace{1pt}
\hrule
\vspace{2pt}

\begin{algorithmic}[1]
\STATE \textbf{Input:} clipped policy objective, group size $G$
\STATE Collect dataset $\mathcal{D}$
\FOR{$x\in\mathcal{D}$}
  \STATE Sample $\tilde a_{1:G}\!\sim\!\pi_{\phi_{\mathrm{old}}}(\cdot\mid x)$
  \STATE Parse actions $a_{1:G}$ and states $x'_{1:G}$
  \FOR{$j=1$ \TO $G$}
    \IF{$a_j$ parses and $\mathrm{chk}(x,a_j)=1$}
      \STATE $R_j\!\leftarrow\! V(x'_j)$
    \STATE \textbf{else}
      \STATE Mark candidate $j$ invalid
    \ENDIF
  \ENDFOR
  \STATE Compare pairs: valid beats invalid; valid scores decide
  \STATE Count wins $W_j$ and losses $L_j$; ties contribute zero
  \STATE $A_j\!\leftarrow\!(W_j-L_j)/\max(G-1,1)$
  \STATE Update $\phi$ using $A_j$ on response tokens
\ENDFOR
\end{algorithmic}

\vspace{1pt}
\hrule
}
\end{minipage}
\vspace{-1.2\intextsep}
\end{wrapfigure}

We propose \textbf{solver-based online RL}, which exploits optimization structure to convert hard-to-verify actions into verifiable subproblems. Our recipe, PO$\star$, couples \OPTSTAR{}'s feasibility checks and \textbf{solver} with policy optimization. For each prompt state, it samples candidate actions and scores the best completion reachable from each next state. Algorithm~\ref{alg:nupo} uses centered pairwise advantages with a clipped policy objective, replacing the usual GRPO advantage estimator~\citep{shao2024deepseekmath}. Related objectives include GSPO~\citep{zheng2025gspo} and PPO~\citep{schulman2017proximal}; we denote our GRPO-style instance by GRPO$\star$.

\textbf{Collecting dataset $\mathcal{D}$.}
In the solver-prefix setting, each instance contributes one selected state from a solver trajectory. Exp.~\ref{task:credit} instead generates autonomous trajectories. \textbf{Pairwise advantages.} For each fixed prompt, sample $G$ responses. Valid actions receive the exact best-completion score $V(x')$; a valid candidate beats an invalid one, and equal scores or two invalid candidates tie. Candidate $j$ receives $A_j=(W_j-L_j)/(G-1)$ from its pairwise wins and losses, with $A_j=0$ for a singleton. These centered advantages train response tokens directly, without further mean/std normalization. Appendix~\ref{sec:rl-training-implementation} distinguishes the legacy rank mapping; Exp.~\ref{task:credit} extends the comparison to root-and-turn groups.

\vspace{-0.125cm}
\section{Related work}
\vspace{-0.125cm}
Prior work studies verifiable-reward optimization \cite{yu2025dapo,shao2024deepseekmath}, domains verified by compilers/solvers \cite{chen2025enigmata,wong2025logicpuzzlerl,wei2025satbench,zhu2025autologi,AlphaGeometryTrinh2024,DeepMind2024AlphaProof,patel2025getCHASE,chen2021evaluatingcodex,hendrycks2021apps,austin2021programsynthesis,li2022alphacode}, automatic task generation and verification \cite{stojanovski2025GYMreasoning,li2025internbootcamp,saxton2019mathematicsdataset,hendrycks2021mathdataset,albalak2025bigmath,astorga2026causim}, search, iterative refinement, and model formulation \cite{yao2023tree,feng2023alphazero,xie2024monte,chen2024alphamathzerops,chen-etal-2024-step,luo2024improveomegaprm,zhang2024mctsrest,zhang2024accessing,guan2025rstar,besta2024graphofthoughts,kocsis2006uct,coulom2006efficient,browne2012survey,astorga2025autoformulation,rauba2026tiny}, and action selection via information gain \cite{kobalczyk2025activetask,kriz2026amortizing}.

SIRL~\citep{chen2025sirl} rewards executable models and solver-interface code, while NP-Engine~\citep{li2025npengine}, NPG-Muse~\citep{wang2025npgmuse}, and G1~\citep{guo2025g1} reward directly generated optimization or graph solutions. RLVE~\citep{zeng2025rlve} studies adaptive verifiable environments, including optimization tasks. These output rewards can be graded; the distinction is credit timing and assignment. Math-Shepherd~\citep{wang2024mathshepherd}, AlphaMath~\citep{chen2024alphamathzerops}, and ReST-MCTS*~\citep{zhang2024mctsrest} derive process supervision from final-answer correctness in LLM-generated continuations. We contribute (i) conditions for automatic task generation and verification in \OPTSTAR{} along a controlled complexity axis; (ii) experiments on incremental structured solutions and efficient search; and (iii) reusable feasibility pruning \textbf{[C1]} and action merging \textbf{[C2]} for search-based offline RL, alongside solver-guided online RL. Solver completion values provide objective-aware partial-state credit, with pairwise comparisons grouped by fixed prompt state or autonomous root and turn.
\begin{table}[t]
\captionsetup{style=mainpaper}
\centering
\caption{Visual overview of \OPTSTAR{} task families studied.}
\label{tab:task-families-visual}
\vspace{-0.10cm}
\includegraphics[width=\linewidth]{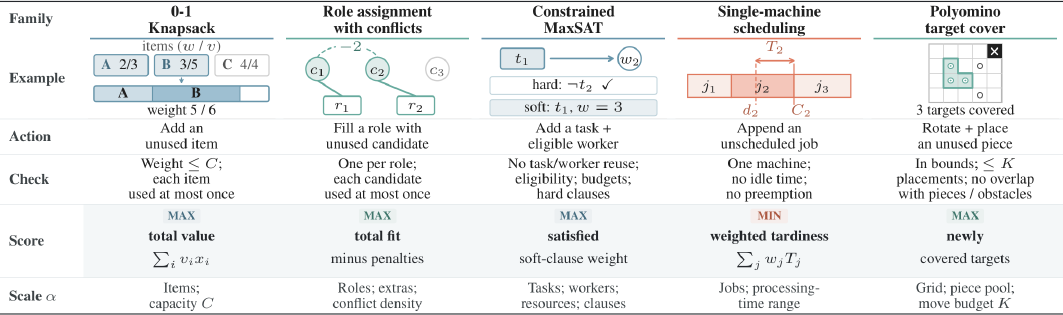}
\vspace{-0.70cm}
\end{table}

\vspace{-.75em}
\section{Experiments}
\label{sec:experiments}
\vspace{-.75em}

We study how search-space growth, controlled by complexity $\alpha_i$, affects the model’s task-solving ability and how offline and online RL can mitigate this difficulty. For offline RL, we use solver-computable optima to analyze (1) search components \textbf{[C1-2]} that improve search and (2) whether solver-based RL is consistently better. For online RL, we first isolate terminal-only versus partial-state credit under autonomous evaluation, then study transfer and curricula. We organize the research as:\vspace{-.125cm}
\begingroup
\definecolor{roadmapHeader}{HTML}{EAF3EF}
\definecolor{roadmapStripe}{HTML}{F7FAF8}
\definecolor{roadmapRule}{HTML}{A3BDB4}
\arrayrulecolor{roadmapRule}
\noindent
{\footnotesize
\setlength{\tabcolsep}{6pt}
\renewcommand{\arraystretch}{1.15}

\rowcolors{2}{roadmapStripe}{white}

\makebox[\linewidth][c]{%
\resizebox{0.925\linewidth}{!}{%
\begin{tabularx}{\linewidth}{@{}>{\bfseries}c l X@{}}
\rowcolor{roadmapHeader}
\textbf{RL} &
\textbf{Topic} &
\textbf{Question} \\
\midrule

Offline &
Search &
\textit{Can structure enable efficient optimization search?}~\ref{task:opte} \\

Offline &
Solver signal &
\textit{Can correct search match solver-guided step improvements?}~\ref{task:optf} \\

Online &
Credit assignment &
\textit{Does intermediate solver credit improve autonomous solutions?}~\ref{task:credit} \\

Online &
Task design &
\textit{Should optimization reasoning be trained on optimization tasks?}~\ref{task:opta} \\

Online &
Generalization &
\textit{Can decomposable spatial tasks capture diverse optimization skills?}~\ref{task:optb} \\

Online &
Curriculum &
\textit{How do different $\alpha_i$ sequences affect learning?}~\ref{task:optc} \\

\bottomrule
\end{tabularx}
}%
}

}
\arrayrulecolor{black}
\endgroup

\textbf{Data generation for training.} We use two data-collection protocols.

\textit{Training from collected trajectory states.} We use four complexity levels $\alpha_1,\ldots,\alpha_4$, generating 1,000 unique instances per task and level. For solver-prefix online training, the configured dataset selects one state per retained instance from its optimal solver trajectory. Candidate actions are policy-sampled, but subsequent prompts remain solver-derived. We concatenate prompts and process levels sequentially, without shuffling. Offline RL instead searches from each root, selects the best discovered trajectory using $\texttt{eval}_t(\cdot)$, and applies SFT to its state--action examples, completing $\alpha_i$ before $\alpha_{i+1}$. Retaining trajectory edges yields 1,000--5,000 offline examples per level, depending on task and complexity (App.~\ref{sec:appendix:tasks}).

\textit{Autonomous trajectory generation.} Learning from solver-provided states leaves open whether the policy can construct complete solutions through its own decisions. Exp.~\ref{task:credit} therefore uses separate fixed-horizon instances and root-started rollouts. The configured training set contains 2,048 unique instances per domain, with 512 at each of four difficulty levels. The model generates one reasoning/action turn at a time, retaining conversation history and receiving generic continuation prompts. The environment validates actions and tracks state internally; it does not return refreshed states or solver feedback to the policy. This produces model-induced trajectories for comparing terminal-only supervision with intermediate solver credit.

\textbf{Instantiation details.} We instantiate offline RL with MCTS$\star$ and online RL with GRPO$\star$. Solver-prefix training extends \texttt{verl}~\citep{sheng2025hybridflow} with eight responses per prompt; Appendix~\ref{sec:rl-training-implementation} distinguishes legacy and pairwise ranking. Our MCTS$\star$ implementation uses 16 rollouts and 20 children per expansion. We evaluate Qwen2.5-3B, Qwen2.5-7B~\citep{qwen25technicalreport}, and Llama-3.2-3B-Instruct~\citep{meta2024llama32modelcard}; the appendix provides prompts, generation settings, and training details.

\vspace{-.15cm}
\subsection{Offline RL on \OPTSTAR{}}
\vspace{-.15cm}

For these experiments we use Llama-3.2-3B-Instruct and evaluate on classical optimization problems. We report aggregate search diagnostics in Table~\ref{tab:offline-search-discovery}, highlight Knapsack in Fig.~\ref{fig:optimization}, and provide additional results in App.~\ref{sec:appendix:results-mcts}.

\refstepcounter{researchQ}
\textbf{\theresearchQ$~$Exp 1: Efficient search on \OPTSTAR{}.} \label{task:opte}$\blacktriangleright$ \textbf{Motivation.} \textbf{(A)} We first study the no-training regime, testing \textbf{[C1-2]}. This isolates the search \textbf{\emph{plug-and-play}} components that generate offline trajectories before training. \textbf{(B)} Second, we examine consistency with the search-efficiency analysis in Theorem~\ref{thm:p1p2_smarter}. $\blacktriangleright$\textbf{Setup.} We evaluate MCTS$\star$ and BeamSearch$\star$ (plus one method in App. \ref{app:numcts-details}), with/without \textbf{[C1-2]}. A DFS solver serves as reference. Table~\ref{tab:offline-search-discovery} averages over Role Assig., MaxSat, Knapsack, and QAP. Metrics are the good-terminal mass $p_g$, effective branching factor $b_{\mathrm{eff}}$, estimated samples for $90\%$ success $k_{90}$, pass@64, terminal feasibility, and exact optimality.

\begingroup
\clearcaptionsetup{table}
\begin{table}[!t]
\centering
\begin{minipage}{\linewidth}
\captionsetup{style=mainpaper,skip=2pt,position=top}
\centering
\caption{Search-only discovery on \OPTSTAR{} ($\epsilon_{\mathrm{rel}}=0.05$). Entries are estimates $\pm$ paired-bootstrap standard errors from 20,000 resamples of 16 task-by-difficulty cells (App.~\ref{app:exp1-search-diagnostics}). Rates are percentages; pass@64 is calibrated restart success. Errors reflect benchmark-cell composition, not seed variability. Bold marks the best point estimate within each family. Table~\ref{tab:offline-search-complete} gives all six metrics.}
\label{tab:offline-search-discovery}
\begingroup
\footnotesize
\setlength{\tabcolsep}{1pt}
\renewcommand{\arraystretch}{1.12}
\renewcommand{\tabstat}[2]{\ensuremath{#1}\,{\scriptsize\color{black!55}$\pm$\,#2}}
\renewcommand{\tabbest}[2]{\ensuremath{\mathbf{#1}}\,{\scriptsize\color{black!55}$\pm$\,#2}}
\sbox{\tablebox}{%
\begin{tabular*}{\dimexpr\linewidth*20/17\relax}{@{\extracolsep{\fill}}lrrrrr@{\hspace{6pt}\extracolsep{\fill}}rrrrr@{}}
\toprule
\multirow{2}{*}{\textbf{Configuration}} & \multicolumn{5}{c}{\textbf{MCTS$\star$}} & \multicolumn{5}{c}{\textbf{BeamSearch$\star$}} \\
\cmidrule(lr){2-6}\cmidrule(l){7-11}
& $p_g\uparrow$ & $b_{\mathrm{eff}}\downarrow$ & pass@64$\uparrow$ & Feas.$\uparrow$ & Exact$\uparrow$
& $p_g\uparrow$ & $b_{\mathrm{eff}}\downarrow$ & pass@64$\uparrow$ & Feas.$\uparrow$ & Exact$\uparrow$ \\
\midrule
Check + de-duplication
& \tabbest{12.2}{3.3} & \tabbest{15.04}{2.86} & \tabbest{88.9}{1.4} & \tabbest{72.8}{6.6} & \tabbest{10.9}{3.4}
& \tabbest{7.2}{1.5} & \tabbest{13.04}{2.05} & \tabbest{93.8}{1.4} & \tabbest{53.3}{3.8} & \tabbest{7.8}{2.1} \\
No check/prune
& \tabstat{4.0}{1.5} & \tabstat{25.70}{3.52} & \tabstat{84.8}{1.4} & \tabstat{16.4}{3.0} & \tabstat{3.7}{1.5}
& \tabstat{1.4}{0.4} & \tabstat{53.07}{9.68} & \tabstat{65.4}{4.3} & \tabstat{10.2}{1.7} & \tabstat{2.0}{0.7} \\
No check or de-duplication
& \tabstat{5.1}{2.0} & \tabstat{25.26}{3.49} & \tabstat{84.6}{1.3} & \tabstat{19.8}{3.6} & \tabstat{4.8}{2.0}
& \tabstat{2.5}{0.7} & \tabstat{82.88}{16.12} & \tabstat{50.9}{4.2} & \tabstat{12.8}{2.6} & \tabstat{2.7}{1.0} \\
Sequential
& \tabstat{4.9}{2.0} & \tabstat{25.30}{3.56} & \tabstat{84.6}{1.4} & \tabstat{19.8}{3.8} & \tabstat{4.6}{2.0}
& \tabstat{4.9}{2.0} & \tabstat{25.30}{3.56} & \tabstat{84.6}{1.4} & \tabstat{19.8}{3.8} & \tabstat{4.6}{2.0} \\
Solver reference
& \tabstat{0.2}{0.1} & \tabstat{520.1}{135.7} & \tabstat{18.7}{5.6} & \tabstat{2.3}{0.9} & \tabstat{0.5}{0.3}
& \tabstat{0.2}{0.1} & \tabstat{499.4}{135.6} & \tabstat{19.4}{5.7} & \tabstat{2.5}{1.1} & \tabstat{0.6}{0.4} \\
\bottomrule
\end{tabular*}%
}
\scalebox{0.85}{\usebox{\tablebox}}
\endgroup

\end{minipage}
\par\vspace{0.6\baselineskip}
\begin{minipage}{\linewidth}
  \captionsetup{type=figure,style=mainpaper,position=bottom,aboveskip=7pt,belowskip=0pt}
  \centering
  \includegraphics[width=1.\textwidth]{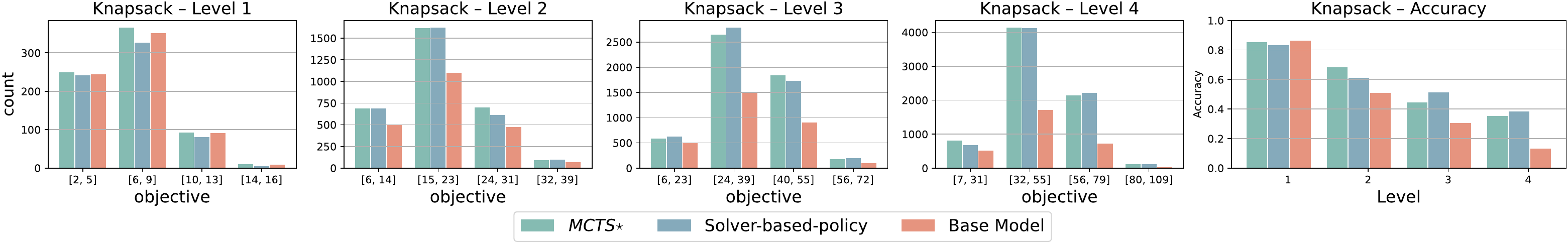}
  \vspace{-4mm}
  \caption{\textbf{Cols 1-4:} Distribution of solution quality (histograms). \textbf{Col 5:} Accuracy on Knapsack problems across four $\alpha_i$. The MCTS$\star$ method approaches solver-guided performance, particularly at lower complexities. This suggests that verification-guided self-improvement can recover much of the gain from solver-guided supervision.}
  \label{fig:optimization}
  \vspace{-1.5em}
\end{minipage}
\end{table}
\endgroup

\textbf{Results. }Table~\ref{tab:offline-search-discovery} reports search concentration under the component ablations. With both enabled, MCTS$\star$ has $b_{\rm eff}=15.04$ and BeamSearch$\star$ has $b_{\rm eff}=13.04$; disabling pruning, with or without duplicate merging, increases the effective branching factor and sharply reduces terminal feasibility and exact optimality. These changes describe greater concentration on high-value terminals under the same budget, consistent with the search-efficiency account in \S\ref{sec:reasoning-vast}. Since $b_{\rm eff}$ is computed from good-terminal mass, this comparison is descriptive.

The solver reference in Table~\ref{tab:offline-search-complete} is faster in wall-clock time, but assigns little mass to $\epsilon$-good natural-language trajectories: pass@64 is below $20\%$, versus $88.9\%$ for MCTS$\star$ and $93.8\%$ for BeamSearch$\star$. QAP is clearest: MCTS$\star$ achieves $b_{\rm eff}=24.27$ and pass@64 $=85.1\%$, while the solver reference has $b_{\rm eff}=727.9$ and pass@64 $=13.3\%$. Thus, offline RL is useful when solver labels are unavailable or costly at partial states, whereas online RL uses solver completion values to rank candidate actions. \textbf{Obs.} Solvers can be much more efficient than LLMs; LLM methods are most useful when solvers are unavailable and exhaustive DFS-style search becomes costly as complexity grows.

\refstepcounter{researchQ}
\theresearchQ~\textbf{Exp 2: Improving models.} \label{task:optf}
$\blacktriangleright$\textbf{Motivation.} We test whether \textbf{(A)} training on the best trajectory found by the model with \textbf{[C1-2]} improves over the base model and \textbf{(B)} solver-assisted training is consistently better.
$\blacktriangleright$\textbf{Setup.} For knapsack (task details in App.~\ref{app:knapsack}), we report: (a) the cumulative number of distinct solution groups found over 16 MCTS rollouts on the test set for each model, and (b) the percentage of samples recovering the best solution. Since this is an optimization problem, multiple valid solutions may exist.
$\blacktriangleright$\textbf{Results.} Figure~\ref{fig:optimization} shows difficulty increases with complexity. At higher complexity, MCTS$\star$ and solver-based MCTS substantially outperform the base model, especially at $\alpha_4$, where they find more \emph{high-value solutions}. Solver-based policy performs best, but MCTS$\star$ is close, suggesting that components \textbf{[C1-2]} largely recover the strongest model-improvement gains.

\vspace{-.125cm}
\subsection{Online RL with \OPTSTAR{}}
\label{sec:q1}
\vspace{-.125cm}

\refstepcounter{researchQ}
\theresearchQ~\textbf{Exp 3: Autonomous optimization and credit assignment.}\label{task:credit}
$\blacktriangleright$\textbf{Motivation.} Does intermediate solver feedback improve the quality of complete solutions when the policy constructs its own trajectory from the root?
$\blacktriangleright$\textbf{Setup.} New Knapsack and Role Assignment instances require four actions, creating three intermediate states for solver credit. Outcome RS/SFT and DPO use terminal-quality labels or preferences; Outcome GRPO receives reward at termination. Solver-based GRPO uses exact completion-value differences $r_t=V(s_{t+1})-V(s_t)$, standardized within each root-and-turn group. Its ranking variant substitutes centered pairwise advantages; intermediate states within a group may differ. The three GRPO arms share Qwen2.5-3B-Instruct initialization, group size 8, training-step budget, and token limits, without equating realized token counts or solver compute. Appendix~\ref{app:credit-control} gives the ordered four-level configuration and implementation details.

$\blacktriangleright$\textbf{Evaluation.} We use 256 held-out roots per domain and eight autonomous rollouts per root, without solver guidance or search. Table~\ref{tab:autonomous-credit} reports terminal feasibility, exact optimality, observed $p_g=\mathrm{pass@1}$ and pass@8, and descriptive $b_{\mathrm{eff}}$. Because $b_{\mathrm{eff}}$ and pass@8 use the same good-terminal indicators, their within-set association does not independently establish predictive validity. Pass@8 is measured from the saved rollouts. Confidence margins reflect evaluation-root sampling, not training-run variability. These new fixed-horizon instances differ from the subsequent solver-prefix evaluation, so the scores are not directly comparable.
$\blacktriangleright$\textbf{Results.} On Knapsack, Solver-based GRPO (ranking) achieves 10.45\% exact optimality, 19.19\% pass@1, and 50.00\% pass@8, versus 1.95\%, 7.47\%, and 34.77\% for Outcome GRPO. Role Assignment is mixed: ranking has the highest feasibility (81.40\%), while Outcome GRPO has the highest pass@8 (44.92\%, versus 36.72\% with ranking). Intermediate credit therefore helps in a task-dependent manner; ranking does not uniformly improve solution quality or multi-sample coverage.

\begin{table}[t]
\captionsetup{style=mainpaper}
\centering
\caption{Autonomous four-action optimization. Entries are estimates $\pm$ approximate 95\% root-cluster confidence margins from 256 held-out roots per domain, with eight rollouts per root. These margins quantify evaluation-root uncertainty, not training-run variability. Rates are percentages; bold marks the best point estimate in each domain, not statistical significance.}
\label{tab:autonomous-credit}
\begingroup
\footnotesize
\setlength{\tabcolsep}{4pt}
\renewcommand{\arraystretch}{1.10}
\sbox{\tablebox}{%
\begin{tabular}{@{}llrrrrr@{}}
\toprule
\multirow{2}{*}{\textbf{Domain}} & \multirow{2}{*}{\textbf{Training signal}} & \multicolumn{2}{c}{\textbf{Terminal quality}} & \multicolumn{3}{c}{\textbf{Success and concentration}} \\
\cmidrule(lr){3-4}\cmidrule(l){5-7}
& & Feasible$\uparrow$ & Exact$\uparrow$ & $p_g{=}$pass@1$\uparrow$ & $b_{\mathrm{eff}}\downarrow$ & pass@8$\uparrow$ \\
\midrule
\multirow{6}{*}{\textit{Knapsack}} & Base & \tabstat{12.16}{1.86} & \tabstat{0.83}{0.41} & \tabstat{2.98}{0.90} & \tabstat{14.440}{0.598} & \tabstat{17.19}{4.63} \\
& Outcome RS/SFT & \tabstat{21.48}{2.33} & \tabstat{2.10}{0.65} & \tabstat{6.98}{1.49} & \tabstat{11.295}{0.770} & \tabstat{33.98}{5.82} \\
& Outcome DPO & \tabstat{19.14}{2.14} & \tabstat{1.51}{0.67} & \tabstat{5.52}{1.31} & \tabstat{12.380}{0.725} & \tabstat{27.34}{5.47} \\
& Outcome GRPO & \tabstat{24.27}{2.47} & \tabstat{1.95}{0.61} & \tabstat{7.47}{1.53} & \tabstat{11.024}{0.786} & \tabstat{34.77}{5.84} \\
& Solver-based GRPO & \tabbest{26.90}{3.08} & \tabstat{5.62}{1.23} & \tabstat{14.50}{2.45} & \tabstat{8.113}{0.870} & \tabstat{49.61}{6.13} \\
& \quad + ranking & \tabstat{23.14}{3.51} & \tabbest{10.45}{2.33} & \tabbest{19.19}{3.35} & \tabbest{7.394}{0.900} & \tabbest{50.00}{6.13} \\
\midrule
\multirow{6}{*}{\textit{Role Assignment}} & Base & \tabstat{10.89}{1.51} & \tabstat{1.22}{0.47} & \tabstat{1.51}{0.55} & \tabstat{15.867}{0.470} & \tabstat{10.94}{3.82} \\
& Outcome RS/SFT & \tabstat{19.53}{2.69} & \tabstat{2.73}{0.84} & \tabstat{3.27}{1.04} & \tabstat{14.264}{0.608} & \tabstat{17.97}{4.70} \\
& Outcome DPO & \tabstat{30.37}{2.37} & \tabstat{4.15}{0.94} & \tabstat{4.79}{1.10} & \tabstat{12.629}{0.706} & \tabstat{28.12}{5.53} \\
& Outcome GRPO & \tabstat{75.59}{2.59} & \tabstat{9.03}{1.59} & \tabstat{10.40}{1.84} & \tabbest{9.377}{0.831} & \tabbest{44.92}{6.10} \\
& Solver-based GRPO & \tabstat{76.46}{2.35} & \tabstat{7.62}{1.49} & \tabstat{8.79}{1.65} & \tabstat{9.977}{0.802} & \tabstat{43.75}{6.10} \\
& \quad + ranking & \tabbest{81.40}{2.65} & \tabbest{9.33}{2.12} & \tabbest{10.50}{2.35} & \tabstat{10.144}{0.823} & \tabstat{36.72}{5.92} \\
\bottomrule
\end{tabular}%
}
\scalebox{0.85}{\usebox{\tablebox}}
\endgroup
\vspace{-1em}
\end{table}

\begin{wraptable}{r}{0.64\textwidth}
\vspace{-0.8em}
\captionsetup{style=mainpaper,skip=2pt}
\centering
\caption{Pass@1 (\%) $\pm$ standard error on optimization tasks using Qwen2.5-3B-Instruct. All models are evaluated on the same prompts, with eight responses per prompt. Bold marks the highest scores. See Appendix~\ref{app:transfer-pairwise} for evaluation details.}
\label{tab:qwen25-3b-pass1-optimization}
\begingroup
\footnotesize
\setlength{\tabcolsep}{1pt}
\renewcommand{\arraystretch}{1.0}
\newcommand{\transferstat}[2]{\shortstack[r]{$#1$\\[-1pt]{\scriptsize\color{black!55}$\pm\,#2$}}}
\newcommand{\transferbest}[2]{\transferstat{\mathbf{#1}}{#2}}
\scalebox{0.97}{%
\resizebox{\linewidth}{!}{%
\begin{tabular}{@{}l*{12}{r}@{}}
\toprule
& \multicolumn{4}{c}{\textbf{Role Assignment (ID)}}
& \multicolumn{4}{c}{\textbf{MaxSAT (OOD)}}
& \multicolumn{4}{c}{\textbf{Scheduling (OOD)}} \\
\cmidrule(lr){2-5}\cmidrule(lr){6-9}\cmidrule(l){10-13}
\textbf{Method}
& $\alpha_1$ & $\alpha_2$ & $\alpha_3$ & $\alpha_4$
& $\alpha_1$ & $\alpha_2$ & $\alpha_3$ & $\alpha_4$
& $\alpha_1$ & $\alpha_2$ & $\alpha_3$ & $\alpha_4$ \\
\midrule
Base
& \transferstat{33.82}{0.78} & \transferstat{25.57}{0.63} & \transferstat{21.15}{0.57} & \transferstat{16.87}{0.49} & \transferstat{42.11}{1.65} & \transferstat{29.44}{1.45} & \transferstat{25.49}{1.22} & \transferstat{22.36}{0.92} & \transferstat{57.67}{0.87} & \transferstat{44.34}{0.69} & \transferstat{36.30}{0.50} & \transferstat{37.30}{0.49} \\
\addlinespace[1.5pt]
MATH-500
& \transferstat{32.50}{0.84} & \transferstat{24.89}{0.66} & \transferstat{20.65}{0.56} & \transferstat{16.58}{0.49} & \transferstat{41.67}{1.69} & \transferstat{27.72}{1.35} & \transferstat{24.82}{1.20} & \transferstat{21.09}{0.86} & \transferstat{57.60}{0.82} & \transferstat{44.47}{0.63} & \transferstat{35.94}{0.49} & \transferstat{37.36}{0.52} \\
\addlinespace[1.5pt]
Code
& \transferstat{33.35}{0.79} & \transferstat{25.77}{0.66} & \transferstat{20.28}{0.56} & \transferstat{16.56}{0.48} & \transferstat{41.94}{1.66} & \transferstat{29.44}{1.39} & \transferstat{25.18}{1.15} & \transferstat{21.81}{0.91} & \transferstat{58.12}{0.82} & \transferstat{44.61}{0.67} & \transferstat{36.07}{0.50} & \transferstat{37.39}{0.52} \\
\addlinespace[1.5pt]
Trained
& \transferbest{43.53}{0.91} & \transferbest{37.31}{0.65} & \transferbest{33.47}{0.58} & \transferbest{28.61}{0.56} & \transferbest{46.82}{1.73} & \transferbest{35.09}{1.68} & \transferbest{30.87}{1.40} & \transferbest{25.09}{1.05} & \transferbest{60.08}{0.96} & \transferbest{46.17}{0.76} & \transferbest{38.96}{0.60} & \transferbest{39.14}{0.59} \\
\bottomrule
\end{tabular}%
}
}
\endgroup
\vspace{-0.8em}
\end{wraptable}

\textbf{Solver-prefix evaluation.} We evaluate model responses from states along solver trajectories using the retained validity-and-reference-value-matching scorer. Pass@$k$ measures this solver-prefix score, separately from autonomous terminal success. Table~\ref{tab:qwen25-3b-pass1-optimization} reports Qwen pass@1; Appendix~\ref{app:optimization-results} compares pass@1 and pass@8 for both 3B models.

\refstepcounter{researchQ}
\theresearchQ~\textbf{Exp 4: Domain training and transfer.} \label{task:opta}
$\blacktriangleright$\textbf{Motivation.} We ask whether optimization-specific training improves over verifiable math or code RL.
$\blacktriangleright$\textbf{Setup.} We compare pairwise training on Role Assignment with ordinary GRPO on MATH-500~\citep{hendrycks2021mathdataset,lightman2024letsverify} (500 problems, two epochs) and MBPP~\citep{austin2021programsynthesis} (four epochs). We evaluate Qwen2.5-3B-Instruct on Role Assignment (ID), MaxSAT, and Machine Scheduling (OOD) using the same frozen solver-prefix evaluation.
$\blacktriangleright$\textbf{Results.} In Table~\ref{tab:qwen25-3b-pass1-optimization}, math/code training remains close to the base model, while pairwise Role Assignment training improves all twelve pass@1 point estimates. These domain comparisons use different training datasets and budgets; Appendix~\ref{app:transfer-pairwise} details the protocol and scorer limitations.

\refstepcounter{researchQ}
\theresearchQ~\textbf{Exp 5: Generalization (spatial tasks).} \label{task:optb}$\blacktriangleright$\textbf{Motivation.} Optimization tasks can involve heterogeneous constraints, and solving them can require a broad set of skills. We demonstrate this using a 2D polyomino task (Table~\ref{tab:task-families-visual}), where models must select, rotate, and translate pieces to cover target cells. $\blacktriangleright$\textbf{Setup.} Qwen2.5-7B was the only model with sufficient capacity to sample rewardable solutions and bootstrap RL in this task. We evaluate acquired capabilities on (1) the held-out polyomino test set and (2) auxiliary spatial QA tasks covering translation, rotation, reflection, symmetry, and coverage. These auxiliary tasks are QA pairs rather than optimization problems, with two difficulty levels based on grid size (details in App.~\ref{sec:appendix:tasks}). $\blacktriangleright$\textbf{Results.} Fig.~\ref{fig:polyomino} shows that the curriculum improves both in-distribution polyomino performance and out-of-distribution spatial QA, suggesting transfer to the underlying spatial operations.

\textbf{Generalization (Math).} GRPO training on Role Assignment+MaxSAT yields higher mean pass@1 on AMC\textquotesingle{}23, AIME\textquotesingle{}24, and AIME\textquotesingle{}25~\citep{MAA_AMC_2023,MAA_AIME_2024,MAA_AIME_2025} for both Qwen2.5-3B and Llama-3.2-3B, averaged over three decoding seeds. App.~\ref{app:math-problems} reports the pass@1 and pass@8 results with standard errors.

\refstepcounter{researchQ}
\theresearchQ~\textbf{Exp 6: How the $\alpha_i$ curriculum affects learned skills.} \label{task:optc} We study curriculum learning on the polyomino task by varying the complexity parameter $\alpha_i$. Using \OPTSTAR{} to generate data across complexity regimes, we compare three training strategies: easy-only ($\alpha_1$), hard-only ($\alpha_4$), and mixed curriculum training ($\alpha_1 \rightarrow \alpha_4$). Figure~\ref{fig:polyomino} (right) shows three trends. First, hard-only training underperforms, consistent with sparse rewards on hard tasks. Second, easy-only training is strong on $\alpha_1$ and $\alpha_2$ but degrades as difficulty increases. Third, curriculum training gives the best performance on the harder ID levels and the strongest OOD Spatial QA results. These gains indicate that curriculum training improves the underlying spatial skills more reliably than single-level training.

\begin{figure}[t]
\captionsetup{style=mainpaper}
    \centering

    \begin{subfigure}[c]{0.49\textwidth}
        \centering
        \includegraphics[width=\textwidth]{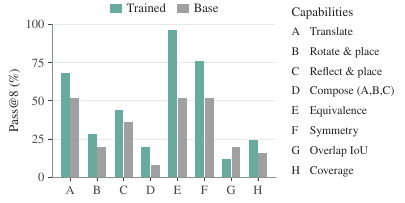}
    \end{subfigure}
    \hfill
    \begin{subfigure}[c]{0.49\textwidth}
        \centering
        \scriptsize
        \setlength{\tabcolsep}{1.4pt}
        \renewcommand{\arraystretch}{1.18}
        \renewcommand{\tabstat}[2]{\ensuremath{#1}\,{\fontsize{5.5}{6}\selectfont\color{black!55}$\pm$\,#2}}
        \renewcommand{\tabbest}[2]{\ensuremath{\mathbf{#1}}\,{\fontsize{5.5}{6}\selectfont\color{black!55}$\pm$\,#2}}
        \resizebox{\linewidth}{!}{%
        \begin{tabular}{@{}llcccc@{}}
            \toprule
            & & \multicolumn{4}{c}{Training set} \\
            \cmidrule(lr){3-6}
            Evaluation task
            & Dist.
            & Base
            & \shortstack{$\alpha_1$--$\alpha_4$\\1k each}
            & \shortstack{$\alpha_1$ only\\4k}
            & \shortstack{$\alpha_4$ only\\4k} \\
            \midrule
            Polyomino $\alpha_1$ & ID & \tabstat{4.4}{1.30} & \tabstat{10.0}{1.90} & \tabbest{32.4}{2.97} & \tabstat{4.0}{1.24} \\
            Polyomino $\alpha_2$ & ID & \tabstat{2.4}{0.97} & \tabstat{22.8}{2.66} & \tabbest{38.4}{3.08} & \tabstat{2.4}{0.97} \\
            Polyomino $\alpha_3$ & ID & \tabstat{5.3}{1.03} & \tabbest{22.5}{1.69} & \tabstat{17.1}{1.59} & \tabstat{3.6}{0.86} \\
            Polyomino $\alpha_4$ & ID & \tabstat{1.2}{0.39} & \tabbest{5.2}{0.77} & \tabstat{2.5}{0.56} & \tabstat{0.4}{0.23} \\
            \midrule
            Spatial QA (diff 1) & OOD & \tabstat{32.0}{3.13} & \tabbest{46.0}{2.95} & \tabstat{28.0}{2.84} & \tabstat{29.5}{3.00} \\
            Spatial QA (diff 2) & OOD & \tabstat{24.5}{2.84} & \tabbest{34.0}{2.93} & \tabstat{31.0}{2.77} & \tabstat{27.5}{3.05} \\
            \bottomrule
        \end{tabular}
        }
    \end{subfigure}

    \captionsetup{skip=5pt}
    \caption{\textbf{2D spatial reasoning.} \textbf{Left:} Spatial QA (difficulty 1), base versus mixed-curriculum training. \textbf{Right:} Curriculum comparison: pass@8 (\%) $\pm$ evaluation SE (percentage points). Bold marks the best means. SEs reflect evaluation uncertainty for saved checkpoints, not training-run variability; Appendix~\ref{sec:spatial-reasoning-results} details the clustering and stratification.}
    \label{fig:polyomino}
    \vspace{-0.45cm}
\end{figure}

\vspace{-.125cm}
\section{Conclusion}
\vspace{-.125cm}

We propose \OPTSTAR{} to elicit optimization-like reasoning. We study how enlarging the search space affects such reasoning (Theorem~\ref{thm:complexity_tradeoff}) and how training helps find high-value solutions. This view motivates two \textbf{\emph{complementary}} methods: (1) \textbf{\emph{plug-and-play} search components [C1--2]}, using feasibility checking and action merging to improve search efficiency, supported by theoretical analysis and search ablations; and (2) \textbf{solver-based RL}, using available fast solvers to provide objective-aware rewards for partial actions. Solver-prefix experiments show gains over domain-transfer references. In autonomous four-action experiments, solver-based credit with ranking improves Knapsack exact-optimality and success rates over terminal-only GRPO; Role Assignment reveals a feasibility--coverage tradeoff. These benefits remain task-dependent. Curriculum learning and generalization through scalable task generation remain promising directions.

\textbf{Limitations and future work.} A promising direction is to \textbf{\emph{scale optimization reasoning}} using LLM priors to automate the generation of more complex and diverse tasks while evaluation remains simple. Scaling task generation and complexity could broaden this testbed for optimization-based generalization. The autonomous comparison covers only two synthetic domains, and its uncertainty estimates concern evaluation instances rather than independent training runs. Solver availability and partial-state evaluation costs constrain solver-based RL.

\begin{ack}
We thank the anonymous NeurIPS reviewers and members of the van der Schaar lab for their insightful comments and suggestions. Nicolas Astorga thanks the W.\ D.\ Armstrong Trust for its support and is funded by the Cystic Fibrosis Trust. This work was supported by Microsoft's Accelerate Foundation Models Academic Research initiative. The authors declare no additional competing financial interests.
\end{ack}

\clearpage
\appendix
\phantomsection
\section*{Appendix Index}
\label{app:index}
\addcontentsline{toc}{section}{Appendix Index}

\begin{tcolorbox}[
  enhanced,
  breakable,
  colback=teal!3,
  colframe=teal!95!black,
  title={\textbf{Quick guide to the appendix}},
  fonttitle=\bfseries,
  coltitle=black,
  arc=1.5mm,
  boxrule=0.6pt,
  left=2mm,
  right=2mm,
  top=1mm,
  bottom=1mm
]
\small
\setlength{\tabcolsep}{4pt}
\renewcommand{\arraystretch}{1.08}

\newcommand{\appidx}[3]{%
  \hyperref[#1]{\textbf{#2}} &
  \hyperref[#1]{#3}\\[2pt]
}
\newcommand{\appidxsub}[3]{%
  \hyperref[#1]{\hspace{1em}#2} &
  \hyperref[#1]{#3}\\[1pt]
}

\begin{tabularx}{\textwidth}{@{}p{1.6cm}X@{}}
\toprule
\textbf{Ref.} & \textbf{Content} \\
\midrule
\appidx{app:info_reasoning}{A}{Information-theoretic formalization of reasoning on \OPTSTAR{}}
\appidxsub{app:scalable-generators}{A.1}{Effective optima and scalable complexity}
\appidxsub{app:interactive-inference}{A.2}{Reasoners, histories, and success}
\appidxsub{app:core-theorems}{A.3}{Core information-theoretic results}
\appidxsub{app:adaptive-hit-curves}{A.4}{Effective branching and empirical estimation}
\appidxsub{app:search-efficiency-theorems}{A.5}{Checking and redundancy removal}

\appidx{sec:task_examples}{B}{Task examples}
\appidxsub{appendix:optimization-tasks-concise}{B.1}{Mathematical optimization tasks}
\appidxsub{appendix:logic-puzzles}{B.2}{Logic puzzle tasks}
\appidxsub{appendix:simulator-tasks}{B.3}{Simulator-based tasks}

\appidx{sec:appendix:tasks}{C}{Additional details on the generation process}
\appidxsub{app:knapsack}{C.1}{0--1 Knapsack}
\appidxsub{app:qap-cfeu}{C.2}{Quadratic Assignment / CF--EU Manhattan}
\appidxsub{app:role-assignment-conflicts}{C.3}{Role Assignment with Conflicts}
\appidxsub{app:maxsat}{C.4}{MaxSat}
\appidxsub{app:single-machine-scheduling}{C.5}{Single-machine scheduling}
\appidxsub{app:polyomino-target-cover}{C.6}{Polyomino Target Cover}
\appidxsub{app:spatial-2d-tasks}{C.7}{Auxiliary Grid-Spatial tasks}

\appidx{app:prompts}{D}{Prompt examples}
\appidx{app:polyomino-examples}{E}{Polyomino examples}

\appidx{sec:models-tasks}{F}{Additional details and extra results}
\appidxsub{app:numcts-details}{F.1}{Search ablations and additional details of MCTS$\star$}
\appidxsub{sec:policy-training-implementation}{F.2}{Policy training implementation for RFT}
\appidxsub{sec:rl-training-implementation}{F.3}{Policy training GRPO}

\appidx{sec:online-rl-results}{G}{Online RL results}
\appidxsub{sec:spatial-reasoning-results}{G.1}{Spatial reasoning results}
\appidxsub{app:math-problems}{G.2}{Math benchmark results}

\appidx{app:optimization-results}{H}{Optimization problem results}
\appidxsub{app:optimization-role}{H.1}{Role Assignment results}
\appidxsub{app:optimization-maxsat}{H.2}{MaxSAT results}
\appidxsub{app:optimization-scheduling}{H.3}{Scheduling results}
\appidxsub{sec:appendix:results-mcts}{H.4}{Additional MCTS results}
\appidxsub{app:exp1-search-diagnostics}{H.5}{Search diagnostics and complete Table 2 point estimates}
\appidx{app:credit-control}{I}{Autonomous four-action control: protocol details}
\appidx{app:broader-impacts}{J}{Broader impacts}
\bottomrule
\end{tabularx}
\end{tcolorbox}

\clearpage

\section{Information-theoretic formalization of reasoning on \OPTSTAR{}}
\label{app:info_reasoning}

\paragraph{Scope and notation.}
All logarithms are natural, so information is measured in nats.
We use the \OPTSTAR{} task definition from the main paper and add only
the probabilistic notation needed for the theory. For a task \(t\), let
\(\mathrm{Term}_t(s)\) denote the set of terminal states reachable from
state \(s\) by feasible trajectories. We write \(s_0\) for the initial
state and use \(\xi\) for a complete rollout, with terminal state
\(s_T(\xi)\).

The results below are stated for finite effective solution spaces. This
covers the finite combinatorial tasks used in our experiments. Countable
or very large spaces can be handled by truncating the search space or by
grouping terminal states into finitely many effective equivalence classes.

\subsection{Effective optima and scalable complexity}
\label{app:scalable-generators}

For a realized task \(t\), let
\[
\mathcal Z_t := \mathrm{Term}_t(s_0)
\]
be its reachable terminal set, or a finite quotient of that set into
effective solution classes. Fix a deterministic tie-breaking rule
\(\mathrm{tb}(\cdot)\) that selects one element from any nonempty finite set.
The canonical optimal terminal state is
\[
Z^\star(t)
:=
\mathrm{tb}
\left(
\arg\max_{z\in\mathcal Z_t}\texttt{eval}_t(z)
\right),
\qquad
V_t^\star
:=
\texttt{eval}_t(Z^\star(t)).
\]
When \(T\sim\mathcal T_\alpha\), the canonical optimum
\[
Z^\star := Z^\star(T)
\]
is a random variable.

\begin{assumption}[Effective uniform optimum]
\label{ass:effective_uniform}
For each complexity level \(\alpha\) considered in the theorem statements,
there exists a finite set \(\mathcal Z_\alpha^\star\) of size
\(N_\alpha\ge 2\) such that
\[
Z^\star\sim\operatorname{Unif}(\mathcal Z_\alpha^\star).
\]
Equivalently,
\[
H(Z^\star)=\log N_\alpha .
\]
The same assumption can be applied after replacing terminal states by
effective solution classes, for example after deterministic tie-breaking,
symmetry reduction, or \(\epsilon\)-optimal grouping.
\end{assumption}

\subsection{Reasoners, histories, and success}
\label{app:interactive-inference}

A reasoner \(R\) may be an LLM policy, a search procedure using an LLM, or
a hybrid method such as beam search, MCTS, feasibility-pruned search, or
solver-guided search. After budget \(j\), the reasoner has produced a
history
\[
H_j .
\]
This history may contain sampled text actions, parsed structured actions,
checker outputs, duplicate-removal decisions, terminal candidates, terminal
values, and internal search statistics.

Let
\[
\mathcal Z_R(t,j)\subseteq \mathrm{Term}_t(s_0)
\]
be the set of feasible terminal states discovered by \(R\) on task \(t\)
by budget \(j\). The reasoner's output is denoted by \(Y_j\). When at least
one terminal state has been found, we take
\[
Y_j
:=
\mathrm{tb}
\left(
\arg\max_{z\in\mathcal Z_R(t,j)}
\texttt{eval}_t(z)
\right).
\]
When no terminal state has been found, we set \(Y_j=\bot\). In all cases,
\(Y_j\) is measurable with respect to \(H_j\).

For \(\epsilon\ge 0\), define the anytime success curve
\[
S_R(j;\alpha,\epsilon)
:=
\Pr_{T\sim\mathcal T_\alpha,R}
\left[
\exists z\in \mathcal Z_R(T,j):
\texttt{eval}_T(z)\ge V_T^\star-\epsilon
\right].
\]
For exact optimality, take \(\epsilon=0\).

The reasoning information accumulated by budget \(j\) is
\[
\mathcal I_R(j;\alpha)
:=
I(Z^\star;H_j).
\]
If the history is revealed incrementally as
\(H_j=(O_1,\ldots,O_j)\), define the per-step information gain
\[
\mathrm{IG}_i
:=
I(Z^\star;O_i\mid H_{i-1}).
\]
By the chain rule,
\begin{equation}
\label{eq:info_chain}
\mathcal I_R(j;\alpha)
=
\sum_{i=1}^j \mathrm{IG}_i .
\end{equation}

\subsection{Core information-theoretic results}
\label{app:core-theorems}

\begin{lemma}[Fano-style bound for identifying the optimum]
\label{lem:fano_opt}
Assume Assumption~\ref{ass:effective_uniform}. For any estimator \(Y\)
of \(Z^\star\),
\[
\Pr[Y=Z^\star]
\le
\frac{I(Z^\star;Y)+\log 2}{\log N_\alpha}.
\]
Equivalently, if
\[
\Pr[Y=Z^\star]\ge 1-\delta,
\]
then
\[
I(Z^\star;Y)
\ge
(1-\delta)\log N_\alpha-\log 2.
\]
\end{lemma}

\begin{proof}
Let
\[
P_e:=\Pr[Y\neq Z^\star].
\]
Since \(Z^\star\) is uniform on \(N_\alpha\) possibilities,
\[
H(Z^\star)=\log N_\alpha .
\]
Let \(E=\mathbf 1\{Y\neq Z^\star\}\). Then
\[
H(Z^\star\mid Y)
\le
H(E\mid Y)+H(Z^\star\mid Y,E)
\le
\log 2 + P_e\log N_\alpha .
\]
Therefore,
\[
I(Z^\star;Y)
=
H(Z^\star)-H(Z^\star\mid Y)
\ge
\log N_\alpha-\log 2-P_e\log N_\alpha .
\]
Since \(1-P_e=\Pr[Y=Z^\star]\), rearranging gives
\[
\Pr[Y=Z^\star]
\le
\frac{I(Z^\star;Y)+\log 2}{\log N_\alpha}.
\]
The second statement follows by rearranging the same inequality.
\end{proof}

\begin{theorem}[Complexity--reasoning tradeoff for \OPTSTAR{}]
\label{thm:complexity_tradeoff}
Assume Assumption~\ref{ass:effective_uniform}. Consider any reasoner \(R\)
that, after budget \(j\), outputs \(Y_j\) measurable with respect to \(H_j\).
Then
\[
\Pr[Y_j=Z^\star]
\le
\frac{\mathcal I_R(j;\alpha)+\log 2}{\log N_\alpha}
=
\frac{\sum_{i=1}^j \mathrm{IG}_i+\log 2}{\log N_\alpha}.
\]
Consequently, maintaining any fixed positive probability of identifying the
canonical optimum as \(N_\alpha\) grows requires
\[
\mathcal I_R(j;\alpha)=\Omega(\log N_\alpha).
\]
Thus, as the effective search space expands, a reasoner must either use more
budget or extract more information per unit of budget.
\end{theorem}

\begin{proof}
Apply Lemma~\ref{lem:fano_opt} with \(Y=Y_j\). Since \(Y_j\) is measurable
with respect to \(H_j\), the data-processing inequality gives
\[
I(Z^\star;Y_j)
\le
I(Z^\star;H_j)
=
\mathcal I_R(j;\alpha).
\]
Substituting this into Lemma~\ref{lem:fano_opt} yields
\[
\Pr[Y_j=Z^\star]
\le
\frac{\mathcal I_R(j;\alpha)+\log 2}{\log N_\alpha}.
\]
The equality with \(\sum_{i=1}^j\mathrm{IG}_i\) follows from
\eqref{eq:info_chain}. The asymptotic statement follows by rearranging the
bound for any fixed nonzero target success probability.
\end{proof}

\begin{corollary}[Budget lower bound]
\label{cor:budget_tradeoff}
Assume Assumption~\ref{ass:effective_uniform}. Suppose there exists
\(\kappa_\alpha>0\) such that, for all \(i\le j\),
\[
\mathrm{IG}_i
=
I(Z^\star;O_i\mid H_{i-1})
\le
\kappa_\alpha .
\]
Then
\[
\Pr[Y_j=Z^\star]
\le
\frac{j\kappa_\alpha+\log 2}{\log N_\alpha}.
\]
Equivalently, achieving
\[
\Pr[Y_j=Z^\star]\ge 1-\delta
\]
requires
\[
j
\ge
\frac{(1-\delta)\log N_\alpha-\log 2}{\kappa_\alpha},
\]
whenever the numerator is positive.
\end{corollary}

\begin{proof}
By the chain rule,
\[
\mathcal I_R(j;\alpha)
=
\sum_{i=1}^j \mathrm{IG}_i
\le
j\kappa_\alpha .
\]
Substituting this into Theorem~\ref{thm:complexity_tradeoff} gives the first
claim. The budget lower bound follows by rearranging.
\end{proof}

\paragraph{Approximate optimality.}
The same argument applies to \(\epsilon\)-optimality by replacing
\(Z^\star\) with a finite effective class variable that indexes the relevant
\(\epsilon\)-good solution region. In that case, \(N_\alpha\) should be read
as the number of effective high-value regions the reasoner must distinguish.

\subsection{Effective branching and empirical estimation}
\label{app:adaptive-hit-curves}

Theorem~\ref{thm:complexity_tradeoff} states that a harder effective search
space requires more informative histories. We now connect this to the
empirical quantities used in the experiments.

\begin{definition}[\(\epsilon\)-good terminal set]
\label{def:good-terminal-set}
For a task instance \(t\) and tolerance \(\epsilon\ge 0\), define
\[
\mathcal G_\epsilon(t)
=
\left\{
s_T\in\mathrm{Term}_t(s_0):
\texttt{eval}_t(s_T)\ge V_t^\star-\epsilon
\right\}.
\]
When the exact optimum \(V_t^\star\) is unavailable, it can be replaced by a
reference value obtained from the best solution found under a large search
budget. The resulting quantity should then be interpreted as
reference-good mass rather than true \(\epsilon\)-optimal mass.
\end{definition}

\begin{definition}[Budgeted hit probability]
\label{def:budgeted-hit}
Fix a search budget \(B\). For a possibly adaptive search procedure \(R\),
define
\[
p_{\mathrm{hit},R}(t;B)
=
\Pr_R
\left[
\mathcal Z_R(t,B)\cap\mathcal G_\epsilon(t)\neq\emptyset
\mid t
\right].
\]
Equivalently,
\[
p_{\mathrm{hit},R}(t;B)
=
\Pr_R
\left[
\max_{z\in\mathcal Z_R(t,B)}
\texttt{eval}_t(z)
\ge
V_t^\star-\epsilon
\mid t
\right],
\]
with the convention that the maximum over an empty discovered set is
\(-\infty\).
\end{definition}

\begin{definition}[Good-terminal mass and effective branching]
\label{def:restart-equivalent-mass}
For fixed budget \(B\), define the budget-normalized good-terminal mass
\[
p_{\epsilon,R}(t;B)
=
1-
\left(
1-p_{\mathrm{hit},R}(t;B)
\right)^{1/B}.
\]
The empirical effective branching factor is
\[
b_{\mathrm{eff},R}(t;B)
=
\frac{1}{p_{\epsilon,R}(t;B)} ,
\]
with the convention that \(b_{\mathrm{eff},R}(t;B)=+\infty\) when
\(p_{\epsilon,R}(t;B)=0\).
When \(B\) is fixed by the experiment, we omit \(B\) and write
\(p_{\mathrm{hit},R}(t)\), \(p_{\epsilon,R}(t)\), and
\(b_{\mathrm{eff},R}(t)\).
\end{definition}

\paragraph{Interpretation.}
The quantity \(p_{\epsilon,R}(t;B)\) is a calibration, not an independence
assumption. It asks: what independent per-budget-unit success probability
would yield the same budget-\(B\) hit probability? Thus it can compare plain
rollouts, beam search, MCTS, feasibility-pruned search, and duplicate-aware
search under a common scale.

\begin{proposition}[Independent restarts as a special case]
\label{prop:independent-special-case}
Suppose \(R\) consists of \(B\) independent rollouts, and a single rollout
hits \(\mathcal G_\epsilon(t)\) with probability \(p\). Then
\[
p_{\mathrm{hit},R}(t;B)
=
1-(1-p)^B,
\]
and therefore
\[
p_{\epsilon,R}(t;B)=p.
\]
\end{proposition}

\begin{proof}
The probability that all \(B\) independent rollouts miss the
\(\epsilon\)-good set is \((1-p)^B\). Taking the complement gives
\(p_{\mathrm{hit},R}(t;B)=1-(1-p)^B\). Substituting into
Definition~\ref{def:restart-equivalent-mass} gives
\(p_{\epsilon,R}(t;B)=p\).
\end{proof}

\begin{proposition}[Hit curves equal anytime success]
\label{prop:adaptive-hit-success-curve}
For any possibly adaptive reasoner \(R\),
\[
S_R(B;\alpha,\epsilon)
=
\mathbb E_{T\sim\mathcal T_\alpha}
\left[
p_{\mathrm{hit},R}(T;B)
\right].
\]
\end{proposition}

\begin{proof}
Condition on \(T=t\). By Definition~\ref{def:budgeted-hit},
\(p_{\mathrm{hit},R}(t;B)\) is exactly the conditional probability that
\(R\) discovers an \(\epsilon\)-good terminal state by budget \(B\).
Taking expectation over \(T\sim\mathcal T_\alpha\) gives the result.
\end{proof}

\paragraph{Empirical estimation.}
For each held-out task instance \(t\), run \(R\) independently \(m\) times
with budget \(B\). Let
\[
c_t
=
\sum_{\ell=1}^m
\mathbf 1
\left[
\mathcal Z_R^{(\ell)}(t,B)\cap\mathcal G_\epsilon(t)\neq\emptyset
\right].
\]
Then
\[
\widehat p_{\mathrm{hit},R}(t;B)=\frac{c_t}{m}.
\]
The corresponding estimates are
\[
\widehat p_{\epsilon,R}(t;B)
=
1-
\left(
1-\widehat p_{\mathrm{hit},R}(t;B)
\right)^{1/B},
\qquad
\widehat b_{\mathrm{eff},R}(t;B)
=
\frac{1}{\widehat p_{\epsilon,R}(t;B)}.
\]
For numerical stability when \(c_t=0\), one may use the smoothed estimate
\[
\widetilde p_{\mathrm{hit},R}(t;B)
=
\frac{c_t+1/2}{m+1}.
\]

\paragraph{Connection to the main theorem.}
Theorem~\ref{thm:complexity_tradeoff} applies to the full history \(H_B\)
of any search procedure, including adaptive methods such as MCTS or beam
search. The empirical hit curve measures whether the search procedure places
enough probability mass on high-value terminal states within the available
budget. Thus, at fixed \(B\), better search or a stronger learned policy
should increase \(p_{\mathrm{hit},R}\), increase \(p_{\epsilon,R}\), and
decrease \(b_{\mathrm{eff},R}\).
These success-based diagnostics do not estimate \(\mathcal I_R\) or
\(N_\alpha\), and therefore do not independently test the information bound.

\subsection{Checking and redundancy removal}
\label{app:search-efficiency-theorems}

We now isolate the local effect of the two search components used in the
paper: feasibility checking and duplicate structured-action removal.

\begin{theorem}[Checker and redundancy removal increase trace efficiency]
\label{thm:p1p2_smarter}
Consider a fixed state with feasible structured action classes
\[
\mathcal F,\qquad |\mathcal F|=K\ge 2.
\]
Let \(W\) be the structured action class whose continuation is optimal under
a fixed tie-breaking rule, and assume
\[
W\sim\operatorname{Unif}(\mathcal F).
\]
A raw text action \(\tilde A_i\) is parsed into a structured class
\[
C_i=\rho(\tilde A_i).
\]
The test succeeds when \(C_i=W\).

Compare two procedures with \(j\le K\):
\begin{itemize}[leftmargin=*]
\item \textbf{Baseline without P1/P2:} draw \(j\) raw text actions from a
fixed distribution independent of \(W\). The parsed classes may be infeasible
or repeated.
\item \textbf{P1+P2:} use the checker to reject infeasible classes and the
parser to reject repeated feasible classes, until \(j\) distinct feasible
classes have been accepted.
\end{itemize}
Let \(H_j\) denote the resulting transcript. For the baseline, let
\[
U_j
=
\left|
\{C_1,\ldots,C_j\}\cap\mathcal F
\right|
\]
be the number of distinct feasible classes tested. Then
\[
\Pr[\text{baseline hits } W]
=
\mathbb E\!\left[\frac{U_j}{K}\right]
\le
\frac{j}{K}
=
\Pr[\text{P1+P2 hits } W],
\]
and
\[
I(W;H_j^{\mathrm{P1+P2}})
\ge
I(W;H_j^{\mathrm{base}}).
\]
The success inequality is strict whenever infeasible or repeated classes
make \(\mathbb E[U_j]<j\). The information inequality is strict whenever
the baseline leaves at least two feasible classes untested with positive
probability.
\end{theorem}

\begin{proof}
Let
\[
S_j:=\{C_1,\ldots,C_j\}\cap\mathcal F
\]
be the set of distinct feasible classes tested by the baseline, and let
\(U_j=|S_j|\). Infeasible classes cannot equal \(W\), and repeated classes
retest the same event.

\textbf{Success probability.}
Conditioned on \(S_j\), the baseline succeeds if and only if \(W\in S_j\).
Since \(W\) is uniform over \(\mathcal F\),
\[
\Pr[\text{baseline hits } W\mid S_j]
=
\frac{U_j}{K}.
\]
Taking expectations gives
\[
\Pr[\text{baseline hits } W]
=
\mathbb E\!\left[\frac{U_j}{K}\right]
\le
\frac{j}{K}.
\]
Under P1+P2, every accepted test is a new feasible class, so \(U_j=j\)
deterministically. Therefore,
\[
\Pr[\text{P1+P2 hits } W]=\frac{j}{K}.
\]

\textbf{Mutual information.}
For \(0\le u\le K\), define
\[
\phi_K(u)
=
\begin{cases}
\left(1-\frac{u}{K}\right)\log(K-u), & 0\le u\le K-1,\\[3pt]
0, & u=K .
\end{cases}
\]
This function is nonincreasing on the integer set \(\{0,\ldots,K\}\).
Conditioned on the tested set \(S_j\), if a hit occurs then \(W\) is
identified exactly. If no hit occurs, then \(W\) is uniform over the
\(K-U_j\) feasible classes in \(\mathcal F\setminus S_j\). Hence
\[
H(W\mid H_j^{\mathrm{base}})
=
\mathbb E[\phi_K(U_j)].
\]
Since \(U_j\le j\) and \(\phi_K\) is nonincreasing,
\[
H(W\mid H_j^{\mathrm{base}})
\ge
\phi_K(j).
\]
Under P1+P2, \(U_j=j\) deterministically, so
\[
H(W\mid H_j^{\mathrm{P1+P2}})
=
\phi_K(j).
\]
Thus
\[
H(W\mid H_j^{\mathrm{base}})
\ge
H(W\mid H_j^{\mathrm{P1+P2}}).
\]
Since \(H(W)=\log K\), this is equivalent to
\[
I(W;H_j^{\mathrm{P1+P2}})
\ge
I(W;H_j^{\mathrm{base}}).
\]
The strictness statements follow from the same inequalities.
\end{proof}

\begin{corollary}[Non-redundant search dominates restart search]
\label{thm:mcts_vs_restart}
Take \(\mathcal F=\{1,\ldots,N\}\), so every structured class is feasible.
A search procedure that tests \(j\) distinct previously untested classes has
success probability \(j/N\), which is at least the success probability of a
restart sampler that may repeat classes. It also has at least as much mutual
information about the optimal class.
\end{corollary}

\begin{proof}
This is Theorem~\ref{thm:p1p2_smarter} with no infeasible classes. Duplicate
removal makes the accepted tests distinct; restart sampling may repeat
classes, so its number of distinct tested classes is at most \(j\).
\end{proof}

\begin{remark}[Fixed raw budgets]
Theorem~\ref{thm:p1p2_smarter} counts accepted feasible structured tests.
In the full search system, the raw generation budget is fixed. The theorem
therefore gives a local explanation: P1 and P2 improve search when they
convert raw generations into more distinct feasible tests. The empirical
quantities \(p_{\mathrm{hit},R}\) and \(b_{\mathrm{eff},R}\) measure whether
this local advantage translates into better budgeted search.
\end{remark}

\clearpage
\section{Task examples}
\label{sec:task_examples}

\textbf{What fits \OPTSTAR{}?} Any multi-step task family with a fast step-checker and a fast final evaluator, whose difficulty scales with a parameter $\alpha$; this spans optimization, logic/constraint puzzles, and simulator-backed domains.

\subsection{Mathematical optimization tasks}
\label{appendix:optimization-tasks-concise}

\scalebox{0.54}{
\begin{tabular}{p{2.5cm} p{4.0cm} p{3.1cm} p{3.1cm} p{2.6cm} p{3.1cm} p{3.1cm}}
\toprule
\textbf{Task} & \textbf{State / Actions} & \textbf{D1: Checker (step)} & \textbf{D2: Evaluator (final)} & \textbf{D3: Stop} & \textbf{D4: Scale ($\alpha$)} & \textbf{Typical Solver} \\
\midrule
\textbf{Knapsack} & Subset of items; add/remove item & Total weight $\le C$ & Total value (max) & All items considered / no improv. & \#items, $C$ & DP, ILP, B\&B \\
\midrule
\textbf{Bin Packing} & Item$\to$bin assignment; place/open bin & Bin load $\le$ cap. & \#bins (min) & All items placed & \#items, capacity & First/Best-Fit, ILP \\
\midrule
\textbf{TSP} & Partial tour; insert/swap city & No repeats; valid edges & Tour length (min) & All cities visited & \#cities & 2/3-opt, B\&B, Held--Karp \\
\midrule
\textbf{VRP} & Vehicle routes; insert customer & Cap./time-window checks & Total route cost (min) & All customers routed & \#cust., vehicles & Clarke--Wright, B\&C, meta-heur. \\
\midrule
\textbf{Job Shop} & (job,machine,time) placements & No machine overlap; precedence & Makespan (min) & All ops scheduled & \#jobs, machines & ILP, CP, local search \\
\midrule
\textbf{Flow Shop} & Job permutation; shared machine order & Feasible start/finish times & Makespan / flow time (min) & All jobs sequenced & \#jobs, machines & Johnson (2-mach), ILP, heur. \\
\midrule
\textbf{Open Shop} & Assign (job, machine, start) & No overlaps on machines & Makespan (min) & All ops assigned & \#jobs, machines & CP, ILP, meta-heur. \\
\midrule
\textbf{Cutting Stock} & Patterns; assign demand to patterns & Pattern length $\le$ stock; demand track & \#stocks used (min) & Demands satisfied & Demands, piece types & Column gen., ILP \\
\midrule
\textbf{Facility Location} & Open facilities; assign customers & Capacity per open site & Open+service cost (min) & All customers assigned & Sites, customers & ILP, B\&B, heuristics \\
\midrule
\textbf{Max Coverage} & Choose up to $k$ sets & Feasible $|S|\le k$ & Covered elements (max) & $k$ used / no improv. & \#sets, $k$ & Greedy, ILP, local search \\
\midrule
\textbf{Set Cover} & Choose sets covering universe & Coverage tracking & \#sets or weight (min) & Universe covered & \#sets, elements & Greedy, ILP, B\&B \\
\midrule
\textbf{Max Flow / Min Cut} & Edge flows; augment path & Capacity \& conservation & $s$--$t$ flow (max) / cut (min) & No augmenting path & Nodes, edges & Edmonds--Karp, Dinic \\
\midrule
\textbf{Assignment} & Agent$\leftrightarrow$task matching & One-to-one constraint & Total cost (min) & $n$ pairs formed & $n$ & Hungarian, ILP \\
\midrule
\textbf{Quadratic Assignment} & Permutation; swap facilities/locations & Valid permutation & $\sum_{i,j}\!$ flow$\cdot$dist (min) & All $n$ placed & $n$, sparsity & B\&B, tabu, meta-heur. \\
\bottomrule
\end{tabular}
}

\clearpage

\subsection{Logic puzzle tasks}
\label{appendix:logic-puzzles}

\scalebox{0.54}{
\begin{tabular}{p{2.5cm} p{4.0cm} p{3.1cm} p{3.1cm} p{2.6cm} p{3.1cm} p{3.1cm}}
\toprule
\textbf{Task} & \textbf{State / Actions} & \textbf{D1: Checker (step)} & \textbf{D2: Evaluator (final)} & \textbf{D3: Stop} & \textbf{D4: Scale ($\alpha$)} & \textbf{Typical Solver} \\
\midrule
\textbf{Sudoku} & Grid; place/remove digit & Row/col/box constraints & Valid completion (boolean) & Grid full and valid & Grid size, givens & Backtracking + CP \\
\midrule
\textbf{Kakuro} & Runs; fill digits 1--9 & Run sum and no repeats & All runs match sums & All cells filled, all sums met & Grid size, runs & CP, DFS/backtracking \\
\midrule
\textbf{KenKen} & Grid; satisfy cages & Row/col uniqueness; cage op & All cages and rows/cols valid & Full valid grid & Grid size, cage ops & CP, backtracking \\
\midrule
\textbf{Nonogram (Picross)} & Grid; shade/clear cells & Row/col run conformity & All run patterns satisfied & All rows/cols consistent & Grid size, run complexity & Line-solver, CP, SAT \\
\midrule
\textbf{Futoshiki} & Grid with $<$/$>$ & Row/col uniqueness; inequalities & All inequalities satisfied & Full valid grid & Grid size, inequality density & CP, backtracking \\
\midrule
\textbf{Slitherlink} & Edges on lattice; toggle edge & Cell numbers match incident edges; degree $\le 2$ & Single loop satisfies all clues & Single simple cycle formed & Grid size, clue density & Loop logic, SAT/ILP \\
\midrule
\textbf{Hashiwokakero (Bridges)} & Bridges between islands & No crossings; degree $\le$ label; $\le2$ parallel & All island degrees match; connected & Degrees matched and single component & \#islands, labels & Graph heur., CP \\
\midrule
\textbf{Nurikabe} & Shade/clear cells & No 2$\times$2 black; island size $\le$ label & Island sizes exact; sea connected & All labels satisfied & Grid size, label layout & CP, BFS/DFS logic \\
\midrule
\textbf{Logic Grid (Zebra)} & Attribute matrix; mark yes/no & No-clash and one-of per category; apply clues & All clues satisfied; bijective assignment & All entities fully assigned & \#entities, attributes & CP, SAT, tableaux \\
\midrule
\textbf{Mastermind} & Color code; propose guess & Score guess with pegs; consistency with history & Min guesses to exact match & Code guessed or guess limit & Code length, colors & Knuth strategy, search \\
\midrule
\textbf{Hitori} & Cells; black/white decisions & No orthogonal black adjacency; track row/col duplicates & No duplicates; white cells connected & All constraints met & Grid size, digit range & CP, DFS/backtracking \\
\midrule
\textbf{Battleships} & Place fleet; mark water & Row/col ship counts; no adjacency; ship shapes & Fleet placed; counts match & All ships placed and valid & Grid size, fleet mix & CP, ILP/backtracking \\
\midrule
\textbf{Hidato} & Place 1..$N$ consecutively & Each $k$ adjacent to $k\!\pm\!1$ & Full 1..$N$ chain & All numbers placed & Grid size, holes & Pathfinding + CP \\
\midrule
\textbf{Tents \& Trees} & Place tents near trees & Tent next to a tree; no tent-tent adjacency; row/col counts & Each tree has 1 tent; counts satisfied & All constraints met & Grid size, tree density & CP, logical heuristics \\
\bottomrule
\end{tabular}
}

\noindent\emph{Abbrev.:} CP = constraint programming; SAT = satisfiability; ILP = integer linear programming; DFS = depth-first search.

\subsection{Simulator-based tasks}
\label{appendix:simulator-tasks}

\scalebox{0.54}{
\begin{tabular}{p{2.5cm} p{4.0cm} p{3.1cm} p{3.1cm} p{2.6cm} p{3.1cm} p{3.1cm}}
\toprule
\textbf{Task} & \textbf{State / Actions} & \textbf{D1: Checker (step)} & \textbf{D2: Evaluator (final)} & \textbf{D3: Stop} & \textbf{D4: Scale ($\alpha$)} & \textbf{Typical Solver} \\
\midrule
\textbf{Comb. circuit debug/synth.} & Netlist; add/remove gate; rewire; change cell & Width/type compat.; one driver/net; no floating pins; no comb.\ cycles & Verilog TB sim; score=\#tests passed / pass--fail & All tests pass / no improv. & \#gates/\#nets; depth; \#vectors & Verilator / Icarus + static checks \\
\midrule
\textbf{Sequential circuit (bounded)} & Netlist with FFs; edit regs/wires/modules & As left + single-clock; reset well-formed; no multiply-driven state & Simulate $T$ cycles; compare trace vs.\ spec & Spec satisfied / cycle budget & \#FFs+\#gates; $T$; \#vectors & Verilator (+ STA) \\
\midrule
\textbf{FSM synthesis (DFA/\allowbreak Mealy/\allowbreak Moore)} & Add states/transitions; label start/accept & Determinism per symbol; well-typed I/O; reachable start & Run labeled traces; score=acc./coverage & All traces satisfied / budget & \#states; alphabet; \#traces$\times$len & Automata simulator \\
\midrule
\textbf{Sorting network (bounded)} & Append comparators $(i,j)$ on $n$ wires & Indices in range; level constraints; well-formed net & Poly test suite; score=\#inputs sorted & Suite fully sorted / budget & \#wires; \#comparators; suite size & SN simulator / PBT harness \\
\midrule
\textbf{CA target synthesis (Life)} & Toggle seed cells; optional pieces & In-bounds; edit budget & Evolve $T$; score=$-\mathrm{dist}(\text{target})$ / exact & Exact match / step budget & Grid size; $T$ & CA engine \\
\midrule
\textbf{Grid-world robot plan} & Append primitives (move/pick/drop) & Parseable; preconds (in-bounds; no collision) & Replay sim; reward=goal $-$ path cost & Goal or horizon & Map size/obstacles; horizon & 2-D grid sim \\
\midrule
\textbf{Compiler pass ordering (IR)} & Sequence optimization passes & Pass applicability; IR verifies/compiles & Run tests; obj.=runtime/size with correct outputs & Tests pass \& no further gain / pass limit & IR size; \#tests; pass budget & LLVM \texttt{opt} + interpreter \\
\midrule
\textbf{Network routing (bounded DES)} & Add/assign routes; flow splits & No simple cycles; link caps. & Simulate $T$; throughput/delay/feas. & All flows delivered or $T$ & \#nodes/\#edges/\#flows; $T$ & Lightweight ns-3--style \\
\bottomrule
\end{tabular}
}

\clearpage

\section{Additional details on the generation process}
\label{sec:appendix:tasks}

This appendix describes the instance-generation and prompt-generation pipeline used by the experiments. Each generated example is stored as a JSON object with an \texttt{id}, a \texttt{category}, a natural-language \texttt{instruction}, a structured \texttt{state}, and an \texttt{answer}. For the optimization domains, the \texttt{state} is the Markovian state used by the rollout code; the \texttt{answer} is the scalar optimum or target value when it is available from the generator. During search, the model never directly edits the full solution object. Instead, it emits a JSON action under the key \texttt{answer}; the parser extracts this action, the domain checker validates it, and the domain transition function applies it to produce the next state.

\paragraph*{Common rollout interface}
All task families are implemented through the same plan-based interface. A domain specification defines:
\begin{enumerate}[leftmargin=0.7cm,itemsep=1pt,parsep=0pt,topsep=2pt]
  \item a prompt function that converts the current structured state into a natural-language prompt;
  \item a JSON action schema, specified by the required action keys;
  \item a step validator, which implements $\texttt{chk}_t(s,a)$;
  \item a state-transition rule, which applies a valid action to the current state;
  \item a terminal-state predicate;
  \item an objective function, which implements $\texttt{eval}_t(\cdot)$ on terminal states.
\end{enumerate}
The model is asked to provide a short reasoning trace and then exactly one action. In the default experiments, the action is parsed from
\[
\texttt{\{"answer": [\{...\}]\}} .
\]
The parser also records whether the response was valid JSON, whether it contained the required keys, and whether the parsed action was domain-feasible.

\paragraph*{Curriculum and levels}
For each task family, the generator exposes four main complexity levels, denoted $\alpha_1,\ldots,\alpha_4$. Higher levels increase the number of variables, the number of choices per step, the amount of clutter or constraints, or the horizon length. The original experiments, excluding the four-action control in Exp.~\ref{task:credit}, generate $1{,}000$ training instances per level. For their solver-prefix online RL setting, the configured dataset selects one state from each optimal solver trajectory, and orders these prompts by level, from $\alpha_1$ to $\alpha_4$. For offline RL, search is run from the root of each instance, the best discovered trajectory is selected using the outcome evaluator, and SFT examples are collected from the states along that trajectory.

\paragraph*{Search-time variants}
For the offline ablation, the same generated instances are evaluated under several search configurations. The full configuration uses both feasibility pruning and duplicate-action merging. The no-pruning configuration permits infeasible intermediate actions to remain in the tree until terminal checking. The no-deduplication configuration keeps duplicate textual proposals even when they parse to the same structured action. The sequential baseline samples rollouts without tree reuse. The solver-reference configuration is an enumerative structural baseline: it enumerates broad next-step candidates directly from the state, rather than using LLM proposals, and is used as a reference trajectory generator in the search diagnostics.

\subsection{0--1 Knapsack}
\label{app:knapsack}

\paragraph*{Problem description}
The knapsack task is an add-only $0$--$1$ knapsack problem. The state contains a capacity, item weights, item values, and the currently selected item set. At each step the model chooses one previously unselected item to add. A step is feasible if the new total weight does not exceed the capacity. The rollout terminates when no remaining item can be legally added.

\paragraph*{Action schema}
The required action key is
\[
\texttt{item\_index}.
\]
The transition appends this item to \texttt{selected} and sorts the selected set. The canonical solution key is the sorted tuple of selected item indices.

\paragraph*{Objective and oracle}
The objective is the total value of the selected items, which is maximized. The generator stores an exact optimum when available. For a partial state, the oracle can compute the best completion value using the knapsack solver, while the offline search setting only uses terminal objective values discovered by rollouts.

\subsection{Quadratic Assignment / CF--EU Manhattan}
\label{app:qap-cfeu}

\paragraph*{Problem description}
The QAP instances use a Manhattan grid. Facilities are assigned one-to-one to grid cells. Facilities have cluster labels, and the pairwise flow between two facilities is high when they are in the same cluster and low otherwise. The objective is
\[
\sum_{i<j} \mathrm{flow}(i,j)\,\|\mathrm{loc}(i)-\mathrm{loc}(j)\|_1 ,
\]
which is minimized.

\paragraph*{Action schema}
The required action keys are
\[
\texttt{facility}, \qquad \texttt{location}.
\]
At each step, the model places exactly one currently unassigned facility into one free grid cell. The transition appends the pair
\[
\texttt{\{"facility": f, "location": [r,c]\}}
\]
to the assignment and sorts the assignment by facility id. The canonical solution key is the ordered list of facility-location pairs.

\paragraph*{Generation process}
For each instance, the generator samples the number of facilities, grid size, cluster map, and a small partial assignment. The remaining unassigned facilities and unoccupied grid cells define the search space. The oracle computes the exact optimal completion by enumerating bijective assignments of the remaining facilities to free cells.

\paragraph*{Complexity levels}
The QAP levels used in the experiments increase the grid size and the number of facilities:
\begin{center}
\begin{tabular}{@{}lcccc@{}}
\toprule
Level & facilities $n$ & grid size & clusters & preassigned fraction \\
\midrule
$\alpha_1$ & 3 & $4\times4$ & 2 & 0.25 \\
$\alpha_2$ & 3 & $5\times5$ & 2 & 0.25 \\
$\alpha_3$ & 4 & $5\times5$ & 2 & 0.25 \\
$\alpha_4$ & 4 & $6\times6$ & 2 & 0.25 \\
\bottomrule
\end{tabular}
\end{center}

\subsection{Role Assignment with Conflicts}
\label{app:role-assignment-conflicts}

\paragraph*{Problem description}
The role-assignment task assigns $R$ roles to $C=R+E$ candidates. Each candidate-role pair has an integer fit score. Some candidate pairs have conflict penalties. A complete assignment must use each role exactly once and each candidate at most once. Conflicts are allowed, but their penalties are subtracted from the final score.

\paragraph*{Action schema}
The required action keys are
\[
\texttt{role}, \qquad \texttt{candidate}.
\]
At each step, the model assigns one unfilled role to one unused candidate. The transition appends the pair
\[
\texttt{\{"role": r, "candidate": c\}}
\]
to the assignment. The canonical solution key is the sorted tuple of role-candidate pairs.

\paragraph*{Objective and tie-break}
The raw objective is
\[
\sum_{(r,c)} \mathrm{fit}(c,r)
-
\sum_{\{c_i,c_j\}\subseteq S} \mathrm{penalty}(c_i,c_j),
\]
where $S$ is the set of selected candidates. Ties are resolved by preferring the complete assignment with the largest minimum individual fit among its selected role-candidate pairs.

\paragraph*{Generation process}
For each level, the generator samples a fit matrix, injects high-fit standouts and near-ties, samples structured conflicts, and rejects or resamples instances whose global optimum is not positive. The oracle enumerates all feasible one-to-one role-candidate assignments and selects the best assignment under the objective and tie-break.

\paragraph*{Complexity levels}
\begin{center}
\begin{tabular}{@{}lcccccc@{}}
\toprule
Level & roles $R$ & extras $E$ & candidates $C$ & conflict density & penalty range & structure \\
\midrule
$\alpha_1$ & 3 & 1 & 4 & 0.15 & [1,5] & chain \\
$\alpha_2$ & 4 & 1 & 5 & 0.20 & [2,5] & random \\
$\alpha_3$ & 5 & 1 & 6 & 0.25 & [2,6] & chain overlaps \\
$\alpha_4$ & 6 & 1 & 7 & 0.35 & [3,6] & clique-biased \\
\bottomrule
\end{tabular}
\end{center}

\subsection{MaxSat}
\label{app:maxsat}

\paragraph*{Problem description}
The MaxSat task combines constrained task selection with worker assignment. The model must choose a subset of tasks and assign each selected task to a distinct eligible worker. The state contains task costs, resource budgets, eligible workers, hard Boolean clauses, and weighted soft clauses.

\paragraph*{Action schema}
The required action keys are
\[
\texttt{task\_index}, \qquad \texttt{worker\_index}.
\]
At each step, the model selects one currently unselected task and assigns it to one unused eligible worker. The transition adds the task to \texttt{selected} and appends
\[
\texttt{\{"task": t, "worker": w\}}
\]
to the assignment list.

\paragraph*{Feasibility}
A step is feasible only if:
\begin{enumerate}[leftmargin=0.7cm,itemsep=1pt,parsep=0pt,topsep=2pt]
  \item the task has not already been selected;
  \item the worker has not already been used;
  \item the worker is eligible for the task;
  \item all resource budgets remain satisfied;
  \item all hard Boolean clauses remain satisfied.
\end{enumerate}
The rollout terminates when no additional task-worker pair can be added without violating these conditions.

\paragraph*{Objective and tie-break}
The learning objective is to maximize the total weight of satisfied soft clauses. For oracle reporting, ties are broken lexicographically: first minimize resource usage in the fixed resource order, then prefer fewer selected tasks.

\paragraph*{Generation process}
The generator samples task costs, budgets, worker eligibility, hard clauses, and soft clauses. Hard clauses include implications, anti-pairs, and small at-most-one groups. Soft clauses include units, implications, anti-pairs, and at-least-one clauses with small integer weights. The oracle enumerates task subsets, checks budget and hard-clause feasibility, verifies that a matching to workers exists, and then scores the solution by the soft-clause objective and tie-breaks.

\paragraph*{Complexity levels}
\begin{center}
\setlength{\tabcolsep}{5pt}
\begin{tabular}{@{}lccccccc@{}}
\toprule
Level & tasks & workers & resources & budget factor & eligibility & hard density & soft density \\
\midrule
$\alpha_1$ & 4 & 2 & 1 & 0.50 & random & 0.30 & 1.20 \\
$\alpha_2$ & 5 & 3 & 1 & 0.55 & random & 0.60 & 1.60 \\
$\alpha_3$ & 6 & 3 & 2 & 0.55 & random & 0.60 & 1.70 \\
$\alpha_4$ & 7 & 4 & 2 & 0.55 & role-patterned & 0.70 & 2.00 \\
\bottomrule
\end{tabular}
\end{center}

\subsection{Single-machine weighted tardiness scheduling}
\label{app:single-machine-scheduling}

\paragraph*{Problem description}
The scheduling task is a single-machine sequencing problem. Each job $j$ has processing time $p_j$, due date $d_j$, and weight $w_j$. The model builds a schedule one job at a time. The machine starts at time $0$, runs without idle time, and jobs are non-preemptive.

\paragraph*{Action schema}
The required action key is
\[
\texttt{job\_index}.
\]
At each step, the model appends exactly one unscheduled job to the current prefix. The transition appends this job index to the order. The rollout terminates when all jobs have been scheduled.

\paragraph*{Objective}
For a complete order, the objective is total weighted tardiness:
\[
\min \sum_j w_j T_j,
\qquad
T_j=\max\{0,C_j-d_j\}.
\]
The domain uses a minimization objective, while the generic search code internally converts objectives into maximization form when needed.

\paragraph*{Generation process}
For each instance, the generator samples the number of jobs, processing times, due dates, and weights. Due dates are sampled using a tardiness-factor and relative-due-date-range recipe. The oracle solves the optimal suffix exactly from a partial prefix and emits a step-by-step path of next jobs.

\paragraph*{Complexity levels}
\begin{center}
\begin{tabular}{@{}lccccc@{}}
\toprule
Level & jobs $n$ & processing range & TF & RDD & max weight \\
\midrule
$\alpha_1$ & 5 & [1,7]  & 0.30 & 0.60 & 3 \\
$\alpha_2$ & 6 & [1,7]  & 0.50 & 0.60 & 5 \\
$\alpha_3$ & 7 & [2,8]  & 0.70 & 0.60 & 10 \\
$\alpha_4$ & 7 & [1,15] & 0.60 & 0.60 & 6 \\
\bottomrule
\end{tabular}
\end{center}

\subsection{\texorpdfstring{Polyomino Target Cover}{Polyomino Target Cover}}
\label{app:polyomino-target-cover}

\paragraph*{Problem description}
The Polyomino Target Cover task is a spatial optimization problem. The board contains target cells \texttt{t}, optional obstacles \texttt{\#}, and possibly pre-placed example pieces. The model has a pool of labeled pieces and a move budget $K$. At each step, it chooses an unused piece, rotates it by $0^\circ$, $90^\circ$, $180^\circ$, or $270^\circ$, and places it at a top-left anchor. The goal is to maximize the number of newly covered targets within the move budget.

\paragraph*{Action schema}
The action contains:
\[
\texttt{piece\_id}, \qquad \texttt{anchor}, \qquad \texttt{rotation}, \qquad \texttt{grid\_after}.
\]
The \texttt{anchor} is the global row-column coordinate of the top-left cell of the transformed piece's tight bounding box. The \texttt{grid\_after} field is included so that the parser and checker can compare the declared placement with the resulting board.

\paragraph*{Feasibility}
A placement is feasible only if all occupied cells lie inside the board and do not overlap obstacles or existing letters. A piece may be used at most once. Reflections are not used in this task.

\paragraph*{Generation process}
The generator first samples a board size, target clusters, singleton target cells, and obstacles. It then samples a piece pool from a level-dependent piece library. Some levels include pre-placed example pieces, which do not count against the move budget. The oracle uses an exact bitboard branch-and-bound search over legal placements and returns the optimal additional target coverage and a decorated path of placements.

\paragraph*{Complexity levels}
\begin{center}
\begin{tabular}{@{}lcccccc@{}}
\toprule
Level & grid & $K$ & pool size & piece set & examples & obstacles \\
\midrule
$\alpha_1$ & $4\times4$ & 1 & 3 & domino-only & 2 & 0 \\
$\alpha_2$ & $5\times5$ & 1 & 3 & domino/square & 2 & 0--1 \\
$\alpha_3$ & $5\times5$ & 2 & 4 & full & 2 & 0--1 \\
$\alpha_4$ & $6\times6$ & 3 & 5 & full & 1 & 0--1 \\
\bottomrule
\end{tabular}
\end{center}

\subsection{Auxiliary Grid-Spatial tasks}
\label{app:spatial-2d-tasks}

\paragraph*{Purpose}
The Grid-Spatial tasks are auxiliary held-out spatial QA tasks used to test whether training on the Polyomino Target Cover optimization task transfers to more elementary spatial reasoning skills. Unlike Polyomino Target Cover, these tasks are \emph{not} multi-step optimization problems. They are mostly one-step questions about moving, rotating, reflecting, comparing, or placing small shapes on a grid. We use them as diagnostic probes for the spatial abilities that are implicitly required by the polyomino optimization task.

\paragraph*{General format}
Each instance contains a small rectangular canvas represented as a Python-style list of lists. Cells may be empty, blocked, or marked as targets:
\[
\texttt{"."}=\text{empty}, \qquad
\texttt{"\#"}=\text{obstacle}, \qquad
\texttt{"t"}=\text{target}.
\]
Some instances also contain pre-placed uppercase letters, which act as occupied cells. The shape is given as a small list-of-lists using a letter such as \texttt{"A"} for occupied cells and \texttt{"."} for empty cells inside the shape's bounding box. Coordinates are always row-column coordinates, with rows and columns indexed from $0$. An anchor \([r,c]\) denotes the top-left position of the transformed shape's tight bounding box after rotation or reflection.

A typical input therefore looks like:
\begin{verbatim}
canvas = [
  [".", ".", ".", ".", "."],
  [".", "#", ".", ".", "."],
  [".", ".", ".", "t", "."],
  [".", ".", ".", ".", "."],
  [".", ".", ".", ".", "."]
]

shape = [
  ["A", "."],
  ["A", "A"]
]

anchor = [2, 1]
operation = rotate 90 degrees clockwise
\end{verbatim}
The model must reason about the transformed coordinates of the occupied cells, check whether the placement is legal, and return the requested answer in JSON form.

\paragraph*{Tasks selected for testing}
The held-out Spatial QA evaluation is organized into the following capability categories. These are the categories reported in the spatial capability tables.

\begin{center}
\begin{tabularx}{\linewidth}{@{}p{3.0cm}L@{}}
\toprule
\textbf{Capability} & \textbf{What the test instance asks} \\
\midrule
\textbf{Translate}
& Move a given shape from one anchor to another without changing its orientation. The model must compute the final occupied cells and usually return the resulting grid or final anchor. \\[2pt]

\textbf{Rotate \& Place}
& Rotate a shape by a specified angle in $\{0^\circ,90^\circ,180^\circ,270^\circ\}$ and place it at a specified anchor. The model must update the shape coordinates after rotation and check whether the placement fits on the canvas. \\[2pt]

\textbf{Reflect \& Place}
& Reflect a shape horizontally, vertically, or diagonally, then place it at a specified anchor. This tests whether the model understands mirror transformations on a discrete grid. \\[2pt]

\textbf{Compose}
& Apply a short composition of transformations, such as rotation followed by reflection followed by translation. These examples test whether the model can keep track of multiple spatial operations in order. \\[2pt]

\textbf{Equivalence}
& Given two shapes, decide whether they are the same up to rotation and, when allowed, reflection. If they are equivalent, the model returns a transformation explaining how one shape maps to the other. \\[2pt]

\textbf{Symmetry}
& Given one shape, report which symmetries it has, such as $90^\circ$ rotation symmetry, $180^\circ$ rotation symmetry, horizontal mirror symmetry, vertical mirror symmetry, or diagonal mirror symmetry. \\[2pt]

\textbf{Overlap / IoU}
& Place two transformed shapes on the same empty canvas and compute their intersection cells, intersection count, and union count. This is a small discrete analogue of intersection-over-union reasoning. \\[2pt]

\textbf{Coverage}
& Given a target cell and a transformed shape, list the anchors for which the shape would cover that target cell. This tests inverse spatial reasoning: instead of asking where a shape lands after an anchor is chosen, the task asks which anchors could have produced coverage of a specific cell. \\
\bottomrule
\end{tabularx}
\end{center}

\paragraph*{Examples of what the selected tasks look like}
Below are simplified examples illustrating the style of the held-out Spatial QA tasks.

\textbf{Translate.}
The prompt gives a shape and an initial/final anchor. The model must shift all occupied cells by the anchor displacement.
\begin{verbatim}
shape = [
  ["A", "."],
  ["A", "A"]
]
old_anchor = [0, 1]
new_anchor = [2, 3]
\end{verbatim}
The expected reasoning is that the occupied local cells of the shape are
\[
(0,0),\ (1,0),\ (1,1).
\]
At anchor \([2,3]\), these become global cells
\[
(2,3),\ (3,3),\ (3,4).
\]

\textbf{Rotate \& Place.}
The prompt gives a shape, a rotation angle, and an anchor.
\begin{verbatim}
shape = [
  ["A", "."],
  ["A", "A"]
]
rotation = 90
anchor = [1, 2]
\end{verbatim}
The model must rotate the shape clockwise, normalize the rotated shape to its tight bounding box, then place that bounding box with top-left cell at \([1,2]\). It must also check that no occupied cell lands outside the board or on an obstacle.

\textbf{Reflect \& Place.}
The prompt gives a reflection mode such as horizontal or vertical.
\begin{verbatim}
shape = [
  ["A", "A", "."],
  [".", "A", "A"]
]
reflection = vertical
anchor = [0, 1]
\end{verbatim}
The model must mirror the shape across the requested axis and then place the reflected shape at the anchor.

\textbf{Compose.}
A compose instance asks the model to apply several transformations in sequence:
\begin{verbatim}
shape = [
  ["A", "."],
  ["A", "A"]
]
operations = rotate 90, then reflect horizontal, then place at [2, 1]
\end{verbatim}
These examples are harder because an error in the first operation changes all later coordinates.

\textbf{Equivalence.}
An equivalence instance gives two shapes:
\begin{verbatim}
shape_1 = [
  ["A", "."],
  ["A", "A"]
]

shape_2 = [
  ["B", "B"],
  ["B", "."]
]
\end{verbatim}
The model must decide whether \texttt{shape\_2} can be obtained from \texttt{shape\_1} by rotation, and possibly reflection. The answer may specify a rotation angle and whether reflection is needed.

\textbf{Symmetry.}
A symmetry instance gives one shape:
\begin{verbatim}
shape = [
  ["A", "A"],
  ["A", "A"]
]
\end{verbatim}
The model must report the transformations that leave the shape unchanged. For the square above, several rotations and mirror symmetries are valid.

\textbf{Overlap / IoU.}
An overlap instance gives two shapes with their own transformations and anchors:
\begin{verbatim}
shape_1 anchor = [1, 1], rotation = 0
shape_2 anchor = [1, 2], rotation = 90
\end{verbatim}
The model must compute the cells occupied by each placed shape, then return the intersection count and union count. This probes whether the model can compare two transformed coordinate sets.

\textbf{Coverage.}
A coverage instance gives a target cell and asks which anchors would make a shape cover that cell:
\begin{verbatim}
target_cell = [2, 3]
shape = [
  ["A", "."],
  ["A", "A"]
]
rotation = 0
\end{verbatim}
The model must reason backward from the target cell to possible anchors. For every occupied local cell \((u,v)\) in the shape, an anchor candidate is
\[
[r,c] = [2-u,\; 3-v].
\]
The legal anchors are those whose resulting placement stays inside the grid and avoids obstacles.

\paragraph*{Difficulty levels used for testing}
The Spatial QA tests use two main held-out difficulty settings. Both use the same task types above, but the second setting increases the grid size, the shape size, and the amount of clutter.

\begin{center}
\small
\begin{tabularx}{\linewidth}{@{}lccccL@{}}
\toprule
\textbf{Difficulty} & \textbf{Canvas} & \textbf{Shape sizes} & \textbf{Obstacles} & \textbf{Decoy letters} & \textbf{Transform variety} \\
\midrule
Diff. 1 & about $5\times5$ & mostly 2--3 cells & 0--1 & 0 & simpler rotations/reflections \\
Diff. 2 & about $6\times6$ & mostly 3--4 cells & 1--2 & 0--1 & full rotations and more reflections \\
\bottomrule
\end{tabularx}
\end{center}

\paragraph*{Why these tasks are useful}
These tasks are designed to isolate the low-level spatial operations needed for the optimization task. Polyomino Target Cover requires the model to choose pieces, rotate them, place them legally, avoid collisions, and reason about target coverage. The auxiliary Grid-Spatial tests separate these abilities into simpler questions. Thus, improved performance on these tests suggests that optimization training is not only teaching the model to imitate a specific polyomino solver, but also improving reusable spatial operations such as coordinate tracking, rotation, reflection, overlap computation, and coverage reasoning.

\clearpage

\section{Prompts (examples)}
\label{app:prompts}

This section gives representative prompts produced by the task promptization functions. The examples show the user-facing prompt, not the hidden checker. In every optimization prompt below, the model is expected to give brief reasoning and then return exactly one JSON action under \texttt{answer}. The action keys shown here match the keys validated by the domain specifications.

\begin{promptbox}{0--1 Knapsack (example)}
You are solving a \textbf{0--1 knapsack} problem \textbf{incrementally}.

\textbf{Objective}\\
Build a high-value feasible set of items. You may only add items. Stop only when no remaining item can be legally added without exceeding capacity.

\textbf{State}
\begin{verbatim}
Capacity = 45

Weights = [
  {"0": 4}, {"1": 18}, {"2": 1}, {"3": 8},
  {"4": 12}, {"5": 22}, {"6": 6}, {"7": 22},
  {"8": 17}, {"9": 19}, {"10": 4}, {"11": 19},
  {"12": 19}, {"13": 16}, {"14": 18}, {"15": 3}
]

Values = [
  {"0": 1}, {"1": 15}, {"2": 1}, {"3": 10},
  {"4": 10}, {"5": 8}, {"6": 5}, {"7": 37},
  {"8": 25}, {"9": 27}, {"10": 5}, {"11": 17},
  {"12": 21}, {"13": 6}, {"14": 15}, {"15": 1}
]

Currently selected item indices = []
Current total weight = 0
Current total value  = 0
\end{verbatim}

\textbf{Rules}
\begin{itemize}\itemsep0pt
  \item Add exactly one item at this step.
  \item Do not add an item that is already selected.
  \item The new total weight must be at most the capacity.
\end{itemize}

Briefly explain the trade-off you are using, such as remaining capacity, value, and value/weight ratio. Then propose the next item to add.

Immediately after reasoning within \texttt{\textless think\textgreater} your reasoning here \texttt{\textless /think\textgreater}, propose \textbf{1} action with the following format:

\begin{tcblisting}{jsonlisting}
{"answer":
[
  {
    "item_index": <int>
  }
]
}
\end{tcblisting}
\end{promptbox}

\begin{promptbox}{CF--EU Manhattan / QAP (example)}
You are solving an \textbf{incremental facility placement} problem.

\textbf{Grid}
\begin{verbatim}
[
  [".", ".", ".", ".", ".", "."],
  [".", ".", ".", ".", ".", "."],
  [".", "1", ".", ".", ".", "."],
  [".", ".", ".", ".", ".", "."],
  [".", ".", ".", ".", ".", "."],
  [".", ".", ".", ".", ".", "."]
]
\end{verbatim}

\textbf{Legend}\\
\texttt{"."} means a free cell. A number means that the corresponding facility is already assigned to that cell.

\textbf{Assigned facilities}
\begin{verbatim}
[
  {"facility": 1, "location": [2, 1]}
]
\end{verbatim}

\textbf{Unassigned facilities}
\begin{verbatim}
[0, 2, 3]
\end{verbatim}

\textbf{Cluster map}
\begin{verbatim}
{"0": 0, "1": 1, "2": 1, "3": 1}
\end{verbatim}

\textbf{Flow rule}\\
Facilities in the same cluster have flow $H=10$; facilities in different clusters have flow $L=1$.

\textbf{Distance metric}\\
Manhattan distance.

\textbf{Objective}\\
Complete the assignment while minimizing
\[
\sum_{i<j}\mathrm{flow}(i,j)\times
\mathrm{Manhattan}(\mathrm{location}(i),\mathrm{location}(j)).
\]

\textbf{Rules}
\begin{itemize}\itemsep0pt
  \item Add exactly one unassigned facility.
  \item Place it in a currently free grid cell.
  \item A grid cell may contain at most one facility.
\end{itemize}

Briefly explain the placement reasoning, especially cluster and flow effects. Then propose the next placement.

Immediately after reasoning within \texttt{\textless think\textgreater} your reasoning here \texttt{\textless /think\textgreater}, propose \textbf{1} action with the following format:

\begin{tcblisting}{jsonlisting}
{"answer":
[
  {
    "facility": <int>,
    "location": [<row>, <col>]
  }
]
}
\end{tcblisting}
\end{promptbox}

\begin{promptbox}{Role Assignment with Conflicts (example)}
You are solving a \textbf{role assignment with conflicts} problem \textbf{incrementally}.

\textbf{Objective}
\begin{enumerate}
  \item Assign exactly one candidate to each role.
  \item Each candidate can be used at most once.
  \item Maximize
  \[
  \Big(\sum \text{chosen fit scores}\Big)
  -
  \Big(\sum \text{conflict penalties among selected candidates}\Big).
  \]
  \item Tie-breaker: among equal total scores, prefer the assignment with the highest minimum individual fit.
\end{enumerate}

\textbf{Current state}
\begin{itemize}\itemsep0pt
  \item Number of roles: \texttt{6}
  \item Number of candidates: \texttt{7}
  \item Already assigned role-candidate pairs: \texttt{[]}
\end{itemize}

\textbf{Fit matrix}\\
Rows are candidates and columns are roles.
\begin{verbatim}
candidate 0: role scores {0: 9, 1: 9, 2: 1, 3: 3, 4: 8, 5: 3}
candidate 1: role scores {0: 6, 1: 4, 2: 3, 3: 1, 4: 8, 5: 1}
candidate 2: role scores {0: 1, 1: 0, 2: 1, 3: 7, 4: 9, 5: 5}
candidate 3: role scores {0: 3, 1: 8, 2: 5, 3: 3, 4: 1, 5: 3}
candidate 4: role scores {0: 6, 1: 2, 2: 6, 3: 8, 4: 1, 5: 0}
candidate 5: role scores {0: 8, 1: 7, 2: 4, 3: 1, 4: 8, 5: 3}
candidate 6: role scores {0: 4, 1: 6, 2: 8, 3: 9, 4: 1, 5: 1}
\end{verbatim}

\textbf{Conflict penalties}
\begin{verbatim}
(0,6): 3
(0,5): 4
(0,2): 6
(0,1): 5
(4,5): 4
(5,6): 3
(1,5): 5
\end{verbatim}

\textbf{Rules}
\begin{itemize}\itemsep0pt
  \item Add exactly one new role-candidate assignment.
  \item The role must be currently unfilled.
  \item The candidate must be currently unused.
  \item Conflicts are allowed, but their penalties are subtracted in the objective.
\end{itemize}

Briefly explain the fit-vs-conflict trade-off and the tie-breaker. Then propose the next assignment.

Immediately after reasoning within \texttt{\textless think\textgreater} your reasoning here \texttt{\textless /think\textgreater}, propose \textbf{1} action with the following format:

\begin{tcblisting}{jsonlisting}
{"answer":
[
  {
    "role": <int>,
    "candidate": <int>
  }
]
}
\end{tcblisting}
\end{promptbox}

\begin{promptbox}{Wished Assignments / Constrained MaxSAT (example)}
You are solving a \textbf{Wished Assignments} puzzle \textbf{incrementally}. This is a constrained task-selection and worker-assignment problem.

\textbf{Objective}
\begin{enumerate}
  \item Maximize the total weight of satisfied soft clauses, subject to budgets, worker eligibility, and hard logic.
  \item Tie-breaker A: among equal soft-clause scores, minimize resource usage lexicographically.
  \item Tie-breaker B: among remaining ties, prefer fewer selected tasks.
\end{enumerate}

\textbf{Budgets}
\[
\text{Money}=12, \qquad \text{Time}=11 .
\]

\textbf{Workers}
\begin{verbatim}
worker_index 0: W1
worker_index 1: W2
worker_index 2: W3
worker_index 3: W4
\end{verbatim}

\textbf{Tasks}
\begin{verbatim}
task_index 0, A: costs {Money: 2, Time: 1}, eligible workers [2]
task_index 1, B: costs {Money: 2, Time: 4}, eligible workers [1]
task_index 2, C: costs {Money: 4, Time: 3}, eligible workers [0]
task_index 3, D: costs {Money: 5, Time: 4}, eligible workers [1, 3]
task_index 4, E: costs {Money: 1, Time: 2}, eligible workers [0, 2]
task_index 5, F: costs {Money: 5, Time: 3}, eligible workers [1, 3]
task_index 6, G: costs {Money: 2, Time: 3}, eligible workers [0, 2]
\end{verbatim}

\textbf{Hard clauses}\\
Here $\lnot$ means negation.
\begin{itemize}\itemsep0pt
  \item $\lnot A \vee \lnot C$
  \item $\lnot A \vee G$
  \item $\lnot C \vee \lnot D$
  \item $\lnot C \vee \lnot F$
  \item $\lnot D \vee \lnot F$
\end{itemize}

\textbf{Soft clauses}
\begin{enumerate}\itemsep0pt
  \item $D \vee \lnot F$ \hfill $(w=1)$
  \item $C \vee G$ \hfill $(w=1)$
  \item $\lnot B$ \hfill $(w=1)$
  \item $\lnot D \vee \lnot F$ \hfill $(w=1)$
  \item $\lnot C \vee F$ \hfill $(w=1)$
  \item $\lnot A \vee \lnot E$ \hfill $(w=2)$
  \item $B \vee C$ \hfill $(w=2)$
  \item $\lnot B \vee C$ \hfill $(w=1)$
  \item $\lnot B \vee G$ \hfill $(w=2)$
  \item $\lnot D \vee \lnot F$ \hfill $(w=1)$
  \item $\lnot D \vee E$ \hfill $(w=1)$
  \item $G \vee C$ \hfill $(w=2)$
  \item $\lnot C \vee B$ \hfill $(w=2)$
  \item $B$ \hfill $(w=2)$
\end{enumerate}

\textbf{Current state}
\begin{verbatim}
Selected tasks = []
Assigned task-worker pairs = []
\end{verbatim}

\textbf{Rules}
\begin{itemize}\itemsep0pt
  \item Add exactly one task and assign it to exactly one worker.
  \item The task must be currently unselected.
  \item The worker must be currently unused.
  \item The worker must be eligible for the task.
  \item The new selected set must satisfy all hard clauses.
  \item The new resource usage must not exceed the budgets.
\end{itemize}

Briefly explain budget feasibility, hard-clause feasibility, worker eligibility, and the soft-clause effect. Then propose the next task-worker assignment.

Immediately after reasoning within \texttt{\textless think\textgreater} your reasoning here \texttt{\textless /think\textgreater}, propose \textbf{1} action with the following format:

\begin{tcblisting}{jsonlisting}
{"answer":
[
  {
    "task_index": <int>,
    "worker_index": <int>
  }
]
}
\end{tcblisting}
\end{promptbox}

\begin{promptbox}{Single-machine Weighted Tardiness Scheduling (example)}
You are solving a \textbf{single-machine weighted tardiness scheduling} problem \textbf{incrementally}.

\textbf{Objective}
\[
\min \sum_j w_j \cdot \max(0, C_j-d_j).
\]

\textbf{Rules}
\begin{itemize}\itemsep0pt
  \item The machine starts at time $0$.
  \item There is no inserted idle time.
  \item Jobs are non-preemptive.
  \item At this step, append exactly one unscheduled job to the current prefix.
\end{itemize}

\textbf{Jobs}
\begin{center}
\begin{tabular}{r|r|r|r|r}
Id & job\_index & $p$ & $d$ & $w$ \\
\hline
A & 0 & 3  & 9  & 4 \\
B & 1 & 8  & 29 & 4 \\
C & 2 & 11 & 22 & 5 \\
D & 3 & 6  & 14 & 1 \\
E & 4 & 15 & 28 & 2 \\
F & 5 & 7  & 30 & 2 \\
G & 6 & 7  & 19 & 5 \\
\end{tabular}
\end{center}

\textbf{Current state}
\begin{verbatim}
Current partial order = []
Current partial order indices = []
Time elapsed so far = 0
Remaining job indices = [0, 1, 2, 3, 4, 5, 6]
\end{verbatim}

Briefly explain your scheduling rationale, such as due dates, processing times, weights, and whether a job risks becoming tardy. Then choose the next job to append.

Immediately after reasoning within \texttt{\textless think\textgreater} your reasoning here \texttt{\textless /think\textgreater}, propose \textbf{1} action with the following format:

\begin{tcblisting}{jsonlisting}
{"answer":
[
  {
    "job_index": <int>
  }
]
}
\end{tcblisting}
\end{promptbox}

\begin{promptbox}{Polyomino Target Cover (example)}
You are playing a \textbf{Polyomino Target Cover} puzzle \textbf{incrementally}.

\textbf{Objective}\\
Maximize the number of target cells covered by placed pieces within the move budget.

\textbf{Constraints}
\begin{itemize}\itemsep0pt
  \item Choose an unused piece.
  \item Rotate it by $0^\circ$, $90^\circ$, $180^\circ$, or $270^\circ$ clockwise.
  \item Place the transformed shape using the top-left anchor of its tight bounding box.
  \item All occupied cells must lie inside the board.
  \item Occupied cells may not overlap \texttt{"\#"} or existing piece letters.
\end{itemize}

\textbf{Budget}
\[
K=3 \quad \text{moves total; current round } 0.
\]

\textbf{Board before example placement}
\begin{verbatim}
initial_grid = [
  [".", ".", ".", ".", ".", "."],
  [".", ".", ".", ".", ".", "."],
  [".", ".", ".", ".", ".", "."],
  [".", ".", ".", ".", ".", "."],
  [".", ".", ".", ".", ".", "."],
  [".", ".", ".", ".", ".", "."]
]
\end{verbatim}

\textbf{Target locations}
\begin{verbatim}
[(0,0), (0,1), (1,0), (1,1), (0,2), (1,2), (2,2),
 (0,3), (1,3), (1,4), (2,3)]
\end{verbatim}

\textbf{Current board after example piece placement}
\begin{verbatim}
current_grid = [
  [".", ".", ".", ".", "A", "A"],
  [".", ".", ".", ".", "A", "A"],
  [".", ".", ".", ".", "A", "."],
  [".", ".", ".", ".", ".", "."],
  [".", ".", ".", ".", ".", "."],
  [".", ".", ".", ".", ".", "."]
]
\end{verbatim}

\textbf{Piece pool}
\begin{verbatim}
A (P5), already placed at anchor [0,4], rotation 0:
[
  ["A", "A"],
  ["A", "A"],
  ["A", "."]
]

B (L4), unused:
[
  ["B", "."],
  ["B", "."],
  ["B", "B"]
]

C (Z4), unused:
[
  ["C", "C", "."],
  [".", "C", "C"]
]

D (O), unused:
[
  ["D", "D"],
  ["D", "D"]
]

E (R6), unused:
[
  ["E", "E", "E"],
  ["E", "E", "E"]
]
\end{verbatim}

Briefly explain the placement reasoning, including feasibility and newly covered targets. Then return exactly one placement.

Immediately after reasoning within \texttt{\textless think\textgreater} your reasoning here \texttt{\textless /think\textgreater}, propose \textbf{1} action with the following format:

\begin{tcblisting}{jsonlisting}
{"answer":
[
  {
    "piece_id": "<ID>",
    "anchor": [<row>, <col>],
    "rotation": <0|90|180|270>,
    "grid_after": [
      [".", ".", "..."],
      ["...", "...", "..."]
    ]
  }
]
}
\end{tcblisting}
\end{promptbox}

\begin{promptbox}{Auxiliary Grid-Spatial QA (example)}
You are solving a \textbf{single-step grid-spatial reasoning} task.

\textbf{Task name}\\
\texttt{rotate\_place}

\textbf{Canvas}
\begin{verbatim}
[
  [".", ".", ".", ".", "."],
  [".", "#", ".", ".", "."],
  [".", ".", ".", "t", "."],
  [".", ".", ".", ".", "."],
  [".", ".", ".", ".", "."]
]
\end{verbatim}

\textbf{Shape}
\begin{verbatim}
[
  ["A", "."],
  ["A", "A"]
]
\end{verbatim}

\textbf{Operation}\\
Rotate the shape by $90^\circ$ clockwise and place it with global anchor \texttt{[2, 1]}.

\textbf{Rules}
\begin{itemize}\itemsep0pt
  \item The anchor is the top-left coordinate of the transformed shape's tight bounding box.
  \item The placement must stay inside the canvas.
  \item The placement must avoid \texttt{"\#"} and existing letters.
\end{itemize}

Return the transformed placement and the resulting grid.

Immediately after reasoning within \texttt{\textless think\textgreater} your reasoning here \texttt{\textless /think\textgreater}, propose the answer with the following format:

\begin{tcblisting}{jsonlisting}
{"answer":
[
  {
    "anchor": [<row>, <col>],
    "rotation": 90,
    "legal": <true|false>,
    "solution_grid": [
      [".", ".", "..."],
      ["...", "...", "..."]
    ]
  }
]
}
\end{tcblisting}
\end{promptbox}

\clearpage

\section{Polyomino examples}
\label{app:polyomino-examples}

\begin{figure}[h]
  \centering

  \begin{subfigure}{0.48\textwidth}
    \centering
    \includegraphics[width=\linewidth]{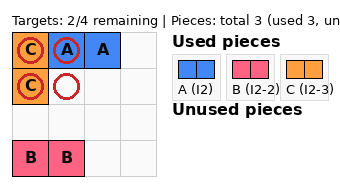}
    \caption{Level 1}
  \end{subfigure}\hfill
  \begin{subfigure}{0.48\textwidth}
    \centering
    \includegraphics[width=\linewidth]{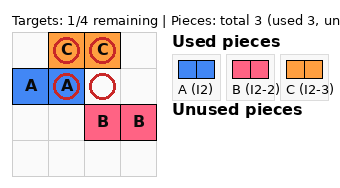}
    \caption{Level 1}
  \end{subfigure}

  \medskip

  \begin{subfigure}{0.48\textwidth}
    \centering
    \includegraphics[width=\linewidth]{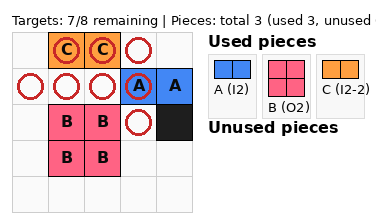}
    \caption{Level 2}
  \end{subfigure}\hfill
  \begin{subfigure}{0.48\textwidth}
    \centering
    \includegraphics[width=\linewidth]{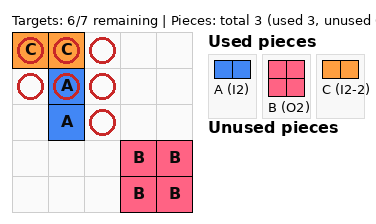}
    \caption{Level 2}
  \end{subfigure}

  \medskip

  \begin{subfigure}{0.48\textwidth}
    \centering
    \includegraphics[width=\linewidth]{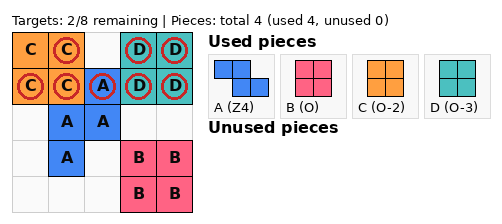}
    \caption{Level 3}
  \end{subfigure}\hfill
  \begin{subfigure}{0.48\textwidth}
    \centering
    \includegraphics[width=\linewidth]{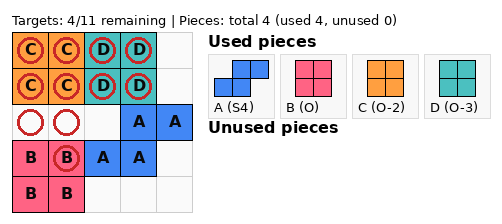}
    \caption{Level 3}
  \end{subfigure}

  \medskip

  \begin{subfigure}{0.48\textwidth}
    \centering
    \includegraphics[width=\linewidth]{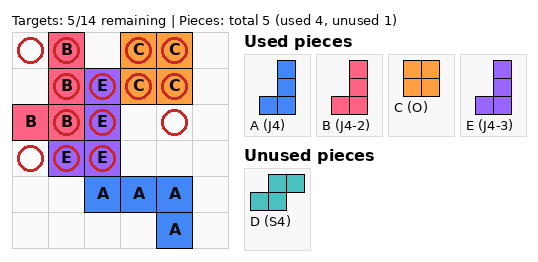}
    \caption{Level 4}
  \end{subfigure}\hfill
  \begin{subfigure}{0.48\textwidth}
    \centering
    \includegraphics[width=\linewidth]{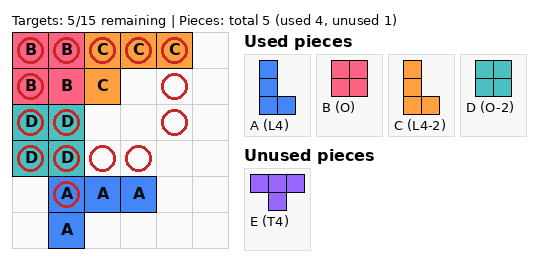}
    \caption{Level 4}
  \end{subfigure}

  \caption{Polyomino target cover examples. This shows the optimal solution found by the solver. The first three levels have two randomly positioned pieces that serve as examples. The last level has only one example.}
  \label{fig:polyomino-grid}
\end{figure}

\clearpage

\section{Additional Details}
\label{sec:models-tasks}
\subsection{Search Ablations}
\label{app:numcts-details}
\label{sec:search-ablations}

\paragraph{Task representation.}
We evaluate search methods for structured combinatorial optimization problems. The experiments use four task families: \texttt{qap}, \texttt{knapsack}, MaxSat, and \texttt{role\_assignment}. For each family, we evaluate difficulty levels 1--4.  Each instance is treated as a partial construction problem: a node in the search tree stores the current partial state, and a child node corresponds to one proposed next action.  The language model proposes actions in a structured format, which is then parsed into a domain state.  The domain checker verifies whether the output is valid JSON, whether it is structurally valid for the task, whether the partial path is feasible, and whether the node is a terminal answer.  Terminal feasible answers are scored by their normalized objective value, with all objectives converted to a maximization convention.

\paragraph{Expansion and filtering.}
When a model-based search node is expanded, the language model proposes \(16\) candidate next actions.  Each candidate receives an edge log-probability from the model.  We use the average token log-probability both as an edge score and, after exponentiation, as the prior \(P_i\).  Invalid parses are discarded.  Depending on the strategy, domain-infeasible partial actions may either be pruned immediately or allowed to remain in the tree until terminal checking.  When redundancy removal is enabled, children with the same canonical solution/state key are merged so that duplicate proposals do not occupy multiple search slots.  Terminal feasibility is always enforced in these ablations: infeasible final answers receive a negative reward rather than being counted as successful solutions.

\paragraph{Standard MCTS.}
The main method is Monte Carlo Tree Search (MCTS).  Each MCTS rollout starts at the root, repeatedly selects a child, expands an unexpanded node, and backpropagates the terminal reward.  Child selection uses a PUCT-style rule:
\[
  \mathrm{score}(i)
  =
  Q_i^{\mathrm{blend}}
  +
  c_{\mathrm{puct}} P_i
  \frac{\sqrt{N_{\mathrm{parent}}}}{1 + N_i + N_i^{\mathrm{bad}}},
\]
where \(Q_i^{\mathrm{blend}}\) is the value used for selection, \(P_i\) is the model prior, \(N_i\) is the child visit count, and \(N_i^{\mathrm{bad}}\) counts visits through incorrect or infeasible partial paths.  The exploration coefficient is \(c_{\mathrm{puct}}=5\).  Feasible terminal solutions contribute objective-based rewards to the path.  Incorrect or invalid final states receive reward \(-1\).  Depth-limited non-answer rollouts receive reward \(0\) and are down-weighted in the value estimate by the depth-discount factor \(0.25\).

\paragraph{Parent-initial-value MCTS.}
The parent-initial-value variant uses the same search procedure as the standard MCTS ablation: it still performs rollouts, uses the PUCT selection formula, maintains visit counts, and backpropagates rewards.  The difference is only in the value term used by PUCT before enough rewards have been observed.  After a node is expanded, each child receives an initial value \(I_i\) from the model score of that child.  The default source is the average token log-probability.  The parent also stores the average initial value of its children.  The value term becomes
\[
  Q_i^{\mathrm{blend}} = \lambda Q_i + (1-\lambda) I_i.
\]
In the standard MCTS ablation, \(\lambda=1.0\), so the initial value is ignored and PUCT uses the observed MCTS value \(Q_i\).  In the parent-initial-value ablation, \(\lambda=0.5\), so the selection value is an equal blend of the observed MCTS value and the model-derived initial value.  The scale parameter is \(\mu=1.0\), meaning the stored initial value is used directly.  This variant is therefore best understood as MCTS with a model-score warm start, not as beam search.

\paragraph{Beam search ablation.}
Beam search is different from the parent-initial-value MCTS run.  Instead of running PUCT rollouts and backpropagating rewards, beam search proceeds layer by layer.  At each depth, it expands every node currently in the beam, collects all non-terminal children, ranks them by cumulative path log-probability, and keeps only the top 4 nodes for the next depth.  The beam score is simply the path log-probability, i.e., the sum of edge average log-probabilities along the partial solution.  The beam run still uses the same parsing, feasibility-pruning, redundancy-removal, and terminal-feasibility options as the corresponding strategy preset, but it does not use PUCT visit counts, PUCT backpropagation, or the parent-initial-value blend.  Unlike MCTS, the beam computation is controlled by maximum depth and beam width rather than by a rollout budget.

\begin{table}[t]
\centering
\small
\caption{Difference between the standard MCTS run, the parent-initial-value MCTS run, and the beam-search ablation.}
\label{tab:search-mode-comparison}
\begin{tabular}{p{0.22\linewidth}p{0.23\linewidth}p{0.43\linewidth}}
\hline
Ablation setting & Search mode & Main behavior \\
\hline
Standard ablation & MCTS, sequential baseline, solver reference & Uses the full strategy set.  \texttt{S1}--\texttt{S3} are MCTS variants, \texttt{S4} is a sequential non-MCTS baseline, and \texttt{S5} is the solver-reference run.  MCTS uses PUCT and reward backpropagation. \\
Parent-initial-value ablation & MCTS with parent initial values & Uses \texttt{S1}--\texttt{S3} and \texttt{S5}.  The search is still MCTS.  The only algorithmic change is that PUCT uses a \(0.5/0.5\) blend of observed value and model-derived initial value. \\
Beam-search ablation & Beam search plus solver reference & Uses \texttt{S1}--\texttt{S3} and \texttt{S5}.  Search is layer-based: expand the current beam, rank by cumulative log-probability, and keep the top 4 partial solutions.  It does not use MCTS visit counts or reward backpropagation for selection. \\
\hline
\end{tabular}
\end{table}

\paragraph{Strategy presets.}
The strategy names define the ablation axis.  The same names are reused across MCTS and beam search so that the effect of pruning and duplicate removal can be compared under different search modes.

\begin{table}[t]
\centering
\small
\caption{Strategy presets used in the ablations.}
\label{tab:ablation-strategy-presets}
\begin{tabular}{p{0.15\linewidth}p{0.77\linewidth}}
\hline
Strategy & Configuration \\
\hline
\texttt{S1} & Full search variant: expansion-time feasibility pruning, redundancy removal, and terminal feasibility enforcement. \\
\texttt{S2} & No expansion-time feasibility pruning; infeasible partial paths may continue until terminal checking; redundancy removal remains enabled. \\
\texttt{S3} & No expansion-time feasibility pruning and no redundancy removal; infeasible partial paths may continue until terminal checking. \\
\texttt{S4} & Sequential non-MCTS baseline used only in the standard ablation.  It samples one path at a time without PUCT tree selection. \\
\texttt{S5} & Solver-reference search with a depth-first frontier.  This run enumerates solver-generated actions rather than sampling language-model actions, and is used mainly to provide reference solutions for evaluation. \\
\hline
\end{tabular}
\end{table}

\paragraph{Default hyperparameters.}
The default experimental settings are summarized below.  Unless otherwise stated, these values are shared across the ablations.

\begin{table}[H]
\centering
\small
\caption{Default hyperparameters used by the ablation experiments.}
\label{tab:ablation-hyperparameters}
\begin{tabular}{p{0.34\linewidth}p{0.58\linewidth}}
\hline
Hyperparameter & Default value \\
\hline
Base language model & \texttt{meta-llama/Llama-3.2-3B-Instruct} \\
Problem families & \texttt{qap}, \texttt{knapsack}, MaxSat, \texttt{role\_assignment} \\
Difficulty levels & 1--4 \\
Maximum search depth & 6 for all difficulty levels \\
MCTS rollout budget & 16 rollouts per instance and strategy \\
Solver-reference budget & 500 solver-reference rollouts per instance \\
Children per expansion & 20 \\
PUCT coefficient \(c_{\mathrm{puct}}\) & 5 for MCTS variants \\
Beam width & 4 for the beam-search ablation \\
Generation settings & temperature 0.7, top-\(p=0.95\), maximum 768 new tokens \\
Repeats and seeds & one repeat with seed list \(\{0\}\) \\
Terminal handling & terminal feasibility enforced; incorrect or invalid final rewards set to \(-1\) \\
Depth-limited rollouts & non-answer depth-limit reward 0, with depth-based value discount 0.25 \\
Parallelism & GPUs 0--3 with 4 workers per GPU; solver reference uses 4 CPU workers \\
\hline
\end{tabular}
\end{table}

\paragraph{Reporting and smartness metrics.}
All smartness quantities are computed as post-processing from the saved ablation metrics; this step does not rerun search.  The same reporting procedure is applied to the standard MCTS ablation, the parent-initial-value MCTS ablation, and the beam-search ablation.

The report forms reference pools of unique terminal solutions.  \texttt{solverref} contains solutions found by \texttt{S5}, while \texttt{unionref} contains all unique terminal solutions found by any strategy.  The primary report setting uses \texttt{unionref}, relative-gap threshold \(\epsilon=0.05\), and \(\delta=0.10\), so the target success probability is \(1-\delta=0.90\).  A solution is counted as \(\epsilon\)-good if it is feasible and within \(5\%\) relative gap of the best-known normalized objective for that instance.

For each problem and difficulty level, the report estimates an effective difficulty \(D_{\alpha,\epsilon}\) from the fraction of reference-pool solutions that are \(\epsilon\)-good.  It then estimates \(\tau\), the rollout budget or wall-clock time required for a strategy to reach the target success probability.  The main smartness score is
\[
  \mathrm{smartness} = \frac{D_{\alpha,\epsilon}}{\tau},
\]
so larger values mean that a strategy makes more progress per rollout or per second.  The enhanced report also exports pass@\(\{8,16,32,64\}\), an effective-branching proxy, invalid proposal rates, duplicate proposal rates, feasible-terminal rates, and exact-optimal rates.

\subsection{Policy Training Implementation for RFT}
\label{sec:policy-training-implementation}

\paragraph{Training objective.}
Policy training is performed as supervised fine-tuning of the base causal language model.  The goal is to teach the model to produce the next search action from a partial problem state.  Each training example contains a system prompt, a user prompt describing the current partial state, and an assistant target containing the next action and its accompanying reasoning.  The model is trained with the standard autoregressive next-token objective on the chat-formatted sequence.  Padding tokens are masked from the loss.

\paragraph{Constructing supervision from search.}
Training data are generated from successful search trajectories.  For each problem instance, search first builds a tree of partial solutions.  Terminal nodes are filtered to keep valid terminal answers, and the best terminal node is selected according to the task objective.  The path from the root to this terminal node is then decomposed into one supervised example per edge.  If a trajectory contains \(d\) actions, it contributes \(d\) training examples.  For an edge \(s_t \rightarrow s_{t+1}\), the input prompt is built from the partial state \(s_t\), and the target is the exact action text stored at \(s_{t+1}\).  This turns a complete discovered solution into step-level imitation data.

Only terminal answer trajectories are used for the default next-step policy data.  If a search run does not find a valid terminal solution for an instance, that instance does not contribute a supervised trajectory.  When multiple terminal solutions are available, the selected trajectory is the one with the best objective value, with ties broken randomly.  The saved examples also contain metadata such as node depth, visit count, and value estimates, but the policy fine-tuning objective uses only the system prompt, user prompt, and next-step target text.

\paragraph{Training variants.}
We use the same fine-tuning procedure for the MCTS-generated and solver-style datasets.  In the MCTS-generated setting, the demonstrations come from the best terminal paths found by the MCTS search policy.  In the solver-style setting, the demonstrations come from a more direct solver-guided trajectory generator.  A no-training baseline is also evaluated by running the base model directly under the same evaluation procedure.  Thus, the training comparison isolates whether policy fine-tuning on search-generated trajectories improves subsequent search and solution quality.

\paragraph{Tokenization and formatting.}
Examples are formatted with the chat template of the base model tokenizer.  The formatted conversation contains a system message, a user message, and an assistant message.  The assistant message is filled with the selected next-step target during training.  Sequences are padded or truncated to a maximum length of 2048 tokens.  The tokenizer uses the end-of-sequence token as the padding token when no dedicated padding token is available.

\paragraph{Optimization.}
Fine-tuning is implemented with the Hugging Face causal language-model training stack.  The default runs fine-tune the base model directly rather than using parameter-efficient adapters.  Mixed precision is selected automatically: bfloat16 is used when supported by the GPU, otherwise float16 is used on CUDA.  The optimizer is AdamW with a linear learning-rate schedule, warmup ratio 0.05, and weight decay 0.01.  Intermediate checkpoints are not saved; the final model and tokenizer are saved after training.

\begin{table}[t]
\centering
\small
\caption{Default supervised fine-tuning settings for the policy model.}
\label{tab:policy-training-hyperparameters}
\begin{tabular}{p{0.34\linewidth}p{0.58\linewidth}}
\hline
Training setting & Default value \\
\hline
Base model & \texttt{meta-llama/Llama-3.2-3B-Instruct} \\
Training objective & causal language-model supervised fine-tuning \\
Training examples & one next-step target per edge on a selected successful trajectory \\
Trajectory selection & best valid terminal solution found by search \\
Learning rate & \(5\times 10^{-6}\) \\
Epochs & 3 \\
Per-device batch size & 4 \\
Gradient accumulation & 16 steps \\
Maximum sequence length & 2048 tokens \\
Optimizer & AdamW \\
Learning-rate schedule & linear schedule with 0.05 warmup ratio \\
Weight decay & 0.01 \\
Precision & bfloat16 when available; otherwise float16 on CUDA \\
Checkpointing & final checkpoint only \\
Evaluation during training & once per epoch when a validation set is provided \\
\hline
\end{tabular}
\end{table}

\paragraph{Evaluation protocol.}
When validation data are provided, the model is evaluated before fine-tuning, at the end of each epoch, and once more after training.  After fine-tuning, the resulting policy is used as the proposal model in the same search procedure used for data generation and evaluation.  This keeps the comparison focused on the effect of the learned proposal policy rather than changes in the downstream search algorithm.

\paragraph{Optional adapter path.}
The training code also supports an optional low-rank adapter mode.  When enabled, LoRA adapters are attached to the query and value projection matrices with rank \(r=8\), scaling \(\alpha=32\), dropout 0.05, and no bias terms.  This path is useful for memory-constrained runs, but the default training configuration uses direct fine-tuning of the base causal language model.

\clearpage
\subsection{Policy training GRPO}

\label{sec:rl-training-implementation}

\paragraph{Legacy ranking and the revised protocol.}
The legacy solver-prefix implementation maps distinct valid completion scores to equally spaced values in $[0,1]$, with tied scores sharing a value and invalid actions receiving $-1$. These values are used directly as advantages, without group standardization; an all-valid, all-tied group receives $+1$. The revised protocol replaces this mapping with the centered pairwise advantages of Algorithm~\ref{alg:nupo}. Valid results outrank invalid ones, invalid results tie, and all-tied or singleton groups receive zero. The implementation compares exact score ties by default; nonzero absolute or relative tie tolerances must be recorded if enabled. Advantages are applied to assistant response tokens without a second normalization.

The configured online dataset selects one solver-trajectory state per instance, deterministically from the instance ID and data seed. Pairwise reruns retain this state selection, difficulty order, and solver-prefix evaluation. Table~\ref{tab:qwen25-3b-pass1-optimization} now reports the completed pairwise transfer rerun. Spatial/curriculum tables retain the original results; Appendix~\ref{app:math-problems} reports the updated GRPO math-transfer comparison for both 3B models. Changing the rank estimator does not make these experiments autonomous. Table~\ref{tab:rl-training-hyperparameters} records the rank-enabled legacy solver-prefix configuration, separately from the transfer rerun below, the math-transfer recipe in Appendix~\ref{app:math-problems}, and the autonomous settings in Appendix~\ref{app:credit-control}.

\paragraph{Pairwise transfer rerun (Table~\ref{tab:qwen25-3b-pass1-optimization}).}
\label{app:transfer-pairwise}
The new Role Assignment pipeline requests 1,000 roots per level and selects one solver-prefix state per retained root, omitting empty solver trajectories without replacement. Training processes levels $\alpha_1$--$\alpha_4$ in order for one epoch, with 128 prompts per batch, eight responses per prompt, learning rate $10^{-6}$, actor KL coefficient $0.001$, and clipping bounds $0.2/0.28$. It uses centered pairwise advantages, zero tie tolerances, and no group filtering or refill. The final incomplete batch is dropped. The rebuilt MATH-500 and MBPP baselines use ordinary group-standardized GRPO for two and four epochs, respectively, with mathematical-answer verification and code-test rewards. These recipes do not equate training-data or update budgets. The rerun uses two H100 NVL GPUs (approximately 94\,GB each). The pipeline records retained counts and runtime overrides in its run manifests.

Evaluation freezes the same twelve domain-by-level files for all four checkpoints. The generator requests 250 roots per stratum and retains all available solver-prefix states. Each prompt receives eight responses at temperature 1, top-$p=0.7$, top-$k=50$, and prompt/response limits of 1,256/2,048 tokens, using the same sampling seed. The retained scorer checks validity and compares the solver-completed scalar objective with the reference value. The default legacy path permits variable action counts and zero scores on solver exceptions; synthetic tests expose repeated-job Scheduling and zero-target failure matches. Thus, the reported score is not a strict single-action optimality certificate. Their occurrence in the saved evaluations has not been audited.

For prompt $i$, let $q_i$ be the fraction of its eight responses marked valid and matching. The point estimate is $\hat p=N^{-1}\sum_i q_i$. Let $G$ be the number of original puzzles, $m_g$ their prefix counts, and $S_g=\sum_{i\in g}q_i$. The displayed standard error is
\[
  \mathrm{SE}=\frac{1}{N}\sqrt{\frac{G}{G-1}\sum_{g=1}^G(S_g-\hat p\,m_g)^2}.
\]
Means and SEs are multiplied by 100. This preserves equal prompt weighting while clustering prefixes from the same puzzle. These are evaluation SEs conditional on each checkpoint, not 95\% confidence margins or independent-training-run variability.

\begin{table}[H]
\centering
\small
\caption{Rank-enabled legacy solver-prefix training configuration. The Table~\ref{tab:qwen25-3b-pass1-optimization} rerun uses the separate transfer pipeline described above.}
\label{tab:rl-training-hyperparameters}
\begin{tabular}{p{0.36\linewidth}p{0.54\linewidth}}
\hline
Training setting & Value \\
\hline
Base model & \texttt{Qwen/Qwen2.5-3B-Instruct} by default; task wrappers also support \texttt{Qwen/Qwen2.5-7B-Instruct} \\
Training algorithm & GRPO-style clipped policy updates in VERL with the legacy rank mapping \\
Validation file & level-1 test Parquet during training; full level 1--4 evaluation after merging \\
Train batch size & 128 prompts \\
Rollouts per prompt & 8 responses \\
Maximum prompt length & 1256 tokens for Role Assignment, MaxSat, and Polyomino Target Cover \\
Maximum response length & 2048 tokens \\
Epochs & 1 for the main task runs \\
Data order & \texttt{data.shuffle=false} \\
Actor learning rate & \(1\times10^{-6}\) \\
Optimizer backend & VERL actor optimizer with FSDP training \\
PPO mini-batch size & 128 \\
PPO micro-batch size & 4 per GPU \\
KL loss & enabled, coefficient \(0.001\), \texttt{low\_var\_kl} \\
KL controller coefficient & \(0.001\) \\
Gradient checkpointing & enabled \\
Remove padding optimization & enabled \\
Rollout engine & vLLM \\
Rollout tensor parallelism & 2 during training \\
Rollout GPU memory utilization & 0.4 \\
Hardware configuration & 1 node, 4 GPUs per node (4A100) \\
Checkpoint saving & effectively final checkpoint only; latest FSDP actor checkpoint is merged for inference \\
\hline
\end{tabular}
\end{table}

\clearpage

\section{Online RL results}
\label{sec:online-rl-results}

\paragraph{Experiment scope.}
This appendix reports the spatial/curriculum results and the updated mathematical-transfer comparison. Spatial models were trained on Polyomino, whereas the math models were trained jointly on Role Assignment and MaxSAT. These use separate task-specific checkpoints from the Role Assignment-only comparison in Table~\ref{tab:qwen25-3b-pass1-optimization}. Appendix~\ref{sec:rl-training-implementation} and the math subsection below describe their respective protocols. Spatial tables report proportions; the math comparison reports percentages.

\subsection{Spatial reasoning}
\label{sec:spatial-reasoning-results}

\paragraph{Evaluation uncertainty in Figure~\ref{fig:polyomino}.}
Each prompt has eight saved responses; pass@8 indicates at least one successful response. Polyomino averages over 250, 250, 449, and 750 oracle-prefix prompts at levels $\alpha_1$--$\alpha_4$, respectively, from $G=250$ puzzles per level. With $m_g$ prompts and $S_g$ successful prompts in puzzle $g$, $N=\sum_g m_g$ and $\hat p=\sum_g S_g/N$, the puzzle-clustered standard error is $\mathrm{SE}=\sqrt{\frac{G}{G-1}\sum_g(S_g-\hat p m_g)^2}/N$. This preserves prompt weighting despite unequal prefix counts. Spatial QA averages eight fixed capability groups of 25 questions each. With within-group sample variance $s_h^2$, its stratified error is $\mathrm{SE}=\sqrt{\sum_{h=1}^8 s_h^2/25}/8$. Means and SEs are multiplied by 100. These evaluation SEs condition on each saved checkpoint; they are neither training-run variability nor 95\% confidence intervals.

\begin{table}[H]
\centering
\caption{Polyomino (Base vs Trained) --- Progressive training L1--L4, reporting pass@k for $k\in\{1,4,8\}$.}
\begin{tabular}{l*{8}{c}}
\toprule
 & \multicolumn{2}{c}{\textbf{Level 1}}  &  \multicolumn{2}{c}{\textbf{Level 2}}  &  \multicolumn{2}{c}{\textbf{Level 3}}  &  \multicolumn{2}{c}{\textbf{Level 4}} \\
\cmidrule(lr){2-3}
\cmidrule(lr){4-5}
\cmidrule(lr){6-7}
\cmidrule(lr){8-9}
\textbf{pass@k} & \textbf{Base} & \textbf{Trained} & \textbf{Base} & \textbf{Trained} & \textbf{Base} & \textbf{Trained} & \textbf{Base} & \textbf{Trained} \\
\midrule
\textbf{Pass@1} & 0.006 & 0.015 & 0.003 & 0.038 & 0.008 & 0.051 & 0.002 & 0.009 \\
\textbf{Pass@4} & 0.024 & 0.056 & 0.012 & 0.134 & 0.028 & 0.154 & 0.006 & 0.031 \\
\textbf{Pass@8} & 0.044 & 0.100 & 0.024 & 0.228 & 0.053 & 0.225 & 0.012 & 0.052 \\
\bottomrule
\end{tabular}
\end{table}

\begin{table}[H]
\centering
\caption{Spatial 2D capabilities --- Pass@8 at Level 1 and Level 2.}
\resizebox{0.95\linewidth}{!}{%
\begin{tabular}{l*{8}{c}}
\toprule
 & \multicolumn{4}{c}{\textbf{Level 1}}  &  \multicolumn{4}{c}{\textbf{Level 2}} \\
\cmidrule(lr){2-5}
\cmidrule(lr){6-9}
\textbf{Capability}  & \textbf{Base} & \textbf{L1--L4} & \textbf{L1 Only} & \textbf{L4 Only} & \textbf{Base} & \textbf{L1--L4} & \textbf{L1 Only} & \textbf{L4 Only} \\
\midrule
\textbf{Translate} & 0.520 & 0.680 & 0.320 & 0.360 & 0.240 & 0.320 & 0.280 & 0.240 \\
\textbf{Rotate \& Place} & 0.200 & 0.280 & 0.080 & 0.120 & 0.160 & 0.280 & 0.240 & 0.240 \\
\textbf{Reflect \& Place} & 0.360 & 0.440 & 0.320 & 0.360 & 0.160 & 0.280 & 0.200 & 0.200 \\
\textbf{Compose (A,B,C)} & 0.080 & 0.200 & 0.040 & 0.040 & 0.040 & 0.120 & 0.040 & 0.160 \\
\textbf{Equivalence} & 0.520 & 0.960 & 0.760 & 0.520 & 0.440 & 0.840 & 0.880 & 0.440 \\
\textbf{Symmetry} & 0.520 & 0.760 & 0.400 & 0.600 & 0.600 & 0.600 & 0.520 & 0.560 \\
\textbf{Overlap IoU} & 0.200 & 0.120 & 0.120 & 0.120 & 0.120 & 0.040 & 0.040 & 0.080 \\
\textbf{Coverage} & 0.160 & 0.240 & 0.200 & 0.240 & 0.200 & 0.240 & 0.280 & 0.280 \\
\bottomrule
\end{tabular}%
}
\end{table}

\begin{table}[H]
\centering
\caption{Spatial 2D capabilities --- Pass@1 at Level 1 and Level 2.}
\resizebox{0.95\linewidth}{!}{%
\begin{tabular}{l*{8}{c}}
\toprule
 & \multicolumn{4}{c}{\textbf{Level 1}}  &  \multicolumn{4}{c}{\textbf{Level 2}} \\
\cmidrule(lr){2-5}
\cmidrule(lr){6-9}
\textbf{Capability}  & \textbf{Base} & \textbf{L1--L4} & \textbf{L1 Only} & \textbf{L4 Only} & \textbf{Base} & \textbf{L1--L4} & \textbf{L1 Only} & \textbf{L4 Only} \\
\midrule
\textbf{Translate} & 0.140 & 0.205 & 0.070 & 0.090 & 0.090 & 0.125 & 0.085 & 0.095 \\
\textbf{Rotate \& Place} & 0.050 & 0.085 & 0.010 & 0.035 & 0.075 & 0.070 & 0.080 & 0.085 \\
\textbf{Reflect \& Place} & 0.190 & 0.190 & 0.145 & 0.165 & 0.055 & 0.075 & 0.045 & 0.045 \\
\textbf{Compose (A,B,C)} & 0.020 & 0.035 & 0.005 & 0.005 & 0.015 & 0.035 & 0.010 & 0.030 \\
\textbf{Equivalence} & 0.295 & 0.605 & 0.630 & 0.325 & 0.325 & 0.505 & 0.680 & 0.315 \\
\textbf{Symmetry} & 0.100 & 0.180 & 0.125 & 0.100 & 0.190 & 0.170 & 0.175 & 0.130 \\
\textbf{Overlap IoU} & 0.025 & 0.020 & 0.015 & 0.015 & 0.015 & 0.005 & 0.010 & 0.015 \\
\textbf{Coverage} & 0.040 & 0.055 & 0.035 & 0.060 & 0.065 & 0.065 & 0.065 & 0.065 \\
\bottomrule
\end{tabular}%
}
\end{table}
\clearpage

\subsection{Math problems}
\label{app:math-problems}

\paragraph{Training and evaluation.}
We compare the Base and GRPO checkpoints of Qwen2.5-3B-Instruct and Llama-3.2-3B-Instruct. For each model, GRPO training uses the joint Role Assignment and MaxSAT curriculum with KL coefficient $0.01$, one training seed, one epoch, 59 updates, 128 prompts per batch, eight responses per prompt, and learning rate $10^{-6}$. Math benchmark questions are not included in this training. All four checkpoints are evaluated on the same AMC\textquotesingle{}23, AIME\textquotesingle{}24, and AIME\textquotesingle{}25 questions (40, 30, and 30) on A100 GPUs, with eight responses per question and decoding seed, temperature 1, top-$p=0.7$, top-$k=50$, and prompt/response limits of 1,024/2,048 tokens.

Table~\ref{tab:math-transfer} reports pass@1 and pass@8 averaged over three decoding seeds. These seeds were carried forward from an earlier exploratory analysis and held fixed for this evaluation. Within each run, pass@1 averages the correctness of the eight responses to each question, while pass@8 records whether at least one is correct. Standard errors are the sample standard deviation of the three run-level estimates divided by $\sqrt{3}$. They describe decoding variation for fixed checkpoints and questions, rather than independent-training-run variability or uncertainty over a population of questions. GRPO has higher mean scores in all reported comparisons; the small benchmark sizes and single training run per model limit broader conclusions.

\begin{table}[H]
\centering
\small
\setlength{\tabcolsep}{5pt}
\renewcommand{\arraystretch}{1.15}
\caption{Math transfer for Qwen2.5-3B and Llama-3.2-3B. Entries are mean pass@1 and pass@8 percentages $\pm$ decoding standard errors across three decoding seeds.}
\label{tab:math-transfer}
\begin{tabular}{@{}llrrrr@{}}
\toprule
 & & \multicolumn{2}{c}{Pass@1} & \multicolumn{2}{c}{Pass@8} \\
\cmidrule(lr){3-4}\cmidrule(lr){5-6}
Model & Benchmark & Base & GRPO & Base & GRPO \\
\midrule
\multirow{3}{*}{Qwen2.5-3B}
 & AMC'23 & \tabstat{36.77}{1.33} & \tabstat{38.75}{0.95} & \tabstat{68.33}{3.00} & \tabstat{72.50}{3.82} \\
 & AIME'24 & \tabstat{5.42}{1.10} & \tabstat{5.83}{0.96} & \tabstat{16.67}{3.85} & \tabstat{21.11}{2.22} \\
 & AIME'25 & \tabstat{2.08}{0.24} & \tabstat{2.78}{0.37} & \tabstat{13.33}{1.92} & \tabstat{15.56}{1.11} \\
\midrule
\multirow{3}{*}{Llama-3.2-3B}
 & AMC'23 & \tabstat{18.65}{1.37} & \tabstat{20.52}{0.55} & \tabstat{51.67}{5.83} & \tabstat{52.50}{1.44} \\
 & AIME'24 & \tabstat{3.06}{0.14} & \tabstat{4.03}{0.28} & \tabstat{14.44}{1.11} & \tabstat{15.56}{2.94} \\
 & AIME'25 & \tabstat{0.28}{0.14} & \tabstat{0.42}{0.24} & \tabstat{2.22}{1.11} & \tabstat{3.33}{1.92} \\
\bottomrule
\end{tabular}
\end{table}

\clearpage

\section{Optimization problems}
\label{app:optimization-results}

\paragraph{Protocol and metrics.}
We compare Base and Role Assignment-trained checkpoints of Qwen2.5-3B-Instruct and Llama-3.2-3B-Instruct. Both use the Role Assignment-only pairwise recipe in Appendix~\ref{app:transfer-pairwise}, with one training run per model; MaxSAT and Scheduling are out of distribution. Base and Trained within each model family use the same frozen prompts, scorer, decoding settings, and sampling seed, with eight responses per prompt. The Qwen pass@1 values reproduce the corresponding rows of Table~\ref{tab:qwen25-3b-pass1-optimization}.

All entries are percentages $\pm$ evaluation-root-cluster standard errors. Pass@1 averages the fraction of responses marked valid and matching the reference value; pass@8 records whether at least one of the eight responses meets both criteria. Prompts receive equal weight, and solver prefixes from the same original puzzle share a cluster. We apply the estimator in Appendix~\ref{app:transfer-pairwise} to each metric's prompt-level values. These errors describe evaluation uncertainty for fixed checkpoints, not independent-training-run variability.

\paragraph{Results.}
Qwen shows more consistent gains across the two sampling budgets: pass@1 improves in all twelve conditions, while pass@8 improves on Role Assignment and MaxSAT but decreases on Scheduling. Llama improves pass@1 at every Role Assignment and Scheduling level, with mixed MaxSAT results, while pass@8 decreases in all twelve conditions. Thus, higher single-response accuracy does not necessarily imply broader coverage with eight responses.

\subsection{Role Assignment (ID)}
\label{app:optimization-role}
\begin{table}[H]
\centering
\small
\setlength{\tabcolsep}{5pt}
\renewcommand{\arraystretch}{1.15}
\caption{Role Assignment (ID): Base versus models trained on Role Assignment with pairwise advantages. Entries are pass@1 and pass@8 percentages $\pm$ evaluation-root-cluster standard errors.}
\label{tab:optimization-role-matched}
\begin{tabular}{@{}llrrrr@{}}
\toprule
 & & \multicolumn{2}{c}{Pass@1} & \multicolumn{2}{c}{Pass@8} \\
\cmidrule(lr){3-4}\cmidrule(lr){5-6}
Model & Level & Base & Trained & Base & Trained \\
\midrule
\multirow{4}{*}{Qwen2.5-3B}
 & $\alpha_1$ & \tabstat{33.82}{0.78} & \tabstat{43.53}{0.91} & \tabstat{91.47}{1.03} & \tabstat{94.67}{0.84} \\
 & $\alpha_2$ & \tabstat{25.57}{0.63} & \tabstat{37.31}{0.65} & \tabstat{86.00}{1.10} & \tabstat{93.30}{0.86} \\
 & $\alpha_3$ & \tabstat{21.15}{0.57} & \tabstat{33.47}{0.58} & \tabstat{79.68}{1.22} & \tabstat{89.76}{0.90} \\
 & $\alpha_4$ & \tabstat{16.87}{0.49} & \tabstat{28.61}{0.56} & \tabstat{70.00}{1.38} & \tabstat{86.07}{0.87} \\
\midrule
\multirow{4}{*}{Llama-3.2-3B}
 & $\alpha_1$ & \tabstat{30.27}{0.84} & \tabstat{46.67}{2.08} & \tabstat{90.53}{1.16} & \tabstat{55.60}{2.25} \\
 & $\alpha_2$ & \tabstat{24.52}{0.68} & \tabstat{43.61}{1.97} & \tabstat{83.60}{1.35} & \tabstat{52.20}{2.10} \\
 & $\alpha_3$ & \tabstat{21.19}{0.56} & \tabstat{39.10}{1.37} & \tabstat{78.80}{1.20} & \tabstat{53.12}{1.66} \\
 & $\alpha_4$ & \tabstat{18.76}{0.46} & \tabstat{36.06}{1.41} & \tabstat{74.60}{1.15} & \tabstat{49.53}{1.70} \\
\bottomrule
\end{tabular}
\end{table}

\subsection{MaxSAT (OOD)}
\label{app:optimization-maxsat}
\begin{table}[H]
\centering
\small
\setlength{\tabcolsep}{5pt}
\renewcommand{\arraystretch}{1.15}
\caption{MaxSAT (OOD): Base versus models trained on Role Assignment with pairwise advantages. Entries are pass@1 and pass@8 percentages $\pm$ evaluation-root-cluster standard errors.}
\label{tab:optimization-maxsat-matched}
\begin{tabular}{@{}llrrrr@{}}
\toprule
 & & \multicolumn{2}{c}{Pass@1} & \multicolumn{2}{c}{Pass@8} \\
\cmidrule(lr){3-4}\cmidrule(lr){5-6}
Model & Level & Base & Trained & Base & Trained \\
\midrule
\multirow{4}{*}{Qwen2.5-3B}
 & $\alpha_1$ & \tabstat{42.11}{1.65} & \tabstat{46.82}{1.73} & \tabstat{91.40}{1.63} & \tabstat{92.11}{1.68} \\
 & $\alpha_2$ & \tabstat{29.44}{1.45} & \tabstat{35.09}{1.68} & \tabstat{81.72}{2.29} & \tabstat{84.14}{2.36} \\
 & $\alpha_3$ & \tabstat{25.49}{1.22} & \tabstat{30.87}{1.40} & \tabstat{74.82}{2.22} & \tabstat{78.00}{2.13} \\
 & $\alpha_4$ & \tabstat{22.36}{0.92} & \tabstat{25.09}{1.05} & \tabstat{72.24}{1.89} & \tabstat{73.16}{1.91} \\
\midrule
\multirow{4}{*}{Llama-3.2-3B}
 & $\alpha_1$ & \tabstat{45.25}{1.67} & \tabstat{46.59}{2.43} & \tabstat{90.68}{1.65} & \tabstat{77.42}{2.55} \\
 & $\alpha_2$ & \tabstat{30.30}{1.34} & \tabstat{27.37}{1.99} & \tabstat{81.03}{2.18} & \tabstat{58.28}{2.85} \\
 & $\alpha_3$ & \tabstat{25.89}{1.10} & \tabstat{23.04}{1.62} & \tabstat{76.04}{2.09} & \tabstat{55.01}{2.56} \\
 & $\alpha_4$ & \tabstat{16.25}{0.77} & \tabstat{12.78}{0.94} & \tabstat{61.21}{2.15} & \tabstat{42.46}{2.18} \\
\bottomrule
\end{tabular}
\end{table}

\subsection{Scheduling (OOD)}
\label{app:optimization-scheduling}
\begin{table}[H]
\centering
\small
\setlength{\tabcolsep}{5pt}
\renewcommand{\arraystretch}{1.15}
\caption{Scheduling (OOD): Base versus models trained on Role Assignment with pairwise advantages. Entries are pass@1 and pass@8 percentages $\pm$ evaluation-root-cluster standard errors.}
\label{tab:optimization-scheduling-matched}
\begin{tabular}{@{}llrrrr@{}}
\toprule
 & & \multicolumn{2}{c}{Pass@1} & \multicolumn{2}{c}{Pass@8} \\
\cmidrule(lr){3-4}\cmidrule(lr){5-6}
Model & Level & Base & Trained & Base & Trained \\
\midrule
\multirow{4}{*}{Qwen2.5-3B}
 & $\alpha_1$ & \tabstat{57.67}{0.87} & \tabstat{60.08}{0.96} & \tabstat{96.00}{0.54} & \tabstat{93.84}{0.69} \\
 & $\alpha_2$ & \tabstat{44.34}{0.69} & \tabstat{46.17}{0.76} & \tabstat{89.40}{0.73} & \tabstat{86.53}{0.85} \\
 & $\alpha_3$ & \tabstat{36.30}{0.50} & \tabstat{38.96}{0.60} & \tabstat{84.80}{0.76} & \tabstat{80.23}{0.81} \\
 & $\alpha_4$ & \tabstat{37.30}{0.49} & \tabstat{39.14}{0.59} & \tabstat{86.51}{0.83} & \tabstat{81.66}{0.86} \\
\midrule
\multirow{4}{*}{Llama-3.2-3B}
 & $\alpha_1$ & \tabstat{45.44}{0.84} & \tabstat{55.97}{1.27} & \tabstat{91.60}{0.81} & \tabstat{69.20}{1.15} \\
 & $\alpha_2$ & \tabstat{34.30}{0.61} & \tabstat{44.78}{1.05} & \tabstat{85.67}{0.91} & \tabstat{58.73}{1.21} \\
 & $\alpha_3$ & \tabstat{28.89}{0.45} & \tabstat{38.71}{0.83} & \tabstat{79.71}{0.83} & \tabstat{57.71}{1.01} \\
 & $\alpha_4$ & \tabstat{29.09}{0.48} & \tabstat{37.91}{0.81} & \tabstat{79.77}{0.88} & \tabstat{55.60}{1.00} \\
\bottomrule
\end{tabular}
\end{table}

\subsection{Additional results MCTS}
\label{sec:appendix:results-mcts}

\begin{figure}[H]
  \centering
  \includegraphics[width=1.\textwidth]{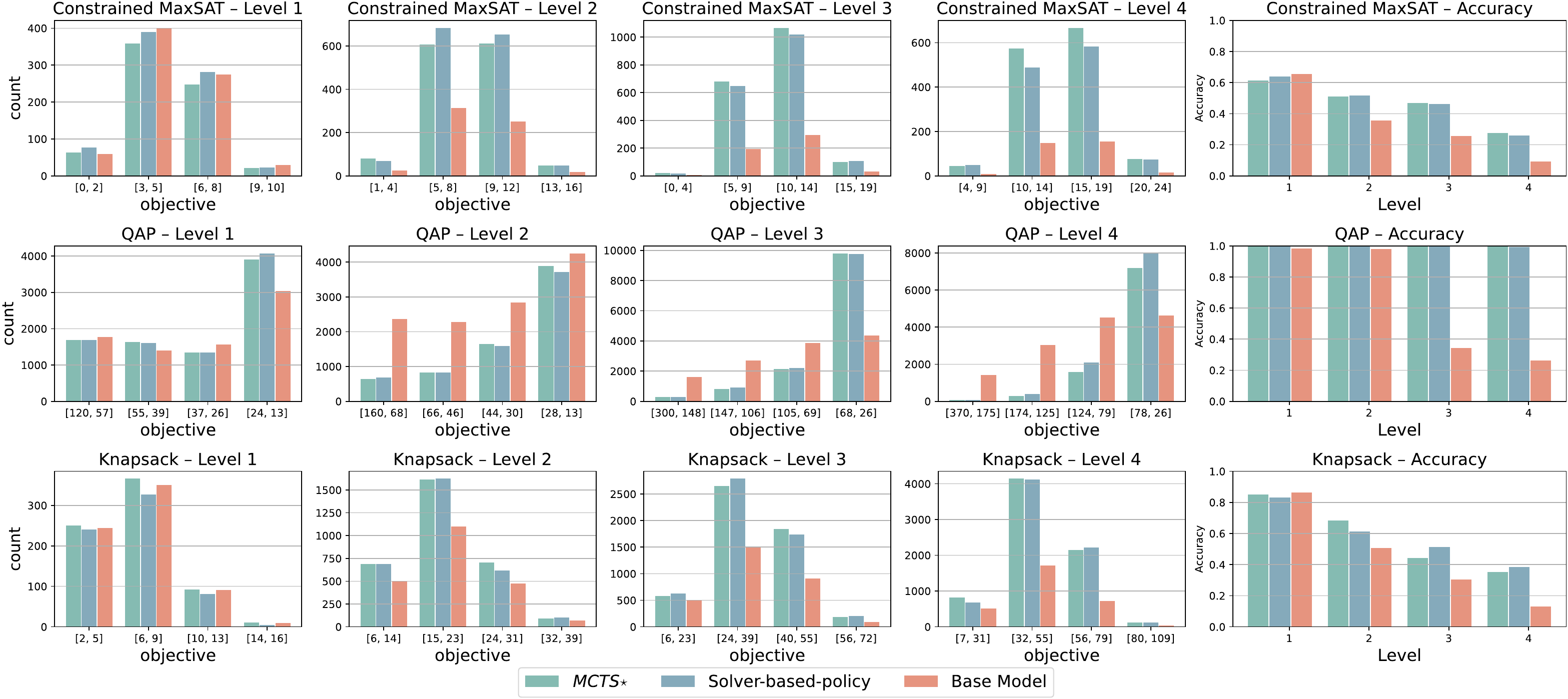}
  \vspace{-1em}
  \caption{Accuracy and solution-quality histograms for baseline, pure MCTS, and solver-guided MCTS on QAP, 0/1 Knapsack, and Constrained MaxSAT tasks as complexity (Level 1 $\rightarrow$ 4) increases.}
  \label{fig:optimization2}
  \vspace{-0.5em}
\end{figure}

\definecolor{mutedteal}{RGB}{103,171,159}
\definecolor{dustyblue}{RGB}{103,149,171}
\definecolor{softpeach}{RGB}{244,199,186}
\definecolor{mutedviolet}{RGB}{149,103,171}
\definecolor{appendixgray}{RGB}{122,122,122}

\subsection{Additional search-only diagnostics for Exp.~1 on \texorpdfstring{OPT$^\star$}{OPT-STAR}}
\label{app:exp1-search-diagnostics}

\begin{table}[t]
\centering
\caption{Complete search-only results for Exp.~1. Entries are estimates $\pm$ paired-bootstrap standard errors from 20,000 resamples of the 16 task-by-difficulty cells. These errors measure sensitivity to benchmark-cell composition, not within-cell or seed variability. Rates are percentages; $b_{\mathrm{eff}}$ and $k_{90}$ are dimensionless. Pass@64 is calibrated restart success. The sequential baseline is shared across both blocks.}
\label{tab:offline-search-complete}
\begingroup
\footnotesize
\setlength{\tabcolsep}{2pt}
\renewcommand{\arraystretch}{1.12}
\renewcommand{\tabstat}[2]{\ensuremath{#1}\,{\scriptsize\color{black!55}$\pm$\,#2}}
\begin{tabular*}{\linewidth}{@{\extracolsep{\fill}}lrrrrrr@{}}
\toprule
\textbf{Configuration} & $p_g\uparrow$ & $b_{\mathrm{eff}}\downarrow$ & $k_{90}\downarrow$ & pass@64$\uparrow$ & Feasible$\uparrow$ & Exact$\uparrow$ \\
\midrule
\multicolumn{7}{@{}l}{\textit{MCTS$\star$}} \\
Check + de-duplication & \tabstat{12.2}{3.3} & \tabstat{15.04}{2.86} & \tabstat{56.5}{5.9} & \tabstat{88.9}{1.4} & \tabstat{72.8}{6.6} & \tabstat{10.9}{3.4} \\
No check/prune & \tabstat{4.0}{1.5} & \tabstat{25.70}{3.52} & \tabstat{73.8}{5.4} & \tabstat{84.8}{1.4} & \tabstat{16.4}{3.0} & \tabstat{3.7}{1.5} \\
No check or de-duplication & \tabstat{5.1}{2.0} & \tabstat{25.26}{3.49} & \tabstat{74.3}{5.2} & \tabstat{84.6}{1.3} & \tabstat{19.8}{3.6} & \tabstat{4.8}{2.0} \\
Sequential (outcome-based) & \tabstat{4.9}{2.0} & \tabstat{25.30}{3.56} & \tabstat{74.2}{5.3} & \tabstat{84.6}{1.4} & \tabstat{19.8}{3.8} & \tabstat{4.6}{2.0} \\
Solver reference & \tabstat{0.2}{0.1} & \tabstat{520.1}{135.7} & \tabstat{1776}{179} & \tabstat{18.7}{5.6} & \tabstat{2.3}{0.9} & \tabstat{0.5}{0.3} \\
\midrule
\multicolumn{7}{@{}l}{\textit{BeamSearch$\star$}} \\
Check + de-duplication & \tabstat{7.2}{1.5} & \tabstat{13.04}{2.05} & \tabstat{43.0}{5.5} & \tabstat{93.8}{1.4} & \tabstat{53.3}{3.8} & \tabstat{7.8}{2.1} \\
No check/prune & \tabstat{1.4}{0.4} & \tabstat{53.07}{9.68} & \tabstat{180.7}{29.3} & \tabstat{65.4}{4.3} & \tabstat{10.2}{1.7} & \tabstat{2.0}{0.7} \\
No check or de-duplication & \tabstat{2.5}{0.7} & \tabstat{82.88}{16.12} & \tabstat{329.9}{44.9} & \tabstat{50.9}{4.2} & \tabstat{12.8}{2.6} & \tabstat{2.7}{1.0} \\
Sequential (outcome-based) & \tabstat{4.9}{2.0} & \tabstat{25.30}{3.56} & \tabstat{74.2}{5.3} & \tabstat{84.6}{1.4} & \tabstat{19.8}{3.8} & \tabstat{4.6}{2.0} \\
Solver reference & \tabstat{0.2}{0.1} & \tabstat{499.4}{135.6} & \tabstat{1743}{185} & \tabstat{19.4}{5.7} & \tabstat{2.5}{1.1} & \tabstat{0.6}{0.4} \\
\bottomrule
\end{tabular*}
\endgroup
\end{table}

\paragraph{Setup.}
This appendix provides additional search-only diagnostics for Exp.~1.  All results use the same no-training setting as in the main text, with the union reference pool, relative tolerance $\epsilon_{\mathrm{rel}}=0.05$, and target failure probability $\delta=0.10$ for the sample-complexity estimates.  We compare the full search configuration (S1: feasibility checking/pruning with duplicate merging), the no-check/pruning ablation (S2), the cumulative ablation without checking/pruning or de-duplication (S3), the sequential baseline when available (S4), and the solver-reference trajectory generator (S5).  We use the same component labels for MCTS$\star$, the parent-initialized MCTS$\star$ variant, and BeamSearch$\star$ as in the main text.

\paragraph{Metric definitions.}
For a task instance $t$, let $V_t^\star$ denote the best value found in the reference pool.  A terminal trajectory is counted as good if its task-normalized relative gap from $V_t^\star$ is at most $\epsilon_{\mathrm{rel}}$.  This definition applies to both maximization tasks and minimization tasks such as QAP.  We denote the rollout-level probability mass of good terminal trajectories by $p_g$.

\paragraph{Aggregation and uncertainty for Table~\ref{tab:offline-search-discovery}.}
Let $\mathcal R$ be the set of evaluated roots and $R=|\mathcal R|$. For each root $r$, let $G_r$ count $\epsilon_{\mathrm{rel}}=0.05$-good terminals among its $M_r$ rollouts. The reported aggregates are
\begin{equation}
\label{eq:search-aggregate-estimators}
 p_g=\frac{\sum_{r\in\mathcal R} G_r}{\sum_{r\in\mathcal R} M_r},\qquad
 b_{\mathrm{eff}}=\exp\!\left[-\frac1R\sum_{r\in\mathcal R}
 \log\!\left(\frac{G_r+1/2}{M_r+1}\right)\right].
\end{equation}
The pooled $p_g$ is an unsmoothed, rollout-weighted frequency, displayed as a percentage. By contrast, $b_{\mathrm{eff}}$ gives each root equal weight in log space and uses the Jeffreys-smoothed rate $(G_r+1/2)/(M_r+1)$, which remains positive even when no good terminal is observed at a root. It is therefore the reciprocal geometric mean of these smoothed rootwise rates. Feasibility and exact optimality average rootwise rates, whereas $p_g$ pools rollouts; these weightings differ when roots have different rollout counts, as in beam search. Weighting, smoothing, and nonlinear aggregation differ, so $b_{\mathrm{eff}}$ need not equal $1/p_g$; even with equal rollout counts, the arithmetic and geometric aggregates differ.

For full MCTS$\star$ with checking and de-duplication, 4,000 roots with 20 rollouts each give $p_g=9{,}743/80{,}000\approx0.1218$, displayed as $12.2\%$. The rootwise log aggregate gives $b_{\mathrm{eff}}=\exp(10{,}843.775004/4{,}000)\approx15.04$, whereas the reciprocal pooled rate is approximately $8.21$. These are different summaries of the same terminal observations. This arithmetic explains their numerical difference; it does not independently test Theorem~\ref{thm:complexity_tradeoff}.

For all configurations in Tables~\ref{tab:offline-search-discovery} and~\ref{tab:offline-search-complete}, we use 20,000 paired bootstrap resamples of the 16 retained task-by-difficulty cells (four tasks and four levels; bootstrap seed 20260504), resampling the same cells jointly for every method and recomputing each aggregate. The standard errors are the sample standard deviations across resamples. They measure sensitivity to benchmark-cell composition, not within-cell or across-seed variability, and do not constitute paired-difference significance tests. The original searches used one repeat with seed 0; the sequential baseline is shared across both search families. In the calibration formulas below, $\widetilde p_g$ denotes a smoothed good-terminal probability at a fixed root; applying nonlinear transforms to pooled averages is generally different from aggregating rootwise diagnostics.

We also report the sample-complexity estimate
\begin{equation}
  k_{90} = \left\lceil \frac{\log(0.10)}{\log(1-\widetilde p_g)} \right\rceil,
\end{equation}
which is the number of independent samples needed to obtain at least one good terminal with probability at least $0.90$ under the calibrated restart model.  For independent restarts at a fixed root, the calibrated success probability is
\begin{equation}
  \mathrm{pass@}k = 1-(1-\widetilde p_g)^k.
\end{equation}
Tables~\ref{tab:offline-search-discovery} and~\ref{tab:offline-search-complete} report this calibrated probability at $k=64$, averaged over roots; $k_{90}$ is likewise averaged over roots. These are restart-model estimates, rather than observed success rates from 64 independent searches. In addition, we report terminal feasibility, exact optimality, and the crossover ratio $\rho=b_{\mathrm{eff}}/b_{\mathrm{uniform}}$.  Values of $\rho<1$ indicate that the search policy concentrates more probability mass on good terminals than uniform sampling from the reference pool.

\paragraph{QAP.}
The Quadratic Assignment Problem (QAP) instances in OPT$^\star$ are minimization tasks.  Each instance assigns facilities bijectively to grid locations, with objective equal to the sum of flow-weighted Manhattan distances over facility pairs.  QAP is a useful stress test because the number of possible terminal assignments grows factorially with the number of unassigned facilities.  As a result, even small changes in good-terminal mass can lead to large changes in $b_{\mathrm{eff}}$ and pass@$k$.

\begin{figure}[t]
  \centering
  \includegraphics[width=\linewidth]{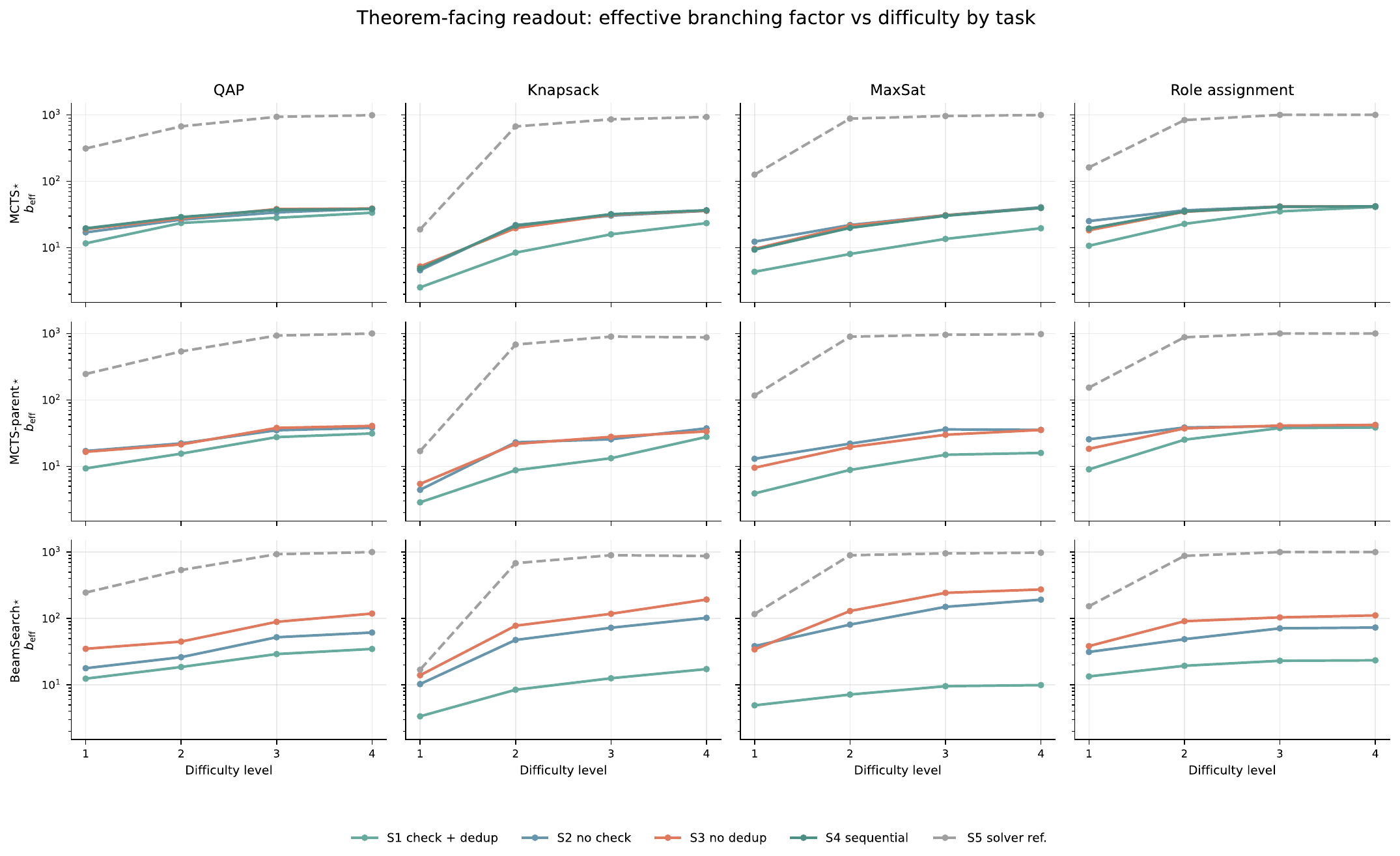}
  \caption{Effective branching factor versus OPT$^\star$ difficulty level, separated by task and search family.  Each column is a task, and each row is a search family.  The $y$-axis is logarithmic, with lower values indicating better search concentration.  Across tasks, the full S1 configuration achieves the lowest or near-lowest $b_{\mathrm{eff}}$ among model-based generators, while removing checking/pruning or duplicate merging increases the difficulty-dependent branching burden.}
  \label{fig:app-exp1-branching-by-task-difficulty}
\end{figure}

\begin{figure}[t]
  \centering
  \includegraphics[width=\linewidth]{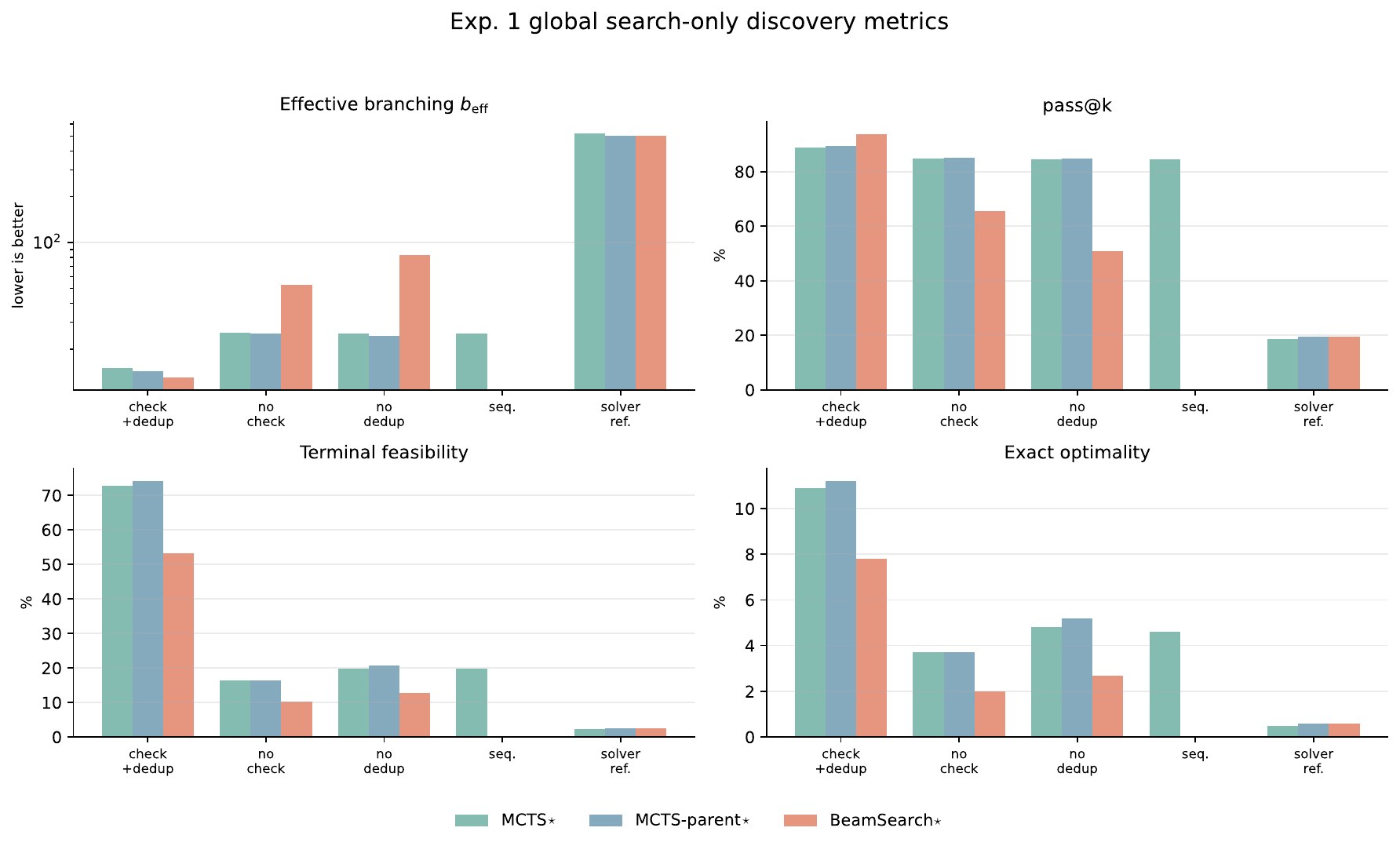}
  \caption{Global search-only metrics across the OPT$^\star$ tasks, with $k=64$ for pass@$k$. Among model-based generators, S1 consistently yields the smallest effective branching factor and the strongest pass@$k$. The ablations disable pruning (S2), or both pruning and duplicate merging (S3), and reduce terminal quality relative to S1.}
  \label{fig:app-exp1-global-summary}
\end{figure}

\begin{figure}[t]
  \centering
  \includegraphics[width=\linewidth]{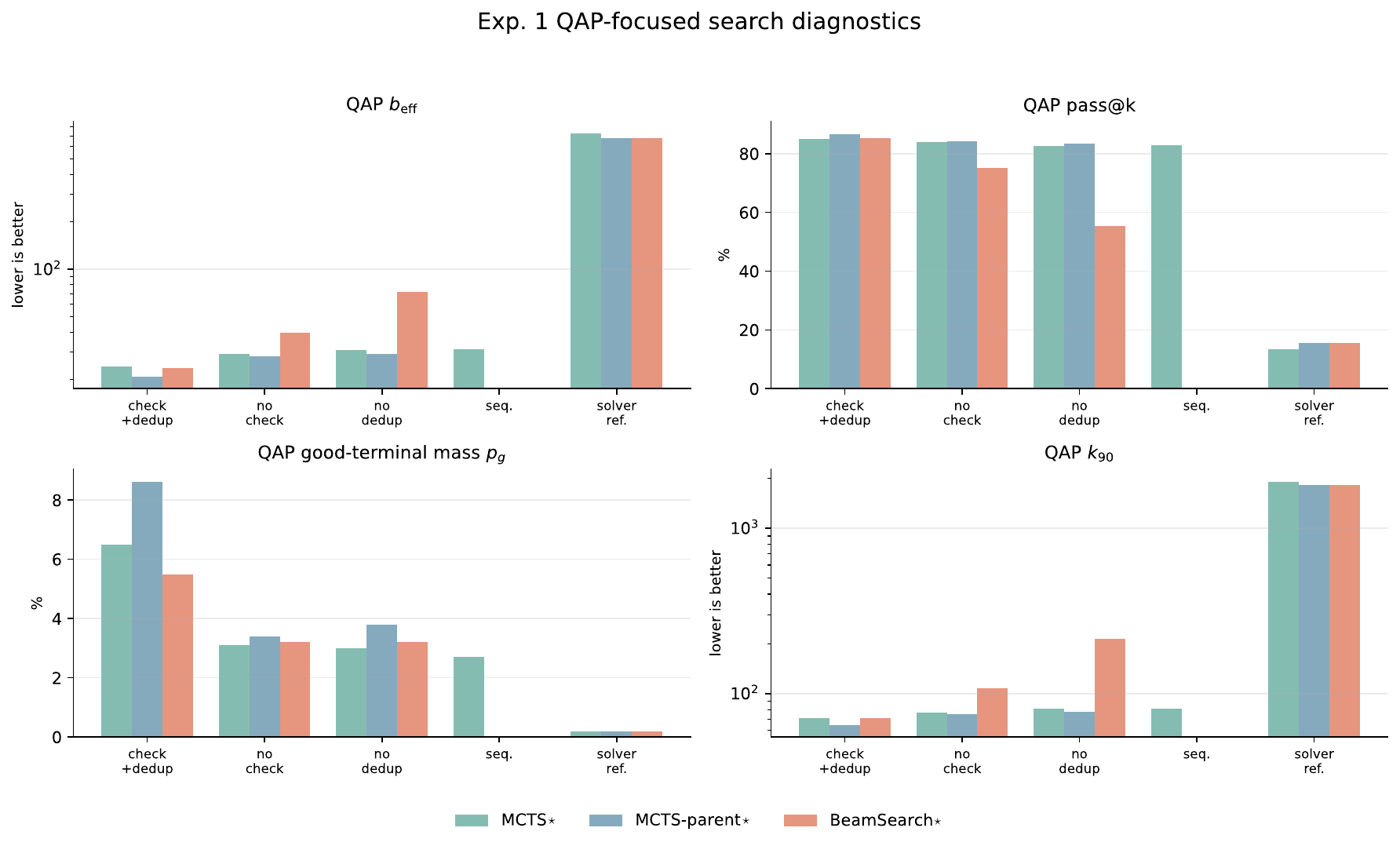}
  \caption{QAP-focused search metrics, with $k=64$ for pass@$k$. QAP metrics average the four difficulty-level summaries and follow the same component pattern as the global aggregate, but the separation from the solver-reference generator is sharper under the trajectory-mass readout. Among the evaluated variants, the parent-initialized MCTS$\star$ variant achieves the lowest QAP $b_{\mathrm{eff}}$.}
  \label{fig:app-exp1-qap-summary}
\end{figure}

\begin{figure}[t]
  \centering
  \includegraphics[width=\linewidth]{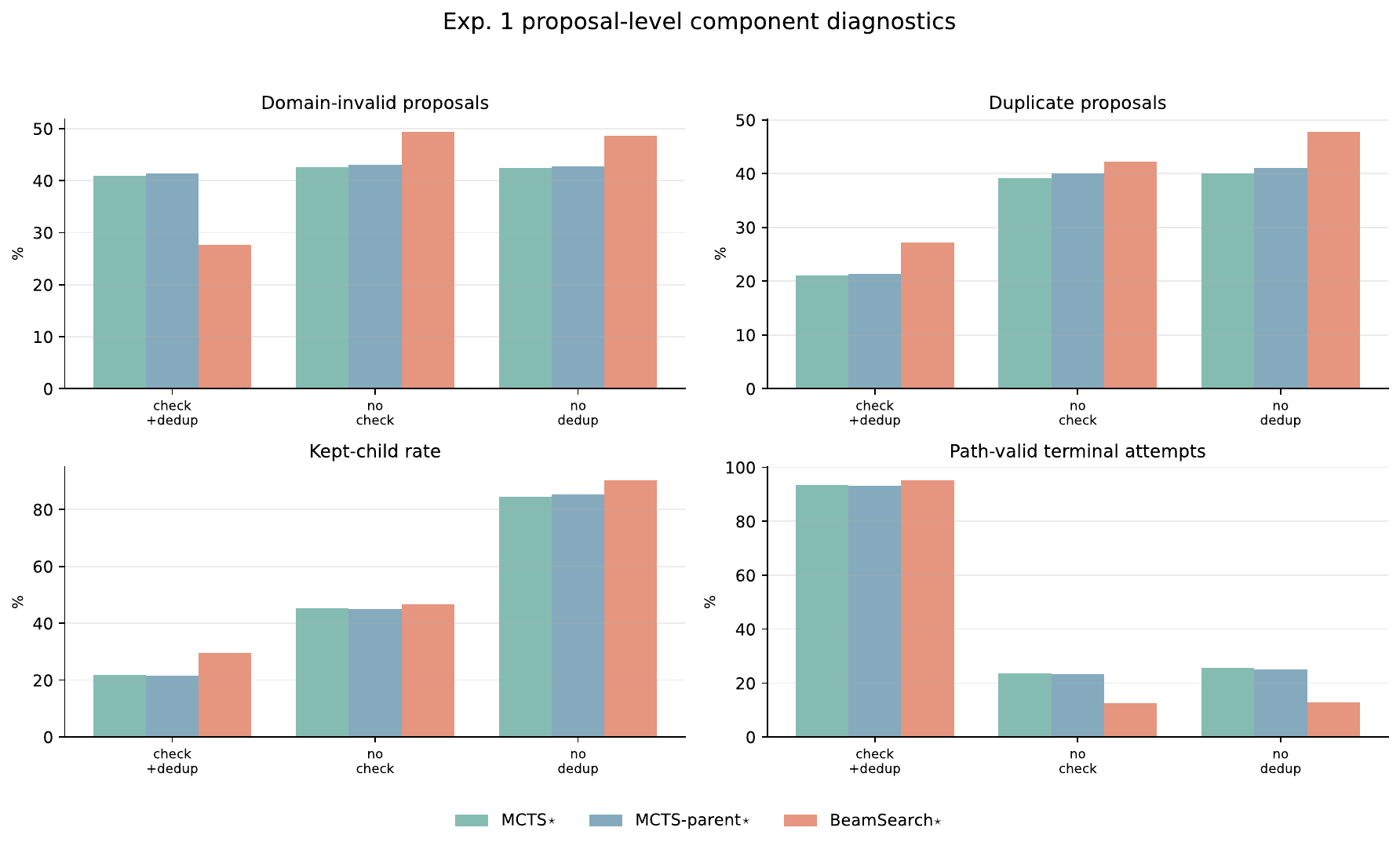}
  \caption{Proposal-level component diagnostics.  The full S1 configuration retains fewer raw proposals, but the retained proposals are substantially more likely to remain path-valid and terminal-feasible.  S2 exposes the search to invalid actions, while S3 retains many duplicate proposals; both effects reduce search efficiency.}
  \label{fig:app-exp1-component-diagnostics}
\end{figure}

\begin{figure}[t]
  \centering
  \includegraphics[width=\linewidth]{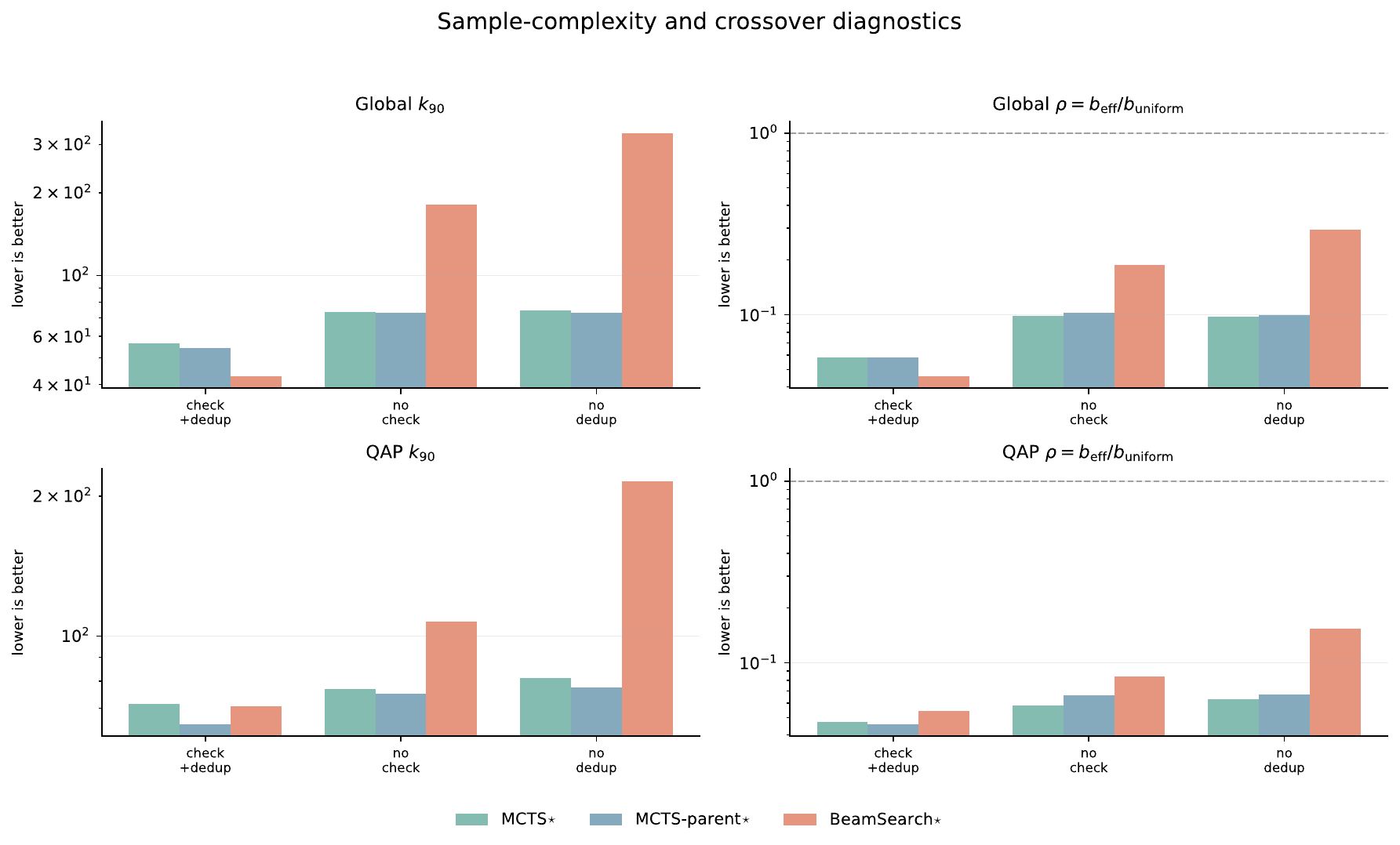}
  \caption{Sample-complexity and crossover readouts.  Lower $k_{90}$ means fewer samples are needed to observe an $\epsilon_{\mathrm{rel}}$-good terminal with high probability.  Lower $\rho=b_{\mathrm{eff}}/b_{\mathrm{uniform}}$ means the search procedure is more efficient than the uniform reference-pool proxy.}
  \label{fig:app-exp1-k-costratio}
\end{figure}

\paragraph{Discussion.}
The task-separated branching curves support the main-text claim that feasibility checking/pruning and duplicate merging reduce the effective branching burden before training.  On the branching readout, S1 achieves lower $b_{\mathrm{eff}}$ curves than the no-check and no-deduplication ablations across the tested difficulty levels.  By contrast, the solver-reference trajectory generator has high $b_{\mathrm{eff}}$, reflecting that it is not a calibrated stochastic search policy over all good terminals.

The aggregate and QAP-specific results show the same component effect through pass@$k$, terminal feasibility, and $k_{90}$.  It is important to distinguish pass@$k$ from exact optimality: pass@$k$ uses the $\epsilon_{\mathrm{rel}}$-good threshold, so a method can often reach near-optimal terminals even when its exact-optimality rate is lower.  For this reason, the solver-reference rows should be interpreted as a trajectory-generation baseline, not as a statement about solver optimality itself.

\begin{samepage}
A supplementary Knapsack audit uses exact feasible-subset counts of 45, 75, 215, and 560 across the four difficulty levels as structural proxies. Their logarithms have a negative association with observed without-replacement pass@8 (Spearman $\rho\in[-0.800,-0.600]$; exact permutation $p\in[0.333,0.417]$). Only four levels are compared, with multiple task properties changing together. This is a non-significant descriptive consistency check: feasible-subset counts do not estimate $N_\alpha$, and the audit does not test the theorem's information requirement.
\par
\end{samepage}

\clearpage
\section{Autonomous four-action control: additional protocol details}
\label{app:credit-control}

\paragraph{Fixed-horizon construction.}
The comparison in Exp.~\ref{task:credit} uses newly constructed Knapsack and Role Assignment instances requiring exactly four structured decisions. Some instances in the original experiments could finish in two decisions, providing limited separation between terminal and intermediate supervision. Four decisions create three genuine nonterminal states. The new control retains the same task families and difficulty principles, but its training and evaluation IDs differ from those used for the solver-prefix results in Table~\ref{tab:qwen25-3b-pass1-optimization}. Within this control, every method shares the same training/evaluation IDs and splits, prompts, action schema, parser, invalid-action handling, and terminal evaluator.

\paragraph{Interaction and supervision.}
The implementation details below describe the inspected runner; historical run settings remain to be confirmed. The initial prompt presents the instance and requests one structured action. The environment parses and validates each action and tracks the state internally. Subsequent turns retain the conversation history and receive a generic continuation prompt, without refreshed state descriptions, solver scores, or oracle actions. A complete legal solution requires four actions. Outcome GRPO assigns the signed terminal objective to a valid complete trajectory and an invalid sentinel otherwise. Outcome RS/SFT selects the highest-quality valid trajectory from each online group, skipping all-invalid groups. Outcome DPO uses a best-versus-worst terminal-quality pair per non-tied group, with the initial model as fixed reference and $\beta=0.1$; skipped groups are not replaced.

\paragraph{Solver credit and pairwise advantages.}
Let $V(s)$ be the exact best achievable terminal objective, signed so that higher is better. The solver-based arms call the domain's exact completion solver and use
\[
  r_{i,t}=V(s_{i,t+1})-V(s_{i,t}),\qquad
  A^{\mathrm{rank}}_{i,t}=\frac{W_{i,t}-L_{i,t}}{K_t-1}.
\]
Solver-based GRPO centers and standardizes the numerical rewards within each root-and-turn group. The ranking arm instead uses the pairwise advantage above, with $W_{i,t}$ and $L_{i,t}$ counting wins and losses against the other $K_t-1$ trajectories in that group. Valid rewards outrank invalid ones; ties contribute zero, and singleton groups receive zero. Trajectories share their initial instance and action index but may occupy different intermediate states. Each advantage applies only to assistant-generated tokens at that turn, with no additional normalization in the ranking arm. For fixed-prefix candidates, subtracting the common pre-action value preserves score ordering; for autonomous trajectories, the pre-action values can differ.

\paragraph{Invalid actions and incomplete trajectories.}
An invalid action receives the sentinel $-99999$, leaves the internal state unchanged, and invalidates the trajectory; generation continues for the fixed horizon. An incomplete or infeasible final trajectory without an earlier invalid action receives the sentinel on its final dense-reward event, preserving earlier turn credit. The pairwise estimator treats the sentinel as invalid rather than using its numerical magnitude.

\paragraph{Matched policy-optimization budget.}
The three GRPO variants share initialization, horizon, rollout group size, training-step budget, and token limits. The matched implementation also shares clipping, token-mean policy loss, and disabled KL regularization and group filtering. Consequently, the solver-based pair isolates the advantage estimator within this recipe; comparison with Outcome GRPO additionally changes credit timing. Equal token limits do not imply equal realized token counts, solver calls, or wall-clock cost. Outcome RS/SFT and DPO share the construction and evaluation protocol but use different training objectives. Table~\ref{tab:autonomous-config} records the current ordered-runner configuration, separately from the solver-prefix settings.

\begin{table}[H]
  \centering
  \small
  \setlength{\tabcolsep}{5pt}
  \renewcommand{\arraystretch}{1.08}
  \caption{Current autonomous implementation configuration. These are defaults
  of the inspected ordered-training runner; resolved historical run manifests
  were unavailable. Token limits specify allocations, not measured token use.}
  \label{tab:autonomous-config}
  \begin{tabular}{@{}p{0.34\linewidth}p{0.61\linewidth}@{}}
    \toprule
    Setting & Configured value \\
    \midrule
    Initial model & Qwen2.5-3B-Instruct \\
    Action horizon & Four assistant reasoning/action turns \\
    Training roots per domain & 512 per level; 2,048 across four levels \\
    Training schedule & One epoch, $\alpha_1\!\to\!\alpha_2\!\to\!\alpha_3\!\to\!\alpha_4$, no shuffle; 32 roots/batch; 64 training steps \\
    Training rollout group & Eight trajectories per root \\
    Actor learning rate & $10^{-6}$ \\
    Held-out evaluation & 64 roots/level, 256/domain; eight trajectories/root \\
    Generated-token cap & 2,048 tokens per assistant action \\
    Prompt / response / context limits & 1,792 / 12,288 / 18,432 tokens \\
    GRPO-arm policy loss & Token mean; clipping $0.2$ below and $0.28$ above, dual-clip cap $10$ \\
    GRPO-arm KL regularization & No reward KL or actor-loss KL \\
    Group filtering / failed-group refill & Both disabled \\
    \bottomrule
  \end{tabular}
\end{table}

\paragraph{Autonomous evaluation and uncertainty.}
No method receives solver actions, best-completion values, or search at evaluation. The policy constructs all four decisions from the initial instance, its conversation history, and continuation prompts. Each domain has 256 held-out roots with eight saved rollouts per root. Terminal feasibility, exact optimality, and $p_g=\mathrm{pass@1}$ use the reported evaluator readouts; pass@8 is observed directly from the eight rollouts rather than inferred from pass@1. The branching statistic $b_{\mathrm{eff}}$ is a descriptive diagnostic. It and pass@8 share the same good-terminal indicators, so their within-set association is not an independent predictive test. Establishing predictive validity would require a separate held-out study across tasks or training stages. The $\pm$ terms in Table~\ref{tab:autonomous-credit} are approximate 95\% root-cluster confidence margins; they quantify held-out-root sampling uncertainty and do not quantify variation across independent training runs.

\section{Broader impacts}
\label{app:broader-impacts}

\OPTSTAR{} provides controlled, automatically verifiable tasks for studying decision-making and may support more efficient planning and resource allocation. These synthetic results do not establish reliability for deployment. In applications such as scheduling or role assignment, optimizing an incomplete objective can disadvantage people or omit fairness and other domain constraints, even when every modeled constraint is satisfied. Feasibility checking and solver optimality certify only the encoded problem, not that its objectives or constraints are socially appropriate. Real-world use therefore requires domain-specific validation, appropriate fairness and safety constraints, and human oversight.


{\small
\begin{thebibliography}{65}
\providecommand{\natexlab}[1]{#1}
\providecommand{\url}[1]{\texttt{#1}}
\expandafter\ifx\csname urlstyle\endcsname\relax
  \providecommand{\doi}[1]{doi: #1}\else
  \providecommand{\doi}{doi: \begingroup \urlstyle{rm}\Url}\fi

\bibitem[Albalak et~al.(2025)Albalak, Phung, Lile, Rafailov, Gandhi,
  Castricato, Singh, Blagden, Xiang, Mahan, and Haber]{albalak2025bigmath}
Alon Albalak, Duy Phung, Nathan Lile, Rafael Rafailov, Kanishk Gandhi, Louis
  Castricato, Anikait Singh, Chase Blagden, Violet Xiang, Dakota Mahan, and
  Nick Haber.
\newblock Big-math: A large-scale, high-quality math dataset for reinforcement
  learning in language models, 2025.
\newblock URL \url{https://arxiv.org/abs/2502.17387}.

\bibitem[{AlphaProof and AlphaGeometry teams}(2024)]{DeepMind2024AlphaProof}
{AlphaProof and AlphaGeometry teams}.
\newblock Ai achieves silver-medal standard solving international mathematical
  olympiad problems.
\newblock
  \url{https://deepmind.google/discover/blog/ai-solves-imo-problems-at-silver-medal-level/},
  July 2024.
\newblock Accessed 2025-09-25.

\bibitem[Astorga et~al.(2026)Astorga, Kriz, and van der Schaar]{astorga2026causim}
Nicol{\'a}s Astorga, Anita Kriz, and Mihaela van der Schaar.
\newblock Scaling Causal Reasoning with Increasingly Complex Causal Simulators.
\newblock In \emph{Advances in Neural Information Processing Systems (NeurIPS)}, 2026.
\newblock URL \url{https://neurips.cc/virtual/2026/poster/148175}.

\bibitem[Astorga et~al.(2025)Astorga, Liu, Xiao, and van der Schaar]{astorga2025autoformulation}
Nicol{\'a}s Astorga, Tennison Liu, Yuanzhang Xiao, and Mihaela van der Schaar.
\newblock Autoformulation of Mathematical Optimization Models Using LLMs.
\newblock In \emph{Proceedings of the 42nd International Conference on Machine Learning}, volume~267 of \emph{Proceedings of Machine Learning Research}, pages 1864--1886. PMLR, 2025.
\newblock URL \url{https://proceedings.mlr.press/v267/astorga25a.html}.

\bibitem[Austin et~al.(2021)Austin, Odena, Nye, Bosma, Michalewski, Dohan,
  Jiang, Cai, Terry, Le, and Sutton]{austin2021programsynthesis}
Jacob Austin, Augustus Odena, Maxwell Nye, Maarten Bosma, Henryk Michalewski,
  David Dohan, Ellen Jiang, Carrie Cai, Michael Terry, Quoc Le, and Charles
  Sutton.
\newblock Program synthesis with large language models.
\newblock \emph{arXiv preprint arXiv:2108.07732}, 2021.

\bibitem[Besta et~al.(2024)Besta, Blach, Kubicek, Gerstenberger, Podstawski,
  Gianinazzi, Gajda, Lehmann, Niewiadomski, Nyczyk, and
  Hoefler]{besta2024graphofthoughts}
Maciej Besta, Nils Blach, Ales Kubicek, Robert Gerstenberger, Michal
  Podstawski, Lukas Gianinazzi, Joanna Gajda, Tomasz Lehmann, Hubert
  Niewiadomski, Piotr Nyczyk, and Torsten Hoefler.
\newblock Graph of thoughts: Solving elaborate problems with large language
  models.
\newblock \emph{Proceedings of the AAAI Conference on Artificial Intelligence},
  38\penalty0 (16):\penalty0 17682--17690, 2024.

\bibitem[Browne et~al.(2012)Browne, Powley, Whitehouse, Lucas, Cowling,
  Rohlfshagen, Tavener, Perez, Samothrakis, and Colton]{browne2012survey}
Cameron~B. Browne, Edward Powley, Daniel Whitehouse, Simon~M. Lucas, Peter~I.
  Cowling, Philipp Rohlfshagen, Stephen Tavener, Diego Perez, Spyridon
  Samothrakis, and Simon Colton.
\newblock A survey of monte carlo tree search methods.
\newblock \emph{IEEE Transactions on Computational Intelligence and AI in
  Games}, 4\penalty0 (1):\penalty0 1--43, 2012.
\newblock \doi{10.1109/TCIAIG.2012.2186810}.

\bibitem[Chen et~al.(2024{\natexlab{a}})Chen, Liao, Li, and
  Fan]{chen2024alphamathzerops}
Guoxin Chen, Minpeng Liao, Chengxi Li, and Kai Fan.
\newblock Alphamath almost zero: process supervision without process,
  2024{\natexlab{a}}.
\newblock URL \url{https://arxiv.org/abs/2405.03553}.

\bibitem[Chen et~al.(2024{\natexlab{b}})Chen, Liao, Li, and
  Fan]{chen-etal-2024-step}
Guoxin Chen, Minpeng Liao, Chengxi Li, and Kai Fan.
\newblock Step-level value preference optimization for mathematical reasoning.
\newblock In \emph{Findings of the Association for Computational Linguistics:
  EMNLP 2024}, pages 7889--7903, Miami, Florida, USA, November
  2024{\natexlab{b}}. Association for Computational Linguistics.
\newblock \doi{10.18653/v1/2024.findings-emnlp.463}.
\newblock URL \url{https://aclanthology.org/2024.findings-emnlp.463/}.

\bibitem[Chen et~al.(2025)Chen, He, Yuan, Chen, Cai, Dai, Yu, Yu, Li, Chen,
  et~al.]{chen2025enigmata}
Jiangjie Chen, Qianyu He, Siyu Yuan, Aili Chen, Zhicheng Cai, Weinan Dai,
  Hongli Yu, Qiying Yu, Xuefeng Li, Jiaze Chen, et~al.
\newblock Enigmata: Scaling logical reasoning in large language models with
  synthetic verifiable puzzles.
\newblock \emph{arXiv preprint arXiv:2505.19914}, 2025.

\bibitem[Chen et~al.(2021)Chen, Tworek, Jun, Yuan, Pinto, Kaplan, Edwards,
  Burda, Joseph, Brockman, Ray, Puri, Krueger, Petrov, Khlaaf, Sastry, Mishkin,
  Chan, Gray, Ryder, Pavlov, Power, Kaiser, Bavarian, Winter, Tillet, Such,
  Cummings, Plappert, Chantzis, Barnes, Herbert-Voss, Guss, Nichol, Paino,
  Tezak, Tang, Babuschkin, Balaji, Jain, Saunders, Hesse, Carr, Leike, Achiam,
  Misra, Morikawa, Radford, Knight, Brundage, Murati, Mayer, Welinder, McGrew,
  Amodei, McCandlish, Sutskever, and Zaremba]{chen2021evaluatingcodex}
Mark Chen, Jerry Tworek, Heewoo Jun, Qiming Yuan, Henrique Ponde de~Oliveira
  Pinto, Jared Kaplan, Harri Edwards, Yuri Burda, Nicholas Joseph, Greg
  Brockman, Alex Ray, Raul Puri, Gretchen Krueger, Michael Petrov, Heidy
  Khlaaf, Girish Sastry, Pamela Mishkin, Brooke Chan, Scott Gray, Nick Ryder,
  Mikhail Pavlov, Alethea Power, Lukasz Kaiser, Mohammad Bavarian, Clemens
  Winter, Philippe Tillet, Felipe~Petroski Such, Dave Cummings, Matthias
  Plappert, Fotios Chantzis, Elizabeth Barnes, Ariel Herbert-Voss,
  William~Hebgen Guss, Alex Nichol, Alex Paino, Nikolas Tezak, Jie Tang, Igor
  Babuschkin, Suchir Balaji, Shantanu Jain, William Saunders, Christopher
  Hesse, Andrew~N. Carr, Jan Leike, Josh Achiam, Vedant Misra, Evan Morikawa,
  Alec Radford, Matthew Knight, Miles Brundage, Mira Murati, Katie Mayer, Peter
  Welinder, Bob McGrew, Dario Amodei, Sam McCandlish, Ilya Sutskever, and
  Wojciech Zaremba.
\newblock Evaluating large language models trained on code.
\newblock \emph{arXiv preprint arXiv:2107.03374}, 2021.

\bibitem[Chen et~al.(2023)Chen, Lin, Sch{\"a}rli, and
  Zhou]{chen2023teachingselfdebug}
Xinyun Chen, Maxwell Lin, Nathanael Sch{\"a}rli, and Denny Zhou.
\newblock Teaching large language models to self-debug.
\newblock \emph{arXiv preprint arXiv:2304.05128}, 2023.

\bibitem[Chen et~al.(2025)Chen, Xia, Shao, Ge, and Ye]{chen2025sirl}
Yitian Chen, Jingfan Xia, Siyu Shao, Dongdong Ge, and Yinyu Ye.
\newblock Solver-Informed RL: Grounding Large Language Models for Authentic Optimization Modeling.
\newblock In \emph{Advances in Neural Information Processing Systems}, volume~38, 2025.
\newblock URL \url{https://arxiv.org/abs/2505.11792}.

\bibitem[Coulom(2007)]{coulom2006efficient}
R{\'e}mi Coulom.
\newblock Efficient selectivity and backup operators in {Monte-Carlo} tree
  search.
\newblock In H.~Jaap van~den Herik, Paolo Ciancarini, and
  H.~H.~L.~M. Donkers, editors, \emph{Computers and Games: CG 2006},
  volume 4630 of \emph{Lecture Notes in Computer Science}, pages 72--83.
  Springer, 2007.
\newblock \doi{10.1007/978-3-540-75538-8_7}.

\bibitem[Feng et~al.(2023)]{feng2023alphazero}
Xidong Feng, Ziyu Wan, Muning Wen, Stephen Marcus McAleer, Ying Wen,
Weinan Zhang, and Jun Wang.
\newblock Alphazero-like Tree-Search can Guide Large Language Model Decoding
  and Training.
\newblock \emph{arXiv preprint arXiv:2309.17179}, 2023.
\newblock URL \url{https://arxiv.org/abs/2309.17179}.

\bibitem[Guan et~al.(2025)Guan, Zhang, Liu, Shang, Sun, Zhu, Yang, and
  Yang]{guan2025rstar}
Xinyu Guan, Li~Lyna Zhang, Yifei Liu, Ning Shang, Youran Sun, Yi~Zhu, Fan Yang,
  and Mao Yang.
\newblock rstar-math: Small llms can master math reasoning with self-evolved
  deep thinking.
\newblock \emph{arXiv preprint arXiv:2501.04519}, 2025.

\bibitem[Guo et~al.(2025)Guo, Li, Wang, Jegelka, and Wang]{guo2025g1}
Xiaojun Guo, Ang Li, Yifei Wang, Stefanie Jegelka, and Yisen Wang.
\newblock G1: Teaching LLMs to Reason on Graphs with Reinforcement Learning.
\newblock In \emph{Advances in Neural Information Processing Systems}, volume~38, 2025.
\newblock URL \url{https://arxiv.org/abs/2505.18499}.

\bibitem[Hendrycks et~al.(2021{\natexlab{a}})Hendrycks, Basart, Kadavath,
  Mazeika, Arora, Guo, Burns, Puranik, He, Song, and
  Steinhardt]{hendrycks2021apps}
Dan Hendrycks, Steven Basart, Saurav Kadavath, Mantas Mazeika, Akul Arora,
  Ethan Guo, Collin Burns, Samir Puranik, Horace He, Dawn Song, and Jacob
  Steinhardt.
\newblock Measuring coding challenge competence with {APPS}.
\newblock In \emph{Proceedings of the Neural Information Processing Systems
  Track on Datasets and Benchmarks}, volume~1, 2021{\natexlab{a}}.
\newblock arXiv:2105.09938.

\bibitem[Hendrycks et~al.(2021{\natexlab{b}})Hendrycks, Burns, Kadavath, Arora,
  Basart, Tang, Song, and Steinhardt]{hendrycks2021mathdataset}
Dan Hendrycks, Collin Burns, Saurav Kadavath, Akul Arora, Steven Basart, Eric
  Tang, Dawn Song, and Jacob Steinhardt.
\newblock Measuring mathematical problem solving with the {MATH} dataset.
\newblock In \emph{Proceedings of the Neural Information Processing Systems
  Track on Datasets and Benchmarks}, 2021{\natexlab{b}}.

\bibitem[Hosseini et~al.(2024)Hosseini, Yuan, Malkin, Courville, Sordoni, and
  Agarwal]{hosseini2024v-star}
Arian Hosseini, Xingdi Yuan, Nikolay Malkin, Aaron Courville, Alessandro
  Sordoni, and Rishabh Agarwal.
\newblock {V}-{ST}a{R}: Training verifiers for self-taught reasoners.
\newblock In \emph{Conference on Language Modeling}, 2024.
\newblock URL \url{https://openreview.net/forum?id=stmqBSW2dV}.

\bibitem[Kobalczyk et~al.(2025)Kobalczyk, Astorga, Liu, and van der Schaar]{kobalczyk2025activetask}
Katarzyna Kobalczyk, Nicol{\'a}s Astorga, Tennison Liu, and Mihaela van der Schaar.
\newblock Active Task Disambiguation with LLMs.
\newblock In \emph{International Conference on Learning Representations (ICLR)}, 2025.
\newblock URL \url{https://proceedings.iclr.cc/paper_files/paper/2025/hash/5e07476b6bd2497e1fbd11b8f0b2de3c-Abstract-Conference.html}.

\bibitem[Kocsis and Szepesv{\'a}ri(2006)]{kocsis2006uct}
Levente Kocsis and Csaba Szepesv{\'a}ri.
\newblock Bandit based monte-carlo planning.
\newblock In \emph{Machine Learning: ECML 2006}, pages 282--293. Springer,
  2006.
\newblock \doi{10.1007/11871842_29}.

\bibitem[Kriz et~al.(2026)Kriz, Berthon, Astorga, and van der Schaar]{kriz2026amortizing}
Anita Kriz, Antonin Berthon, Nicol{\'a}s Astorga, and Mihaela van der Schaar.
\newblock Amortizing Sequential Experimental Design in Reasoning LLMs.
\newblock \emph{arXiv preprint arXiv:2610.20832}, 2026.
\newblock URL \url{https://arxiv.org/abs/2610.20832}.

\bibitem[Lake et~al.(2017)Lake, Ullman, Tenenbaum, and
  Gershman]{lake2017buildingthinklikepeople}
Brenden~M Lake, Tomer~D Ullman, Joshua~B Tenenbaum, and Samuel~J Gershman.
\newblock Building machines that learn and think like people.
\newblock \emph{Behavioral and brain sciences}, 40:\penalty0 e253, 2017.

\bibitem[Lanchantin et~al.(2025)Lanchantin, Chen, Lan, Li, Saha, Wang, Xu, Yu,
  Yuan, Weston, et~al.]{lanchantin2025bridginggapofflineonline}
Jack Lanchantin, Angelica Chen, Janice Lan, Xian Li, Swarnadeep Saha, Tianlu
  Wang, Jing Xu, Ping Yu, Weizhe Yuan, Jason~E Weston, et~al.
\newblock Bridging offline and online reinforcement learning for llms.
\newblock \emph{arXiv preprint arXiv:2506.21495}, 2025.

\bibitem[Li et~al.(2023)Li, Hessel, Yu, Ren, Chang, and
  Choi]{li2023symbolicdistillcot}
Liunian~Harold Li, Jack Hessel, Youngjae Yu, Xiang Ren, Kai-Wei Chang, and
  Yejin Choi.
\newblock Symbolic chain-of-thought distillation: Small models can also
  ``think'' step-by-step.
\newblock In \emph{Proceedings of the 61st Annual Meeting of the Association
  for Computational Linguistics (Volume 1: Long Papers)}, pages 2665--2679.
  Association for Computational Linguistics, 2023.
\newblock \doi{10.18653/v1/2023.acl-long.150}.
\newblock URL \url{https://aclanthology.org/2023.acl-long.150/}.

\bibitem[Li et~al.(2025)]{li2025internbootcamp}
Peiji Li, Jiasheng Ye, Yongkang Chen, Yichuan Ma, Zijie Yu, Kedi Chen,
Xiaozhe Li, Ganqu Cui, Haozhan Li, Jiacheng Chen, Chengqi Lyu, Wenwei Zhang,
Linyang Li, Qipeng Guo, Dahua Lin, Bowen Zhou, and Kai Chen.
\newblock {InternBootcamp} technical report: Boosting {LLM} reasoning with
  verifiable task scaling.
\newblock \emph{arXiv preprint arXiv:2508.08636}, 2025.
\newblock URL \url{https://arxiv.org/abs/2508.08636v2}.


\bibitem[Li et~al.(2025)Li, Fang, Ding, Li, Duan, Liu, and Chen]{li2025npengine}
Xiaozhe Li, Xinyu Fang, Shengyuan Ding, Linyang Li, Haodong Duan, Qingwen Liu, and Kai Chen.
\newblock NP-Engine: Empowering Optimization Reasoning in Large Language Models with Verifiable Synthetic NP Problems.
\newblock \emph{arXiv preprint arXiv:2510.16476}, 2025.
\newblock URL \url{https://arxiv.org/abs/2510.16476}.

\bibitem[Li et~al.(2022)Li, Choi, Chung, Kushman, Schrittwieser, Leblond,
  Eccles, Keeling, Gimeno, Dal~Lago, Hubert, Choy, de~Masson~d'Autume,
  Babuschkin, Chen, Huang, Welbl, Gowal, Cherepanov, Molloy, Mankowitz,
  Sutherland~Robson, Kohli, de~Freitas, Kavukcuoglu, and
  Vinyals]{li2022alphacode}
Yujia Li, David Choi, Junyoung Chung, Nate Kushman, Julian Schrittwieser,
  R{\'e}mi Leblond, Tom Eccles, James Keeling, Felix Gimeno, Agustin Dal~Lago,
  Thomas Hubert, Peter Choy, Cyprien de~Masson~d'Autume, Igor Babuschkin,
  Xinyun Chen, Po-Sen Huang, Johannes Welbl, Sven Gowal, Alexey Cherepanov,
  James Molloy, Daniel~J. Mankowitz, Esme Sutherland~Robson, Pushmeet Kohli,
  Nando de~Freitas, Koray Kavukcuoglu, and Oriol Vinyals.
\newblock Competition-level code generation with {AlphaCode}.
\newblock \emph{Science}, 378\penalty0 (6624):\penalty0 1092--1097, 2022.
\newblock \doi{10.1126/science.abq1158}.
\newblock arXiv:2203.07814.

\bibitem[Lightman et~al.(2024)Lightman, Kosaraju, Burda, Edwards, Baker, Lee,
  Leike, Schulman, Sutskever, and Cobbe]{lightman2024letsverify}
Hunter Lightman, Vineet Kosaraju, Yuri Burda, Harrison Edwards, Bowen Baker,
  Teddy Lee, Jan Leike, John Schulman, Ilya Sutskever, and Karl Cobbe.
\newblock Let's verify step by step.
\newblock In \emph{The Twelfth International Conference on Learning
  Representations}, 2024.
\newblock URL \url{https://openreview.net/forum?id=v8L0pN6EOi}.


\bibitem[Luo et~al.(2024)Luo, Liu, Liu, Phatale, Guo, Lara, Li, Shu, Zhu, Meng,
  et~al.]{luo2024improveomegaprm}
Liangchen Luo, Yinxiao Liu, Rosanne Liu, Samrat Phatale, Meiqi Guo, Harsh Lara,
  Yunxuan Li, Lei Shu, Yun Zhu, Lei Meng, et~al.
\newblock Improve mathematical reasoning in language models by automated
  process supervision.
\newblock \emph{arXiv preprint arXiv:2406.06592}, 2024.

\bibitem[{Mathematical Association of America}(2023)]{MAA_AMC_2023}
{Mathematical Association of America}.
\newblock American mathematics competitions ({AMC}) 2023.
\newblock Mathematical Association of America, 2023.
\newblock Benchmark dataset distributed by \texttt{math-ai}.
\newblock URL \url{https://huggingface.co/datasets/math-ai/amc23}.
\newblock Accessed 2026-11-07.

\bibitem[{Mathematical Association of America}(2024)]{MAA_AIME_2024}
{Mathematical Association of America}.
\newblock American invitational mathematics examination ({AIME}) 2024.
\newblock Mathematical Association of America, 2024.
\newblock Benchmark dataset by Yifan Zhang and the Math-AI Team.
\newblock URL \url{https://huggingface.co/datasets/math-ai/aime24}.
\newblock Accessed 2026-11-07.

\bibitem[{Mathematical Association of America}(2025)]{MAA_AIME_2025}
{Mathematical Association of America}.
\newblock American invitational mathematics examination ({AIME}) 2025.
\newblock Mathematical Association of America, 2025.
\newblock Benchmark dataset by Yifan Zhang and the Math-AI Team.
\newblock URL \url{https://huggingface.co/datasets/math-ai/aime25}.
\newblock Accessed 2026-11-07.


\bibitem[{Meta}(2024)]{meta2024llama32modelcard}
{Meta}.
\newblock Llama 3.2 3b instruct model card.
\newblock Hugging Face model card, 2024.
\newblock URL \url{https://huggingface.co/meta-llama/Llama-3.2-3B-Instruct}.

\bibitem[Patel et~al.(2025)Patel, Reddy, and Bahdanau]{patel2025getCHASE}
Arkil Patel, Siva Reddy, and Dzmitry Bahdanau.
\newblock How to get your llm to generate challenging problems for evaluation.
\newblock \emph{arXiv preprint arXiv:2502.14678}, 2025.

\bibitem[{Qwen Team}(2024)]{qwen25technicalreport}
{Qwen Team}.
\newblock Qwen2.5 technical report, 2024.
\newblock URL \url{https://arxiv.org/abs/2412.15115}.

\bibitem[Rauba et~al.(2026)Rauba, Fanconi, and van der Schaar]{rauba2026tiny}
Paulius Rauba, Claudio Fanconi, and Mihaela van der Schaar.
\newblock Tiny Autoregressive Recursive Models.
\newblock In \emph{ICLR 2026 Workshop on AI with Recursive Self-Improvement}, 2026.
\newblock URL \url{https://arxiv.org/abs/2603.08082}.

\bibitem[Saxton et~al.(2019)Saxton, Grefenstette, Hill, and
  Kohli]{saxton2019mathematicsdataset}
David Saxton, Edward Grefenstette, Felix Hill, and Pushmeet Kohli.
\newblock Analysing mathematical reasoning abilities of neural models.
\newblock In \emph{International Conference on Learning Representations
  (ICLR)}, 2019.
\newblock arXiv:1904.01557.

\bibitem[Schulman et~al.(2017)Schulman, Wolski, Dhariwal, Radford, and
  Klimov]{schulman2017proximal}
John Schulman, Filip Wolski, Prafulla Dhariwal, Alec Radford, and Oleg Klimov.
\newblock Proximal policy optimization algorithms.
\newblock \emph{arXiv preprint arXiv:1707.06347}, 2017.

\bibitem[Shao et~al.(2024)Shao, Wang, Zhu, Xu, Song, Bi, Zhang, Zhang, Li, Wu,
  and Guo]{shao2024deepseekmath}
Zhihong Shao, Peiyi Wang, Qihao Zhu, Runxin Xu, Junxiao Song, Xiao Bi, Haowei
  Zhang, Mingchuan Zhang, Y.~K. Li, Y.~Wu, and Daya Guo.
\newblock {DeepSeekMath}: Pushing the limits of mathematical reasoning in open
  language models.
\newblock \emph{arXiv preprint arXiv:2402.03300}, 2024.

\bibitem[Sheng et~al.(2025)Sheng, Zhang, Ye, Wu, Zhang, Zhang, Peng, Lin, and
  Wu]{sheng2025hybridflow}
Guangming Sheng, Chi Zhang, Zilingfeng Ye, Xibin Wu, Wang Zhang, Ru~Zhang,
  Yanghua Peng, Haibin Lin, and Chuan Wu.
\newblock Hybridflow: A flexible and efficient {RLHF} framework.
\newblock In \emph{Proceedings of the Twentieth European Conference on Computer
  Systems}. ACM, 2025.
\newblock \doi{10.1145/3689031.3696075}.

\bibitem[Silver et~al.(2016)Silver, Huang, Maddison, Guez, Sifre, Van
  Den~Driessche, Schrittwieser, Antonoglou, Panneershelvam, Lanctot,
  et~al.]{silver2016mastering}
David Silver, Aja Huang, Chris~J Maddison, Arthur Guez, Laurent Sifre, George
  Van Den~Driessche, Julian Schrittwieser, Ioannis Antonoglou, Veda
  Panneershelvam, Marc Lanctot, et~al.
\newblock Mastering the game of go with deep neural networks and tree search.
\newblock \emph{nature}, 529\penalty0 (7587):\penalty0 484--489, 2016.

\bibitem[Silver et~al.(2017)Silver, Schrittwieser, Simonyan, Antonoglou, Huang,
  Guez, Hubert, Baker, Lai, Bolton, et~al.]{silver2017mastering}
David Silver, Julian Schrittwieser, Karen Simonyan, Ioannis Antonoglou, Aja
  Huang, Arthur Guez, Thomas Hubert, Lucas Baker, Matthew Lai, Adrian Bolton,
  et~al.
\newblock Mastering the game of go without human knowledge.
\newblock \emph{nature}, 550\penalty0 (7676):\penalty0 354--359, 2017.

\bibitem[Singh et~al.(2023)]{singh2023beyond}
Avi Singh, John~D. Co-Reyes, Rishabh Agarwal, Ankesh Anand, Piyush Patil,
Xavier Garcia, Peter~J. Liu, James Harrison, Jaehoon Lee, Kelvin Xu,
Aaron Parisi, Abhishek Kumar, Alex Alemi, Alex Rizkowsky, Azade Nova,
Ben Adlam, Bernd Bohnet, Gamaleldin Elsayed, Hanie Sedghi, Igor Mordatch,
Isabelle Simpson, Izzeddin Gur, Jasper Snoek, Jeffrey Pennington, Jiri Hron,
Kathleen Kenealy, Kevin Swersky, Kshiteej Mahajan, Laura Culp, Lechao Xiao,
Maxwell~L. Bileschi, Noah Constant, Roman Novak, Rosanne Liu, Tris Warkentin,
Yundi Qian, Yamini Bansal, Ethan Dyer, Behnam Neyshabur,
Jascha Sohl-Dickstein, and Noah Fiedel.
\newblock Beyond human data: Scaling self-training for problem-solving with
  language models.
\newblock \emph{arXiv preprint arXiv:2312.06585}, 2023.
\newblock URL \url{https://arxiv.org/abs/2312.06585}.


\bibitem[Stojanovski et~al.(2025)Stojanovski, Stanley, Sharratt, Jones,
  Adefioye, Kaddour, and K{\"o}pf]{stojanovski2025GYMreasoning}
Zafir Stojanovski, Oliver Stanley, Joe Sharratt, Richard Jones, Abdulhakeem
  Adefioye, Jean Kaddour, and Andreas K{\"o}pf.
\newblock {REASONING GYM}: Reasoning environments for reinforcement learning
  with verifiable rewards, 2025.
\newblock URL \url{https://arxiv.org/abs/2505.24760}.

\bibitem[Trinh et~al.(2024)Trinh, Wu, Le, He, and
  Luong]{AlphaGeometryTrinh2024}
Trieu Trinh, Yuhuai Wu, Quoc Le, He~He, and Thang Luong.
\newblock Solving olympiad geometry without human demonstrations.
\newblock \emph{Nature}, 2024.
\newblock \doi{10.1038/s41586-023-06747-5}.

\bibitem[Wang et~al.(2024)Wang, Li, Shao, Xu, Dai, Li, Chen, Wu, and Sui]{wang2024mathshepherd}
Peiyi Wang, Lei Li, Zhihong Shao, Runxin Xu, Damai Dai, Yifei Li, Deli Chen, Yu Wu, and Zhifang Sui.
\newblock Math-Shepherd: Verify and Reinforce LLMs Step-by-step without Human Annotations.
\newblock In \emph{Proceedings of the 62nd Annual Meeting of the Association for Computational Linguistics (Volume 1: Long Papers)}, pages 9426--9439. Association for Computational Linguistics, 2024.
\newblock \doi{10.18653/v1/2024.acl-long.510}.
\newblock URL \url{https://aclanthology.org/2024.acl-long.510/}.


\bibitem[Wang et~al.(2025)Wang, Liu, Tang, Chen, Li, Zhang, Zi, Zhang, and Li]{wang2025npgmuse}
Yuyao Wang, Bowen Liu, Jianheng Tang, Nuo Chen, Yuhan Li, Qifan Zhang, Chenyi Zi, Chen Zhang, and Jia Li.
\newblock NPG-Muse: Scaling Long Chain-of-Thought Reasoning with NP-Hard Graph Problems.
\newblock \emph{arXiv preprint arXiv:2508.20373}, 2025.
\newblock URL \url{https://arxiv.org/abs/2508.20373}.

\bibitem[Wei et~al.(2025)Wei, Wu, Wan, Suresh, Tan, Zhou, Koyejo, Wang, and
  Aiken]{wei2025satbench}
Anjiang Wei, Yuheng Wu, Yingjia Wan, Tarun Suresh, Huanmi Tan, Zhanke Zhou,
  Sanmi Koyejo, Ke~Wang, and Alex Aiken.
\newblock Satbench: Benchmarking llms' logical reasoning via automated puzzle
  generation from sat formulas.
\newblock \emph{arXiv preprint arXiv:2505.14615}, 2025.

\bibitem[Wei et~al.(2022)Wei, Wang, Schuurmans, Bosma, ichter, Xia, Chi, Le,
  and Zhou]{wei2022cot}
Jason Wei, Xuezhi Wang, Dale Schuurmans, Maarten Bosma, brian ichter, Fei Xia,
  Ed~Chi, Quoc~V Le, and Denny Zhou.
\newblock Chain-of-thought prompting elicits reasoning in large language
  models.
\newblock In S.~Koyejo, S.~Mohamed, A.~Agarwal, D.~Belgrave, K.~Cho, and A.~Oh,
  editors, \emph{Advances in Neural Information Processing Systems}, volume~35,
  pages 24824--24837. Curran Associates, Inc., 2022.

\bibitem[Wong et~al.(2025)Wong, Deng, He, Chen, You, Dong, Liang, Shen, Cui,
  and Zhang]{wong2025logicpuzzlerl}
Zhen~Hao Wong, Jingwen Deng, Runming He, Zirong Chen, Qijie You, Hejun Dong,
  Hao Liang, Chengyu Shen, Bin Cui, and Wentao Zhang.
\newblock Logicpuzzlerl: Cultivating robust mathematical reasoning in llms via
  reinforcement learning.
\newblock \emph{arXiv preprint arXiv:2506.04821}, 2025.

\bibitem[Xie et~al.(2024)Xie, Goyal, Zheng, Kan, Lillicrap, Kawaguchi, and
  Shieh]{xie2024monte}
Yuxi Xie, Anirudh Goyal, Wenyue Zheng, Min-Yen Kan, Timothy~P Lillicrap, Kenji
  Kawaguchi, and Michael Shieh.
\newblock Monte carlo tree search boosts reasoning via iterative preference
  learning.
\newblock \emph{arXiv preprint arXiv:2405.00451}, 2024.

\bibitem[Yao et~al.(2023{\natexlab{b}})Yao, Yu, Zhao, Shafran, Griffiths, Cao,
  and Narasimhan]{yao2023tree}
Shunyu Yao, Dian Yu, Jeffrey Zhao, Izhak Shafran, Thomas~L. Griffiths, Yuan
  Cao, and Karthik Narasimhan.
\newblock Tree of thoughts: Deliberate problem solving with large language
  models.
\newblock In \emph{Advances in Neural Information Processing Systems},
  volume~36, 2023{\natexlab{b}}.

\bibitem[Yao et~al.(2023{\natexlab{a}})Yao, Zhao, Yu, Du, Shafran, Narasimhan,
  and Cao]{yao2023react}
Shunyu Yao, Jeffrey Zhao, Dian Yu, Nan Du, Izhak Shafran, Karthik Narasimhan,
  and Yuan Cao.
\newblock React: Synergizing reasoning and acting in language models.
\newblock In \emph{International Conference on Learning Representations
  (ICLR)}, 2023{\natexlab{a}}.

\bibitem[Yu et~al.(2025)Yu, Zhang, Zhu, Yuan, Zuo, Yue, Dai, Fan, Liu, Liu,
  et~al.]{yu2025dapo}
Qiying Yu, Zheng Zhang, Ruofei Zhu, Yufeng Yuan, Xiaochen Zuo, Yu~Yue, Weinan
  Dai, Tiantian Fan, Gaohong Liu, Lingjun Liu, et~al.
\newblock Dapo: An open-source llm reinforcement learning system at scale.
\newblock \emph{arXiv preprint arXiv:2503.14476}, 2025.

\bibitem[Yuan et~al.(2023)]{yuan2023scaling}
Zheng Yuan, Hongyi Yuan, Chengpeng Li, Guanting Dong, Keming Lu, Chuanqi Tan,
Chang Zhou, and Jingren Zhou.
\newblock Scaling relationship on learning mathematical reasoning with large
  language models.
\newblock \emph{arXiv preprint arXiv:2308.01825}, 2023.
\newblock URL \url{https://arxiv.org/abs/2308.01825}.

\bibitem[Zelikman et~al.(2024)Zelikman, Harik, Shao, Jayasiri, Haber, and
  Goodman]{zelikman2024quietstar}
Eric Zelikman, Georges~Raif Harik, Yijia Shao, Varuna Jayasiri, Nick Haber, and
  Noah~D. Goodman.
\newblock Quiet-{ST}a{R}: Language models can teach themselves to think before
  speaking.
\newblock In \emph{Conference on Language Modeling}, 2024.
\newblock URL \url{https://openreview.net/forum?id=oRXPiSOGH9}.

\bibitem[Zelikman et~al.(2022)Zelikman, Wu, Mu, and Goodman]{zelikman2022star}
Eric Zelikman, Yuhuai Wu, Jesse Mu, and Noah Goodman.
\newblock Star: Bootstrapping reasoning with reasoning.
\newblock \emph{Advances in Neural Information Processing Systems},
  35:\penalty0 15476--15488, 2022.

\bibitem[Zeng et~al.(2025)]{zeng2025rlve}
Zhiyuan Zeng, Hamish Ivison, Yiping Wang, Lifan Yuan, Shuyue Stella Li, Zhuorui Ye, Siting Li, Jacqueline He, Runlong Zhou, Tong Chen, Chenyang Zhao, Yulia Tsvetkov, Simon Shaolei Du, Natasha Jaques, Hao Peng, Pang Wei Koh, and Hannaneh Hajishirzi.
\newblock RLVE: Scaling Up Reinforcement Learning for Language Models with Adaptive Verifiable Environments.
\newblock \emph{arXiv preprint arXiv:2511.07317}, 2025.
\newblock URL \url{https://arxiv.org/abs/2511.07317}.

\bibitem[Zhang et~al.(2024{\natexlab{a}})Zhang, Zhoubian, Hu, Yue, Dong, and
  Tang]{zhang2024mctsrest}
Dan Zhang, Sining Zhoubian, Ziniu Hu, Yisong Yue, Yuxiao Dong, and Jie Tang.
\newblock Rest-mcts*: Llm self-training via process reward guided tree search.
\newblock \emph{arXiv preprint arXiv:2406.03816}, 2024{\natexlab{a}}.

\bibitem[Zhang et~al.(2024{\natexlab{b}})Zhang, Huang, Zhou, Li, and
  Ouyang]{zhang2024accessing}
Di~Zhang, Xiaoshui Huang, Dongzhan Zhou, Yuqiang Li, and Wanli Ouyang.
\newblock Accessing gpt-4 level mathematical olympiad solutions via monte carlo
  tree self-refine with llama-3 8b.
\newblock \emph{arXiv preprint arXiv:2406.07394}, 2024{\natexlab{b}}.

\bibitem[Zheng et~al.(2025)Zheng, Liu, Li, Chen, Yu, Gao, Dang, Liu, Men, Yang,
  Zhou, and Lin]{zheng2025gspo}
Chujie Zheng, Shixuan Liu, Mingze Li, Xiong-Hui Chen, Bowen Yu, Chang Gao, Kai
  Dang, Yuqiong Liu, Rui Men, An~Yang, Jingren Zhou, and Junyang Lin.
\newblock Group sequence policy optimization, 2025.
\newblock URL \url{https://arxiv.org/abs/2507.18071}.

\bibitem[Zhou et~al.(2023)Zhou, Sch{\"a}rli, Hou, Wei, Scales, Wang,
  Schuurmans, Cui, Bousquet, Le, and Chi]{zhou2022least}
Denny Zhou, Nathanael Sch{\"a}rli, Le~Hou, Jason Wei, Nathan Scales, Xuezhi
  Wang, Dale Schuurmans, Claire Cui, Olivier Bousquet, Quoc Le, and Ed~H. Chi.
\newblock Least-to-most prompting enables complex reasoning in large language
  models.
\newblock In \emph{International Conference on Learning Representations}, 2023.
\newblock URL \url{https://openreview.net/forum?id=WZH7099tgfM}.

\bibitem[Zhu et~al.(2025)Zhu, Huang, Peng, Lu, Yu, Cheng, Qiu, Huang, and
  Lin]{zhu2025autologi}
Qin Zhu, Fei Huang, Runyu Peng, Keming Lu, Bowen Yu, Qinyuan Cheng, Xipeng Qiu,
  Xuanjing Huang, and Junyang Lin.
\newblock Autologi: Automated generation of logic puzzles for evaluating
  reasoning abilities of large language models.
\newblock \emph{arXiv preprint arXiv:2502.16906}, 2025.

\end{thebibliography}
}
\end{document}